\documentclass{article}

\usepackage[a4paper, total={6in, 9in}]{geometry}
\usepackage[square,sort,comma,numbers]{natbib}

\usepackage[utf8]{inputenc} 
\usepackage[T1]{fontenc}    
\usepackage{hyperref}       
\usepackage{url}            
\usepackage{booktabs}       
\usepackage{amsfonts}       
\usepackage{nicefrac}       
\usepackage{microtype}      
\usepackage{url}            
\hypersetup{colorlinks, citecolor=blue, linkcolor=red, urlcolor=blue}
\usepackage{mathtools}
\usepackage{graphicx}
\def\equationautorefname~#1\null{(#1)\null}

\usepackage{graphicx}
\usepackage{epstopdf}
\usepackage{amsopn}

\usepackage[bb=boondox]{mathalfa}
\usepackage{comment}
\usepackage[caption=false]{subfig}
\usepackage{defs}

\usepackage{algorithm}
\usepackage{algpseudocode}
\ifpdf
  \DeclareGraphicsExtensions{.eps,.pdf,.png,.jpg}
\else
  \DeclareGraphicsExtensions{.eps}
\fi

\usepackage{cleveref}

\usepackage{enumitem}
\setlist[enumerate]{leftmargin=.5in}
\setlist[itemize]{leftmargin=.5in}

\newtheorem{definition}{Definition}
\newtheorem{lemma}{Lemma}
\newtheorem{corollary}{Corollary}
\newtheorem{theorem}{Theorem}
\newtheorem{proof}{Proof}
\usepackage{float} 
\usepackage{xcolor}
\usepackage{xr}
\usepackage{wrapfig}

\title{Adversarial robustness of sparse local Lipschitz predictors 
}

\author{Ramchandran Muthukumar\thanks{Department of Computer Science \& Mathematical Institute for Data Science, Johns Hopkins University (\texttt{rmuthuk1@jhu.edu}).}
\and Jeremias Sulam\stepcounter{footnote}\stepcounter{footnote}\stepcounter{footnote}\stepcounter{footnote}\stepcounter{footnote}\thanks{Department of Biomedical Engineering  \& Mathematical Institute for Data Science,  Johns Hopkins University (\texttt{jsulam1@jhu.edu}).}
}

\date{}

\makeatletter
\newcommand*{\addFileDependency}[1]{
  \typeout{(#1)}
  \@addtofilelist{#1}
  \IfFileExists{#1}{}{\typeout{No file #1.}}
}
\makeatother

\AtBeginDocument{\colorlet{defaultcolor}{.}}
\newcommand{\addition}[1]{#1}
\begin{document}

\maketitle

\begin{abstract}
This work studies the adversarial robustness of parametric functions composed of a linear predictor and a non-linear representation map. 
Our analysis relies on \emph{sparse local Lipschitzness} (SLL),
an extension of local Lipschitz continuity 
that better captures the stability and reduced effective dimensionality of predictors upon local perturbations.  
SLL functions preserve a certain degree of structure, given by the sparsity pattern in the representation map, and include several popular hypothesis classes, such as piece-wise linear models, Lasso and its variants, and deep feed-forward \relu networks. 
We provide a tighter robustness certificate on the minimal energy of an adversarial example, as well as tighter data-dependent non-uniform bounds on the robust generalization error of these predictors.  We instantiate these results for the case of deep neural networks and provide numerical evidence that supports our results,  shedding new insights into natural regularization strategies
to increase the robustness of these models.
\end{abstract}

\section{Introduction}\label{sec: introduction}
During the past decade, deep learning has proven a successful model for a variety of real-world data-driven tasks such as 
image classification \cite{he2016deep}, 
language modeling \cite{Devlin2019BERTPO},
and more.
Modern deep learning architectures compose a learned representation map with a linear classifier, where the former can be feedforward, convolutional, recurrent, or attention maps, sequentially combined with nonlinear activation functions.
Despite the strong empirical success of these models,
a complete understanding of important properties,
such as generalization \cite{Jiang2020FantasticGM} and robustness \cite{Madry2018TowardsDL}, is lacking. 
Importantly, state-of-the-art deep learning models are vulnerable to adversarially crafted small perturbations to input, termed 
\emph{adversarial examples}  \cite{Szegedy2014IntriguingPO}.
This vulnerability limits 
the deployment of these models in safety-critical tasks such as autonomous driving \cite{Cao2019AdversarialSA} and healthcare \cite{Kermany2018IdentifyingMD}.

Adversarial examples are easy to generate, hard to detect \cite{Goodfellow2015ExplainingAH, Carlini2017AdversarialEA},
can be deployed in the physical world \cite{Eykholt2018PhysicalAE, Kurakin2017AdversarialEI}, and are often transferable across predictors for the same task \cite{Liu2017DelvingIT,Papernot2016TransferabilityIM}.
This has led to significant empirical research to defend models against attacks 
\cite{pmlr-v70-cisse17a, pmlr-v80-wong18a} 
as well subsequent work on improving these attacks to compromise the performance of defended models \cite{Athalye2018ObfuscatedGG, Brendel2018DecisionBasedAA, moosavi2016deepfool}. 
Several works have explored strategies to either improve or evaluate the robustness of modern deep learning models.
For instance, \textit{adversarial training} improves robustness by injecting adversarial attacks during the training phase 
\cite{Madry2018TowardsDL,  Salman2019ProvablyRD}, while other works focus on certifying the level of corruption a model can withstand \cite{Fazlyab2019EfficientAA, Hein2017FormalGO,  Raghunathan2018CertifiedDA, Sinha2017CertifiableDR, Weng2018TowardsFC, Zhang2018EfficientNN}.  

Amidst the rapidly evolving empirical insights, there has been concurrent research aimed at providing theoretical guarantees on adversarial robustness for different hypothesis classes \cite{Attias2019ImprovedGB, Awasthi2020AdversarialLG, khim2018adversarial, yin2018rademacher}. Some of these works study the computational and statistical limits of adversarial attacks \cite{Bubeck2019AdversarialEF, Fawzi2018AdversarialVF, mahloujifar2018curse}. Others study trade-offs between robustness and natural (or benign) performance \cite{Chen2020MoreDC, Tsipras2019RobustnessMB,  pmlr-v97-zhang19p}, provable guarantees for adversarial training \cite{AllenZhu2020FeaturePH, Zhang2020OverparameterizedAT}, or analyze optimal levels of provable adversarial defenses \cite{pmlr-v97-cohen19c, Shafahi2019AreAE}. 

In this work, we focus on two central questions of adversarial robustness: \textit{certified robustness} and \textit{robust generalization}. 
Our analysis for both of these questions will rely on the sensitivity of the model to changes in both its input and its parameters, a quantity that is naturally characterized by its Lipschitz constant. 
This view  can be quite limited, however: 
for general non-linear functions,
such sensitivity to perturbations can greatly vary across the input space (for different samples),
and across the hypothesis space (for different predictors). 
In this work, we show that local measures of sensitivity that additionally account for structural invariance in the outputs lead to tighter stability bounds and more informative results.

\subsection{Outline}\label{subsec: outline} The paper is organized as follows. 
We elaborate on the formal task of supervised learning and adversarial robustness in \Cref{sec: robust-supervised-learning}. 
Our contributions are summarized in \Cref{sec: contrib}. 
The next two sections collect our main results, \Cref{sec:cert-rob} for certified robustness and \Cref{sec: rob-gen} for robust generalization. 
Finally, we demonstrate experimental results in \Cref{sec:experimental} and conclude with future directions in \Cref{sec:conclusions}.

\subsection{Notation}\label{subsec: notation} 
Throughout this work, scalar quantities are denoted by lower or upper case (not bold) letters, and vectors with bold lower case letters.
Matrices are denoted by upper case letters: $\mt{W}$ is a matrix with \textit{rows} $\vc{w}_i$. 
The Frobenius and operator norms are denoted by $\norm{\cdot}_F$ and $\norm{\cdot}_2$, respectively. 
For any matrix $\W \in \mathbb{R}^{p \times d}$ with \emph{rows} $\vc{w}_i$, for $u,v \geq 1$, the group $(u,v)$ norm\footnote{Defined over rows rather than columns.} is defined as
$\norm{\W}_{u,v} := \norm{~ (\norm{\vc{w}_1}_u, \ldots, \norm{\vc{w}_p}_u)~}_v.$
\addition{We informally refer to the Euclidean norm of a vector as its \textit{energy}.
We denote by $\succeq$ the element-wise $\geq$ operator for vectors.}
Sets and spaces are denoted by capital (and often calligraphic) letters, with the exception of the set $[p] = \{1,\dots,p\}$. 
For a Banach space $\cW$ embedded with norm $\norm{\cdot}_\cW$, 
we denote by $\cB^{\cW}_{r}(\w)$, a bounded ball centered around point $\w$ with radius $r$.
When describing a composition of affine functions, such as deep neural networks, 
$\mt{W}^k$ refers to the parameters corresponding to layer $k$. 
More generally, outside of norms, superscripts indicate layer index. 
We denote by $\mathcal{P}_I$, the index selection operator that restricts an input to the coordinates specified in the set $I$. For a vector $\x\in\mathbb{R}^d$ and $I \subset [d]$, $\mathcal{P}_I:\mathbb R^d \to \mathbb R^{|I|}$ is defined as $\mathcal{P}_I(\x) := \x[I]$. For a matrix $\W\in\mathbb{R}^{p\times d}$ and $I\subset [p]$, $\mathcal{P}_I(\W) \in\mathbb{R}^{|I|\times d}$ restricts $\W$ to the \emph{rows} specified by $I$. 
For row and column index sets $I \subset [p]$ and $J \subset [d]$, $\mathcal{P}_{I,J}(\W) \in \mathbb{R}^{|I| \times |J|}$ restricts $\W$ to the corresponding sub-matrix. 

\section{Robust Supervised Learning}
\label{sec: robust-supervised-learning}
Consider the task of multi-class classification 
with a bounded input space $\cX = \{\x \in \mathbb{R}^d ~|~ \norm{\x}_2 \leq 1\}$ 
and labels $\cY = \{1,\ldots, C\}$ 
from an unknown distribution $\cD_\cZ$ over $\cZ := (\cX \times \cY)$.
We search for a hypothesis in $\mathcal{H} := \{h : \cX \rightarrow \cYd \}$ 
that is an accurate predictor of label $y$ given input $\x$. Note that $\cY$ and $\cYd$ need not be the same. In this work, we consider $\cYd = \mathbb{R}^C$, and consider the predicted label of the hypothesis $h$ as $\hat{y}(\x) :=\argmax_{j} [h(\x)]_j$\footnote{The $\argmax$ here breaks ties deterministically.}. 
Throughout this work, the quality of a predictor $h \in \cH$ 
at a sampled data point  $\vc{z} = (\x, y) \in \cZ$  
is measured by a $b$-bounded $\mathsf{L}_{\mathrm{loss}}$-Lipschitz loss function 
$\ell : (\cH \times \cZ) \rightarrow \left[0,b\right]$. With these elements, the population risk of a hypothesis $R : \cH \rightarrow [0,b]$ is the expected loss it incurs on a randomly sampled data point, 
$R(h) := \expect_{\vc{z} \sim \cD_\cZ} \Big[ \ell \big(h, \vc{z}\big )\Big].$ 
The goal of supervised learning is to obtain a hypothesis with low risk.
While the true distribution $\cD_\cZ$ is unknown, we assume access to a training set $\samp_T =  \{\vc{z}_1,\ldots,\vc{z}_m \}$ such that $\z_i = (\x_i,y_i) \overset{\textrm{i.i.d}}{\sim} \cD_\cZ$, and 
we instead minimize the \textit{empirical risk}, i.e.
the average loss on the training sample $\samp_T$, i.e. 
$\hat{R} (h) :=
\frac{1}{m}\sum_{i=1}^m \ell \left(h, (\x_{i},y_{i})\right)$.
\addition{
We note two canonical choices of loss function for classification tasks: the zero-one loss $\ell^{[0/1]}$, and the margin loss $\ell^{\gamma}$ with threshold $\gamma > 0$ \cite{mohri2018foundations}. 
The zero-one loss is $1$ for incorrect prediction and zero otherwise. 
The margin loss $\ell^\gamma$ is based on a margin operator $\mathcal{M} : \cY' \times \cY \rightarrow \mathbb{R}$, $\mathcal{M}(\vt, y) := [\vt]_{y} - \max_{j \neq y} [\vt]_j$\footnote{The predicted label is correct if $\mathcal{M}(h(\x),y) \geq 0$.},
\begin{equation*}
\ell^\gamma (h, \z) := \min
\left\{1, \max\left\{0, 1 - \frac{\mathcal{M}(h(\x), y) }{\gamma}\right\}\right\}. 
\end{equation*}
The margin loss is $0$ only for correct prediction with sufficient margin $\mathcal{M}(h(\x), y) \geq \gamma$. 
The margin loss with threshold $\gamma>0$ is $\frac{2}{\gamma}$-Lipschitz w.r.t  
change in predictor output \cite{mohri2018foundations}.
}

The sensitivity of a predictor to changes in inputs or parameters is characterized by their global Lipschitz constants. 
For a predictor $h \in \cH$, we let $\mathsf{L}_{\mathrm{inp}}$ denote the maximal change in the output of the predictor $h$ upon a change in its input. 
Similarly, for a suitable norm defined on the hypothesis class $\cH$, we denote by $\mathsf{L}_{\text{par}}$ the global Lipschitz constant measuring the sensitivity of the output to changes in the parameters of the predictor. 
Formally, for all  $\x, \xtil \in \cX$ and all $\hat{h},h \in \cH$, we have that
\[\norm{h(\xtil) - h(\x)}_2 \leq \mathsf{L}_{\mathrm{inp}}\norm{\xtil - \x}_2, \qquad \norm{\hat{h}(\x) - h(\x)}_2 \leq \mathsf{L}_{\text{par}} \norm{\hat{h} - h}_\cH.\]

\subsection{Adversarial Robustness}\label{subsec: adv-rob}
To evade test-time adversarial attacks,
we seek predictors that are 
robust to adversarial corruptions in the bounded set $\cB^{\cX}_{\nu}(\mathbf{0}) := \{\bdel \in \cX ~|~ \norm{\bdel}_{2}\leq \nu\}$.
The robust loss $\ell_{\mathrm{rob}}  (h, \z) := \max_{\bdel \in \cB^{\cX}_{\nu}(\mathbf{0})}  ~ \ell \big(h, (\x+\bdel,y) \big)$, 
captures the quality of a predictor $h$ under an attack.
We term the population (resp. empirical) risk evaluated on the robust loss as the robust population (resp. empirical) risk,  
\begin{equation*}
R_{\mathrm{rob}}(h) := \expect_{\vc{z} \sim \cD_\cZ} \Big[ \ell_{\mathrm{rob}} \big(h, \vc{z}\big )\Big], 
\quad 
\hat{R}_{\mathrm{rob}} (h) :
= \frac{1}{m}\sum_{i=1}^m \ell_{\mathrm{rob}} \Big(h, \z_i\Big).
\end{equation*}
In this case, the 
\emph{robust} global Lipschitz constant $\Lparnu$ measures parameter sensitivity on corrupted inputs, 
\begin{equation*}
\forall~h, \hat{h} \in \cH, ~\forall~\x \in \cX, \quad  \max_{\bdel \in \cB^{\cX}_{\nu}(\mathbf{0})}~\norm{\hat{h}(\x+\bdel) - h(\x+\bdel)} \leq \mathsf{L}_{\mathrm{par},\nu} \norm{\hat{h} - h}.
\end{equation*}
\addition{Note since $(\x+\bdel)$ might not be in the original input domain $\cX$, $\Lparnu$ can differ from $\Lpar$.}

In this work, we focus on two central problems: \emph{Certified Robustness}, providing a guarantee that a predictor $h$, that correctly classifies an input $\x$, 
will not be changed if contaminated with an adversarial perturbation of bounded norm $\norm{\bdel}_2\leq \bar{r}(\x)$; and \emph{Robust Generalization}, seeking to understand when a predictor $h$ learned on a collection of samples $\samp$ can generalize to corrupted unseen data, i.e. when $R_{\mathrm{rob}}(h)$ is low. \addition{Lastly, while we focus on the widely studied constraint set of $\ell_2$-bounded perturbations \cite{Szegedy2014IntriguingPO}, 
most of the derived analysis is directly extendable to general $\ell_p$ norms, and we will comment on these whenever relevant.}


\subsection{Representation-Linear Hypothesis class} \label{subsec: rep-linear-hyp}
In this work we consider a class of structured hypotheses $\cH_{\cA,\cW}$ called \emph{representation}-\emph{linear} hypotheses, with parameters $(\A,\W)$,
where classification weights $\A \in \cA \subset \mathbb{R}^{C \times p}$ act upon
a learned representation map ${\Phi}_{\W} : \cX \rightarrow \mathbb{R}^p$
with representation weights $\W \in \cW$,
\begin{equation*}
\cH_{\cA, \cW} := \{h_{\A,\W} : \cX \rightarrow \mathbb{R}^C, \;  h_{\A,\W}(\x) := \A \Phi_{\W}(\x), \; \forall \; \A \in \cA \text{ and } \W 
\in \cW\}.
\end{equation*}
The parameters of each hypothesis $(\A,\W)$ are learned based on sample data $\samp_T$. 
We assume that the representation space (image of $\Phi_{\W}$) is embedded with the Euclidean norm $\norm{\cdot}_{2}$\footnote{We assume that the norm $\norm{\cdot}_\cA$ is sub-multiplicative and consistent with the Euclidean vector norm $\norm{\cdot}_2$.}.
Naturally, each choice of a representation map $\Phi$ results in a corresponding hypothesis class. 
\addition{
For the discussion in this paper, we assume a consistent choice of parameterizing functions in $\cH$, thus functions with different parameters are considered to be different.} 
For the sake of simplicity, we denote $\cH = \cH_{\cA,\cW}$, $h(\x) = h_{\A,\W}(\x)$ and $\Phi(\x) = \Phi_{\W}(\x)$
when clear from the context. 
We discuss common representation maps.

\paragraph{Linear Representations} 
The simplest case is that of representation maps that are linear; i.e. $\Phi(\x) = \W\x$, for some $\W\in\cW \subset \mathbb R^{p\times d}$. 
One could require additional structure,  such as taking $p<d$ and taking $\cW$ as the set of projection matrices, thus computing low dimensional projections of the data. 
Linear low dimensional representations have been shown beneficial in the context of adversarial robustness \cite{Awasthi2020AdversarialRV}. 

\paragraph{Supervised Sparse Coding} 
Here for a dictionary $\W \in \cW_{\mathrm{rip}}$\footnote{$\cW_{\text{rip}}$ is the oblique manifold of matrices with unit-norm columns 
satisfying the restricted isometry property. A matrix $\W$ is RIP with constant $\eta_s$ if this is the smallest value so that, for any $s$-sparse vector $\bm{\alpha} \in \mathbb{R}^p : \|\bm{\alpha}\|_0=s$, $\W$ is close to an isometry, i.e. $(1-\eta_s)\norm{\bm{\alpha}}^2_2 \leq \norm{\W \bm{\alpha}}^2_2 \leq (1+\eta_s)\norm{\bm{\alpha}}_2^2$.
}, the representation map $\Phi$ encodes input $\x$ as 
$\Phi (\x) \in \underset{\bm{\alpha} \in \mathbb{R}^p}{\arg \min}~ \frac{1}{2}\norm{\x - \W \bm{\alpha}}^2_2 + \lambda \norm{\bm{\alpha}}_1$, the solution to a LASSO problem \cite{tibshirani1996regression}. 
In this case, the representation $\Phi$ is nonlinear and encourages sparsity.
This hypothesis class denoted $\cHssc$, is termed Task Driven Dictionary Learning in \cite{mairal2011task}, and frequently used in computer vision \cite{diamant2016task}, and analyzed in the context of adversarial robustness in \cite{sulam2020adversarial}.

\paragraph{Feedforward Neural Networks}
The representation map implemented by a depth-$(K+1)$ neural network, $\Phi^{[K]}$, is a sequence of $K$ affine maps
composed with a nonlinear activation function. The most common choice for this activation is the Rectifying Linear Unit, or ReLU, $\sigma : \mathbb{R} \to \mathbb{R}^{\geq 0}$, defined by  $\sigma(x) = \max\{x,0\}$, acting entry wise in an input vector.
Each layer\footnote{Recall that superscripts in parameters and variables for neural network classes index the layer.} has a weight $\W^k \in \cW^k \subset \mathbb{R}^{d^k\times d^{k-1}}$ and bias $\vb^k \in \mathfrak{B}^k \subset \mathbb{R}^{d^k}$. 
For this family of functions, the representation map $\Phi^{[K]}$ has parameters $\{\W^{k}, \vb^k\}_{k=1}^K$
and is formally given by 
\[
\Phi^{[K]}(\x) 
:= \act{\W^K \act{\W^{K-1} \cdots \act{\W^1\x + \vb^1} \cdots +\vb^{K-1}} + \vb^K}.
\]
We denote the hypothesis class $\HNNK$ with parameter space $\cA\times \prod_{k=1}^K (\cW^k\times \mathfrak{B}^k)$.


\section{Contributions}\label{sec: contrib}
In this work we present results for robustness certificates and generalization guarantees based on a novel tool we term \textit{sparse local Lipschitzness} (SLL).
SLL measures the \emph{local} sensitivity of a predictor $h$, while additionally requiring the preservation of sparsity patterns in the representation $\Phi$ within a \emph{sparse local radius}. 
The additional structural constraint enables us to express any SLL predictor by an equivalent \emph{simpler} function with fewer \emph{active} degrees of freedom, or parameters. Importantly, our definition of SLL is flexible, allowing for any degree of sparsity levels and for a controlled, ``tunable'' trade-off between sparsity and local sensitivity, as measured by the sparse local radius. SLL predictors are a subclass of local Lipschitz predictors and include common representation-linear hypothesis classes such as the ones mentioned previously. 

We present a certified radius for any SLL predictor w.r.t. input that improves on standard local (and global) analysis (see \Cref{lemma:cert-rob-gen-loc} and \Cref{corollary:cert-rob-fnn-loc}). 
\addition{Compared to traditional Lipschitz analysis,} we then demonstrate a tighter data-dependent non-uniform bound on the robust generalization error for predictors that are (robust) SLL w.r.t. parameters (see \Cref{thm: generalization-sll}, \Cref{app: thm: robust-generaalization-sll}, \Cref{corollary: nonuniformriskriskMNN}).
Our bounds depend mildly on the power of the adversary: the adversarial energy \addition{$\nu$} only impacts the \textit{fast} term in the upper bound \addition{i.e. $\mathcal{O}\left(\frac{\nu}{m}\right)$}, an improvement from recent results \cite{Awasthi2020AdversarialLG, yin2018rademacher} \addition{which are $\mathcal{O}\left(\frac{\nu}{\sqrt{m}}\right)$}. 

We instantiate these results for deep neural networks as a particular case. 
\Cref{fig:500_500_activity_levels} shows that for a trained neural network, at any test input, the number of active neurons in each layer is \emph{at most} half the width. Here, as the corruption to inputs increases in energy, the number of neurons that flip activation states increases. Our analysis quantifies the stability of the strongly inactive neurons at each input to corruptions in input and or parameters. For both certified robustness and robust generalization, this highlights the reduced dimensionality of the predictor to small corruptions. 
\begin{figure}[]
    \centering
    \includegraphics[height=2.6in]{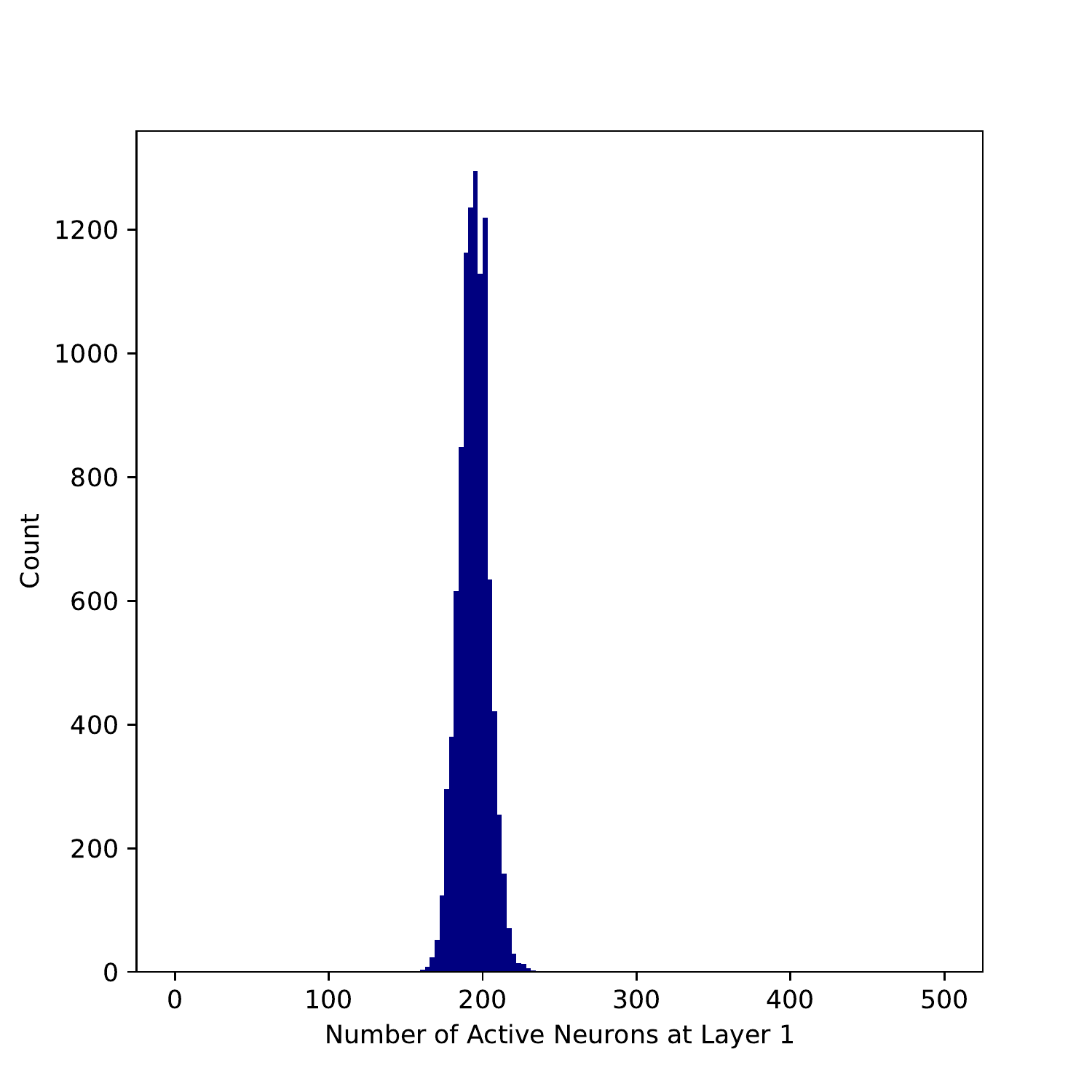}
    \includegraphics[height=2.6in]{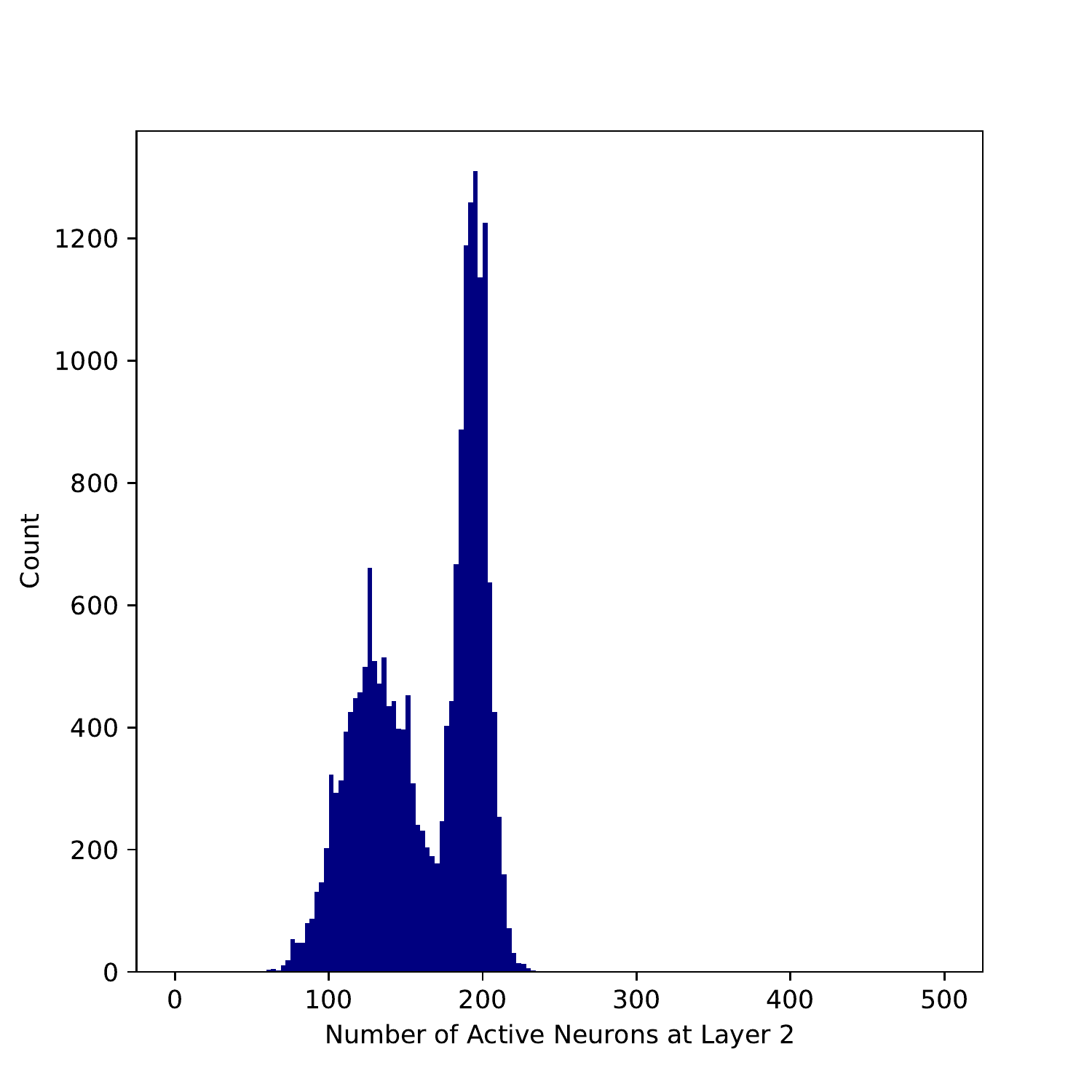}
    \caption{Distribution of active neurons \addition{across validation data} in each layer of a 2-layer \addition{feedforward} network trained on MNIST.}
    \label{fig:500_500_activity_levels}
\end{figure}
We show that the sparse local radius at each layer, for either input or parameter sensitivity, depends on the alignment between the pre-activation layer inputs and the rows of the layer weight matrix.  
For input sensitivity, we show that the sparse local Lipschitz scale of a depth-$(K+1)$ network is the product of operator norms of reduced linear maps at each layer, resulting in $2-3$ fold improvement over the global Lipschitz constant for typical settings (see \Cref{sec:experimental}).
In turn, for parameter sensitivity, we demonstrate that the sparse local Lipschitz scale is given by an upper bound on the operator norms of any reduced linear maps by incorporating a measure of coherence between rows. This sparse local Lipschitz scale has thus a better dependence on depth than the global Lipschitz constant. \addition{Finally, we note that while other neural network architectures are part of the Representation-Linear Hypothesis class, and can be shown to be sparse local Lipschitz functions (e.g. for convolutional networks), we refrain from instantiating our results for these architectures in this work and defer these interesting questions to future work.}


\section{Certified Robustness}
\label{sec:cert-rob}
For a fixed predictor $h$ (potentially data dependent), input $\x$ and perturbation $\bdel$, we denote by $\hat{y}(\x)$ and $\hat{y}(\x+\bdel)$ the predicted labels before and after corruption, respectively.
A point-wise certified radius function $r_{\mathrm{cert}}: \cX \rightarrow \mathbb{R}^{\geq 0}$ is a guarantee that for any bounded perturbation $\bdel$ the predicted label remains unchanged, that is
\[
\norm{\bdel}_2 \leq r_{\mathrm{cert}}(\x) \implies \hat{y}(\x+\bdel) = \hat{y}(\x).
\]
We develop a point-wise certified radius that relies on the local sensitivity of the predictor $h$.

\subsection{Sparse Local Lipschitz w.r.t. Inputs}\label{subsec:SLL-input}

We start by characterizing a class of functions that preserve sparsity in their output for bounded perturbations to inputs. 
\addition{Throughout this manuscript, we refer to sparsity as \textit{the number of zero entries} of a vector. 
} We say that a vector $\vt \in\mathbb R^d$ is $s$-sparse if it has an inactive set $I$ of size $s$; that is, if there exists $I$ so that $\mathcal{P}_{I}(\vt) = \mathbf{0} \in \mathbb R^{s}$.
Naturally, $\vt$ is $s$-sparse only when $s\leq d-\|\vt\|_0$. 
With these elements, we are now ready to define the sparse local Lipschitzness of a function.

\begin{definition}\text{(Sparse Local Lipschitzness {\it w.r.t.\;}Input)}
\\
\label{def: lip-input}
Let $\x$ and $\Phi(\x)$ be $s_{in}$ and $s_{out}$ sparse, respectively, and let $\vc{s} = (s_{in}, s_{out})$. The representation map $\Phi$ is $\vc{s}$-\textit{sparse local Lipschitz at $\x$ (w.r.t. inputs)}  if there exists an index set $I_{out}$ of size $s_{out}$, a local radius $r \geq 0$ and Lipschitz scale $l\geq 0$ such that for any perturbed input $\xtil \in \cB^{\cX}_{r}(\x)$ with a common inactive set $I_{in}$ of size $s_{in}$, i.e. $\mathcal{P}_{I_{in}}(\xtil) = \mathcal{P}_{I_{in}}(\x) = \mathbf{0}$, one has that 
\[\norm{\Phi(\xtil) - \Phi(\x)}_2 \leq l \norm{\xtil-\x}_2 ~ \land ~ \mathcal{P}_{I_{out}}(\Phi(\xtil)) = \mathcal{P}_{I_{out}}(\Phi(\x)) = \mathbf{0}.\]
\end{definition}

In words, $\Phi$ is sparse local Lipschitz at $\x$ if, for points in a neighborhood of $\x$ that preserve a certain input sparsity pattern, the function is local Lipschitz \emph{and} preserves a certain representation sparsity pattern. 
It is important to note that the sparse local sensitivity \addition{i.e. the trio of index set $I_{\text{out}}$, radius $r$ and Lipschitz scale $l$} are dependent on the specific input $\x$ as well as the sparsity levels $\vc{s}$ which can range through all possible sizes of index sets in the inputs and in the representations, i.e. 
$\vc{s} \in \mathcal{S} := [d] \times [p]$. However, $\Phi$ will only be SLL at $\x$ for $\vc{s} \in \{0,\dots, d-\|\x\|_0\} \times \{0,\dots, p-\|\Phi(\x)\|_0\}$. 
More generally, for a fixed representation map $\Phi$ we can define a local radius function $r_{\mathrm{inp}} : \cX \times \mathcal{S} \rightarrow \mathbb{R}^{\geq 0}$ and local Lipschitz scale functions\footnote{
For any input $\x$ and a choice of sparsity levels $\vc{s} = (s_{in}, s_{out}) \in \mathcal{S}$ such that $s_{in} > d-\norm{\x}_0$ or $s_{out} > p-\norm{\Phi(\x)}_0$, for simplicity we let the corresponding radius $\rinp(\x, \vc{s}) := 0$ and local Lipschitz scale $\linp(\x, \vc{s}):=\infty$. 
}, 
$l_{\mathrm{inp}} : \cX \times \mathcal{S} \rightarrow \mathbb{R}^{\geq 0}.$ 
and extend our definition. 

\begin{definition}{(Sparse Local Lipschitz function)}\\
\label{def: sll-fun}
We say that the representation map $\Phi$ is \textit{sparse local Lipschitz w.r.t. inputs} if for any $\x \in \cX$, for all appropriate\footnote{Valid sparsity levels are in $\{0,\dots, d-\|\x\|_0\} \times \{0,\dots, p-\|\Phi(\x)\|_0\}$.} sparsity levels $\vc{s}$,
$\Phi$ is $\vc s$-sparse local Lipschitz at $\x$ with associated radius $\rinp(\x,\vc{s})$ and local Lipschitz scale $\linp(\x,\vc{s})$.
\end{definition}
In this way, if $\Phi$ is sparse locally Lipschitz w.r.t inputs, for an appropriate $\xtil$,
\[
\norm{\xtil - \x}_2 \leq \rinp(\x, \vc{s}) \implies \norm{\Phi(\xtil) - \Phi(\x)}_2 \leq \linp(\x, \vc{s}) \norm{\xtil - \x}_2.
\]
Note that SLL representations $\Phi$ are also local Lipschitz with radius $\rinp(\x, \vc{0})$ and local Lipschitz scale $\linp(\x, \vc{0})$. 
Additionally, if $\rinp(\x, \vc{0}) = \infty$ for all $\x \in \cX$, then the representation is global Lipschitz with constant $\max_{\x \in \cX} \linp(\x, \vc{0})$.  
\addition{There could be multiple radius functions $\rinp$ and $\linp$ that meet the requirements of \Cref{def: lip-input,def: sll-fun}. For each hypothesis class, we assume a fixed choice of these functions.} 
For representation-linear hypothesis classes $\mathcal H$, it is natural to couple this notion of sensitivity with the classification weight $\A$. More precisely, under appropriate conditions on $\xtil$, we can have that
\[
\norm{h(\xtil)-h(\x)}_2 = \norm{\A \Phi(\xtil) - \A \Phi(\x)}_2 \leq \norm{\A}_2 l_{\mathrm{inp}}(\x, \vc{s}) \cdot \norm{\xtil -\x}_2.
\]
We are now ready to present our first result that develops a certified radius for SLL predictors.

\begin{theorem}\label{lemma:cert-rob-gen-loc}
Consider a predictor $h$ with classification weight $\A$, sparse local Lipschitz representation map $\Phi$ and a fixed sparsity level $\vc{s}=(0, s_{\mathrm{out}})$. 
Let $h$ classify an input $\x$ as label $\hat{y}(\x)$ with classification margin $\mathcal{M}(h(\x), \hat{y}(\x)) > 0$. 
Upon bounded perturbations $\bdel$, the predicted label of the classifier remains unchanged, i.e. $\hat{y}(\x+\bdel) = \hat{y}(\x)$ when 
the energy of the perturbation is below the certificate, $\norm{\bdel}_2 \leq \rcert(\x, \vc{s})$ where, 
\begin{equation*}
	\rcert(\x, \vc{s}) := \min \left\{r_{\mathrm{inp}}(\x,  \vc{s}) ,~  \frac{\mathcal{M}(h(\x), \hat{y}(\x))}{2 \norm{\A}_2 \cdot l_{\mathrm{inp}}(\x,\vc{s}) } \right \}.
\end{equation*}
\end{theorem}
\addition{\begin{proof}
Observe that the predicted labels remain unchanged when
$\mathcal{M}(h(\x+\bdel), \hat{y}(\x)) \geq 0$.
Since the margin operator $\mathcal{M}(\cdot, j)$ is $2$-Lipschitz in $\cYd$, thus $\mathcal{M}\big(h(\x),\hat{y}(\x)\big) - \mathcal{M}\big(h(\x+\bdel), \hat{y}(\x)\big) \leq 2 \norm{h(\x+\bdel) - h(\x)}_p$.
On the other hand, note that for $\vc{s} = (0, s_{\mathrm{out}})$ there are no sparsity constraints on the perturbed input. 
Hence,
\[
\norm{\bdel}_2 \leq \rinp(\x,\vc{s})
\implies
\norm{h(\x+\bdel)-h(\x)}_2 \leq \norm{\A}_2 \cdot \linp(\x, \vc{s}) \norm{\bdel}_2.
\]
As a result, $\mathcal{M}(h(\x), \hat{y}(\x)) - \mathcal{M}\big(h(\x+\bdel), \hat{y}(\x)\big)
\leq 2 \norm{\A}_2 \cdot \linp(\x, \vc{s}) \norm{\bdel}_2.$
\end{proof}}

The first constraint on the certificate ensures 
that the perturbation is within the local radius. 
The second constraint, on the other hand, ensures that the effect of the perturbation does not exceed the classification margin in the representation space. In the second constraint, the distance between the original and perturbed representation is bounded using the local Lipschitz scale. The input sparsity level $s_{in}$ is naturally fixed to $0$ to ensure that all $\ell_2$ bounded perturbations are covered in the certified radius guarantee.
This theorem encompasses local or global Lipschitz representations as the special case when $\vc{s}=\mathbf{0}$ (thus, it is more flexible than global analyses), and it is a generalization of the result in \cite[Proposition 1]{tsuzuku2018lipschitz}. 

\subsubsection{Composition of sparse local Lipschitz functions}\label{subsubsec: composition}
Useful representation maps can often be obtained as the composition  of several intermediate maps, as in the case of feedforward neural networks and multi-layered sparse coding \cite{sulam2019multi}.
More generally, consider $K$ intermediate layer representation maps $\Phi^{(k)} : \mathbb{R}^{d^{k-1}} \rightarrow \mathbb{R}^{d^k}$ for $1\leq k \leq K$, which are then composed to obtain $\Phi^{[K]}$, 
\begin{equation}
\label{eq:multilayer-composition}
\Phi^{[K]}(\x) := \Phi^{(K)} \circ \Phi^{(K-1)} \circ \cdots \circ \Phi^{(1)} \left(\x\right).
\end{equation}
Let $(s^0, s^1, \ldots, s^K)$ denote an appropriate choice of sparsity level for each intermediate representation from $\x$ to $\Phi^{[k]}(\x)$\footnote{That is, one where $s^k \leq d^k-\|\Phi^{[k]}(\x)\|_0$.}.
By defining the layer-wise input-output sparsity levels $\vc{s}^{(k)} := (s^{k-1}, s^k)$, and the cumulative input-output sparsity levels $\vc{s}^{[k]} := (s^0, s^1, \ldots, s^k)$, we now show one can compose sparse local Lipschitz functions to obtain a function of the same class. 

\begin{lemma}
\label{lemma: sparse-local-lip-composition}
Assume that each $\Phi^{(k)}$ in \Cref{eq:multilayer-composition} is SLL w.r.t. inputs with local radius functions $\rinp^{(k)}$ and local Lipschitz scale $\linp^{(k)}$.
Then the composed representation maps $\Phi^{[k]}$ are also SLL w.r.t. inputs with local radius functions $\rinp^{[k]}$ and local Lipschitz scale $\linp^{[k]}$ given by
\begin{equation*}
    r_{\mathrm{inp}}^{[k]}\Big(\x,\;  \vc{s}^{[k]}\Big) := \min_{1\leq n \leq k}  \frac{r_{\mathrm{inp}}^{(n)}\Big(\Phi^{[n-1]}(\x),\; \vc{s}^{(n)} \Big) }{\displaystyle l_{\mathrm{inp}}^{[n-1]} \Big(\x,\;  \vc{s}^{[n-1]} \Big)},
    ~\quad~
    l_{\mathrm{inp}}^{[k]}\Big( \x,\; \vc{s}^{[k]} \Big) := \prod_{n=1}^k l_{\mathrm{inp}}^{(n)}\Big(\Phi^{[n-1]}(\x),\; \vc{s}^{(n)}\Big). 
\end{equation*}
\end{lemma}

\addition{\textit{Proof Sketch.} 
For the base case $k=1$, by definition, $r_{\mathrm{inp}}^{[1]}\Big(\x,\;  \vc{s}^{[1]}\Big)  = r_{\mathrm{inp}}^{(1)}\Big(\x,\;  \vc{s}^{(1)}\Big)$ and $l_{\mathrm{inp}}^{[1]}\Big( \x,\; \vc{s}^{[1]} \Big) = l_{\mathrm{inp}}^{(1)}\Big( \x,\; \vc{s}^{(1)} \Big)$.
Consider the case when $k=2$, $\Phi^{[2]}(\x) := \Phi^{(2)} \circ \Phi^{(1)}(\x)$. 
Consider a perturbation $\bdel$ in the initial input. By the definition of SLL, if $\norm{\bdel}_2 \leq r_{\mathrm{inp}}^{[1]}\Big(\x,\; \vc{s}^{[1]} \Big)$, then, $\norm{\Phi^{[1]}(\x+\bdel)-\Phi^{[1]}(\x)}_2 \leq l_{\mathrm{inp}}^{[1]}\Big(\x,\; \vc{s}^{[1]} \Big) \norm{\bdel}_2.$ Note that the perturbation in the first layer outputs $\Phi^{[1]}(\x+\bdel)-\Phi^{[1]}(\x)$ is a perturbation in the second layer inputs to the map $\Phi^{(2)}$. 
Hence, if 
$\norm{\Phi^{[1]}(\x+\bdel)-\Phi^{[1]}(\x)}_2 \leq
r_{\mathrm{inp}}^{(2)}\Big( \Phi^{[1]}(\x),\; \vc{s}^{(2)} \Big)$, then 
\begin{align*}
\norm{\Phi^{(2)} \circ \Phi^{[1]}(\x+\bdel)-\Phi^{(2)}\circ \Phi^{[1]}(\x)}_2 
&\leq l_{\mathrm{inp}}^{(2)}\Big(\Phi^{[1]}(\x),\; \vc{s}^{(2)} \Big) \norm{\Phi^{[1]}(\x+\bdel)-\Phi^{[1]}(\x)}_2 \\
&\leq l_{\mathrm{inp}}^{(2)}\Big(\Phi^{[1]}(\x),\; \vc{s}^{(2)} \Big) \cdot l_{\mathrm{inp}}^{(1)}\Big(\x,\; \vc{s}^{(1)} \Big) \norm{\bdel}_2. \\
&=: l_{\mathrm{inp}}^{[2]}\Big( \x,\; \vc{s}^{[2]} \Big) \norm{\bdel}_2.
\end{align*}
One can similarly extend this logic to the case $k > 2$.  
If each of the intermediate representation maps are SLL w.r.t. inputs, then by appropriately weaving the sparsity levels at each layer we can show that the composed representation map is also SLL. The complete detailed proof by induction can be found in \cref{app: lemma: sparse-local-lip-composition}. 
}

As before, local Lipschitz and global Lipschitz compositions are special cases of the result above. 
One can readily use this local Lipschitz scale function $\linp^{[K]}$ to obtain a certified radius as per \Cref{lemma:cert-rob-gen-loc} for functions that are compositions of sparse-local Lipschitz functions. 

\addition{
\subsubsection{Optimal Certified Radius}\label{subsubsec:optrad}
For any sparse local Lipschitz predictors, the certified radius in \Cref{lemma:cert-rob-gen-loc} at the trivial choice of sparsity $\rcert(\x,\vc{0})$ 
is computed using $\linp^{[K]}(\x,\mathbf{0})$, 
the global Lipschitz constant for $\Phi^{[K]}$. 
Increasing the sparsity vector $\vc{s}$ entry wise can result in a smaller local Lipschitz scale $\linp^{[K]}(\x, \vc{s})$ at the expense of a smaller local radius $\rinp^{[K]}(\x, \vc{s})$. 
Hence for a given input $\x$, the best robustness certificate $r^*(\x)$ is generated by a specific choice of  sparsity level $\vc{s}^*(\x)$ that achieves low Lipschitz scale in a sufficiently large neighborhood,
\begin{equation}
\label{eq:optimalsparse}
\vc{s}^*(\x) := \argmax_{\vc{s}}~ \rcert(\x, \vc{s}), \quad r^*(\x) := \rcert(\x, \vc{s}^*(\x)).
\end{equation}
The complexity of this optimization problem will depend on the specific hypothesis class, and on the function $r_\text{cert}$. For a composition of SLL predictors, the number of feasible sparsity levels in \Cref{eq:optimalsparse} is $\mathcal{O}\left(\prod_{k=1}^{K} d^k\right)$, exponential in the number of intermediate maps $K$. 
Thus, rather than search for the optimal sparsity vector $\vc{s}^*(\x)$, we propose to approximate this solution (sometimes exactly) by a binary search over the space of the certified radius instead. More precisely, it is easy to see that $r^*(\x)\in[0,\|\x\|_2]$. Thus, if one has access to an algorithm $\mathcal{A}(\x,\nu)$ that can return, for any given input and energy level $\nu$, an appropriate sparsity vector $\hat{\vc{s}}$ so that $\nu \leq \rinp^{[K]}(\x, \hat{\vc{s}})$, one can implement a binary search over $\nu\in[0,\|\x\|_2]$ by checking if 
\begin{equation}\label{eq:condition_safe}
    \quad\nu \leq  \frac{\mathcal{M}(h(\x), \hat{y}(\x))}{2 \norm{\A}_2 \cdot l^{[K]}_{\mathrm{inp}}(\x,\hat{\vc{s}})}.
\end{equation}
If this is satisfied, such a level of sparsity is safe. This allows us to carry out a binary refinement over $\nu$ until a tolerance level, $tol$, is satisfied, reducing the complexity of this search to $\mathcal{O}
\left(\log_2\left(\frac{\norm{\x}_2}{tol}\right)\right)$\footnote{Note we suppress the complexity of the algorithm $\mathcal{A}(\x, \nu)$, which will depend on the specific hypothesis class, and refer simply to the search complexity.}. Naturally, the quality of this solution depends on the algorithm $\mathcal{A}(\x,\nu)$. We will see that for functions that satisfy a notion of \textit{monotonicity}, such an algorithm can be easily instantiated and can be, in a specific sense, optimal. }
\addition{
\begin{definition}(Monotone Ordering)\label{def: monotone-ord}
A sparse local Lipschitz representation map $\Phi$ 
with radius function $\rinp$ and Lipschitz scale $\linp$
is said to have monotone ordering if 
\begin{align}
\text{Lipschitz condition :}~ & 
(s_1^{\text{in}}, s_1^{\text{out}}) 
\preceq 
(s_2^{\text{in}}, s_2^{\text{out}}) 
\implies
\linp(\x, (s_2^{\text{in}}, s_2^{\text{out}}) ) \leq \linp(\x, (s_1^{\text{in}}, s_1^{\text{out}}) ) \\
\text{Radius condition 1:}~~ & 
s_1^{\text{in}} \leq s_2^{\text{in}}, \; \forall\; s \in [d^{\text{out}}],~ \rinp(\x, (s_1^{\text{in}}, s)) \leq \rinp(\x, (s_2^{\text{in}}, s)) 
\\
\text{Radius condition 2:}~~ & 
s_1^{\text{out}} \leq s_2^{\text{out}}, \; \forall\; s \in [d^{\text{in}}],~ \rinp(\x, (s, s_1^{\text{out}})) \geq \rinp(\x, (s, s_2^{\text{out}})) 
\end{align}
\end{definition}
\Cref{alg: greedy-sparse} implements the rule $\mathcal A(\x,\nu)$ described above. This algorithm is correct, as now make precise.}
\begin{algorithm}
\caption{$\mathcal{A}(\x,\nu)$ to generate valid sparse vector $\hat{\vc{s}}$ for $\Phi^{[K]} = \Phi^{(K)} \circ \cdots \circ \Phi^{(1)}$.}
\label{alg: greedy-sparse}
\begin{algorithmic}
\State \textbf{Require} : Input $\x$ and perturbation energy level $\nu$. 
\State \textbf{Require} : Intermediate sparse local radius $\rinp^{(k)}$ and local Lipschitz scale $\linp^{(k)}$ $\forall\; k\in [K]$.
\State \textbf{Ensure} : Sparsity vector $\hat{\vc{s}}$ such that $\nu \leq \rinp^{[K]}(\x, \hat{\vc{s}})$.
\State \textbf{Initialize} : $\hat{\vc{s}} := \{0, \ldots, 0\} \in \mathbb{R}^{K+1}$.
\State \textbf{Initialize} : Perturbation level, $\hat{\nu}^0 := \nu$.
\For {k=1 to K}
\State $\hat{s}^k \gets$ maximum $s$ such that $\hat{\nu}^{k-1} \leq \rinp^{(k)}\left(\Phi^{[k-1]}(\x), \{\hat{s}^{k-1}, s\}\right)$.
\State $\hat{\nu}^{k} \gets \hat{\nu}^{k-1} \cdot \linp^{(k)}\left(\Phi^{[k-1]}(\x), \{\hat{s}^{k-1}, \hat{s}^k\}\right)$.
\EndFor
\State \textbf{Return} : $\hat{\vc{s}} = \{0, \hat{s}^1, \ldots, \hat{s}^K\}$.
\end{algorithmic}
\end{algorithm}
\addition{ \begin{lemma}
The sparsity vector $\hat{\vc{s}}$ generated by \Cref{alg: greedy-sparse} is correct, i.e. $\rinp^{[K]}(\x, \hat{\vc{s}}) \geq \nu$. 
If $\Phi^{[K]}$ is such that each intermediate map $\Phi^{(k)}$ has monotone ordering as per \Cref{def: monotone-ord} then, \Cref{alg: greedy-sparse} is maximal, i.e. for any sparsity vector $\bar{\vc{s}}$,
$
\nu \leq \rinp^{[K]}(\x, \bar{\vc{s}}) \implies 
\bar{\vc{s}} \preceq \hat{\vc{s}}
$. 
Additionally under monotone ordering, if $\nu$ is deemed unsafe, i.e. \cref{eq:condition_safe} is not satisfied, then there exists no sparsity vector $\bar{\vc{s}}$ such that $\nu \leq \rcert(\x, \bar{\vc{s}})$.
\end{lemma}
\begin{proof}
Define $\hat{\vc{s}}$ and $\hat{\nu}^k := \nu \cdot \linp^{[k]}(\x, \hat{\vc{s}})$ as in \Cref{alg: greedy-sparse}.
    For each layer $k$, the choice of $\hat{s}^k$ ensures that,
    $\hat{\nu}^{k-1} = \nu \cdot \linp^{[k-1]}(\x, \{0,\hat{s}^1, \ldots, \hat{s}^{k-1}\}) \leq \rinp^{[k]} (\x, \{0,\hat{s}^1, \ldots, \hat{s}^k\})$ thus ensuring correctness.
    Consider any sparsity vector $\bar{\vc{s}}$ such that $\nu \leq \rinp^{[K]}(\x, \bar{\vc{s}})$ and let $\bar{\nu}^k := \nu \cdot \linp^{[k]}(\x, \bar{\vc{s}})$. We aim to show that $\bar{\vc{s}} \leq \hat{\vc{s}}$ necessarily. 
    For $k=1$, $\nu \leq \rinp^{[1]}(\x, \{0, \bar{s}^1\})$,  and by construction, 
    \[
    \hat{s}^1 := \max s \text{ such that } \nu \leq \rinp^{(1)}(\x, \{0, s\}).
    \]    
    Hence $\bar{s}^1 \leq \hat{s}^1$. Furthermore, by the monotone ordering assumption, 
    \begin{equation}\label{eq: bar-s}
    \hat{\nu}^1\; \overset{(a)}{\leq} \;\bar{\nu}^1\; \overset{(b)}{\leq}\; \rinp^{(2)}(\Phi^{[1]}(\x), \{\bar{s}^1, \bar{s}^2\}) \;\overset{(c)}{\leq}\; \rinp^{(2)}(\Phi^{[1]}(\x), \{\hat{s}^1, \bar{s}^2\}) 
    \end{equation}
    In the above statement, (a) follows from monotone ordering of Lipschitz scale, (b) follows from assumption on $\bar{\vc{s}}$ and (c) follows from monotone ordering of the radius w.r.t the input sparsity level. 
    Further $\hat{s}^2$ is defined as the maximum sparsity level such that, 
    \[
    \hat{s}^2 := \max s \text{ such that } \hat{\nu}^1 \leq \rinp^{(2)}(\Phi^{[1]}(\x), \{\hat{s}^1, s\}).
    \]
    From \Cref{eq: bar-s}, $\bar{s}^2$ is also feasible for the above optimization and hence, $\bar{s}^2 \leq \hat{s}^2$. 
    One can repeat these arguments until layer $K$ to show that $\bar{\vc{s}} \leq \hat{\vc{s}}$. Hence \Cref{alg: greedy-sparse} chooses the maximal sparsity vector.     
    Now, consider the set of sparsity vectors $\mathcal{S}_{\mathrm{good}}$ such that $\nu \leq \rinp^{[K]}(\x, \vc{s})$.
    If there exists a witness vector $\vc{s}$ such that $\rcert(\x, \vc{s}) \geq \nu$, then certainly $\vc{s} \in \mathcal{S}_{\mathrm{good}}$.  
    We have just shown than for all $\vc{s} \in \mathcal{S}_{\mathrm{good}}$, we have $\vc{s} \preceq \hat{\vc{s}}$ and hence by the monotone ordering property, $\linp^{[K]}(\x, \hat{\vc{s}}) \leq \linp^{[K]}(\x, \vc{s})$ and thus, $\rcert(\x, \vc{s}) \leq \rcert(\x, \hat{\vc{s}})$. Therefore if $\rcert(\x, \hat{\vc{s}}) < \nu$ then for all $\vc{s} \in \mathcal{S}_{\mathrm{good}}$ we necessarily have $\rcert(\x, \vc{s}) <\nu$ and the conclusion follows.
\end{proof}
In what follows we will study a classes of functions that have monotone ordering.
}
\subsection{Certified Robustness for Feedforward Neural Networks}
\label{subsec: cert-rob-fnn}
For the remainder of this section, we will refer to feedforward neural networks exclusively as the representation map $\Phi^{[K]}$, given by the composition of $K$ piece-wise affine maps
$\Phi^{(k)}(\vc{t}) := \act{\W^k \vc{t}+\vb^k}$. 
The presence of the $\relu$ activation in each feedforward map naturally encourages some degree of sparsity at each layer. 
For each feedforward map $\Phi^{(k)}$, we denote the inactive set of the representation at any intermediate input $\vc{t}\in \mathbb{R}^{d^{k-1}}$ by $\mathcal{I}^k(\vc{t}):= \{j \in [d^k] ~|~ \Wvec{k}{j} \vc{t} + \vb^k_j \leq 0\}$, while its complement contains the support of layer $k$, i.e. $\mathcal{J}^k(\vc{t}) = \text{Supp}(\Phi^{(k)}(\vc{t})) = [d^k]\backslash\mathcal{I}^k(\vc{t})$.
In what follows, we define index sets $I^k$ as subsets of the full inactive set $\mathcal{I}^k(\vc{t})$, i.e. $I^k\subseteq \mathcal{I}^k(\vc{t})$, and the corresponding index set $J^k=(I^k)^C$ as supersets of the support set, i.e. $J^k \supseteq \mathcal{J}^k(\vc{t}) $\footnote{For ease of notation in the following discussion we denote $J^{0}:=[d^0]$.}. Figure \ref{fig:supports_diagram}(a) illustrates these sets for a given layer $\Phi^{(k)}(\vc{t})$.  Later, in Section \ref{sec:experimental}, we will demonstrate typical levels of sparsity achieved by common feedforward models.

\begin{figure}
    \centering
    \includegraphics[trim = 0 0 0 20,clip, width=\textwidth]{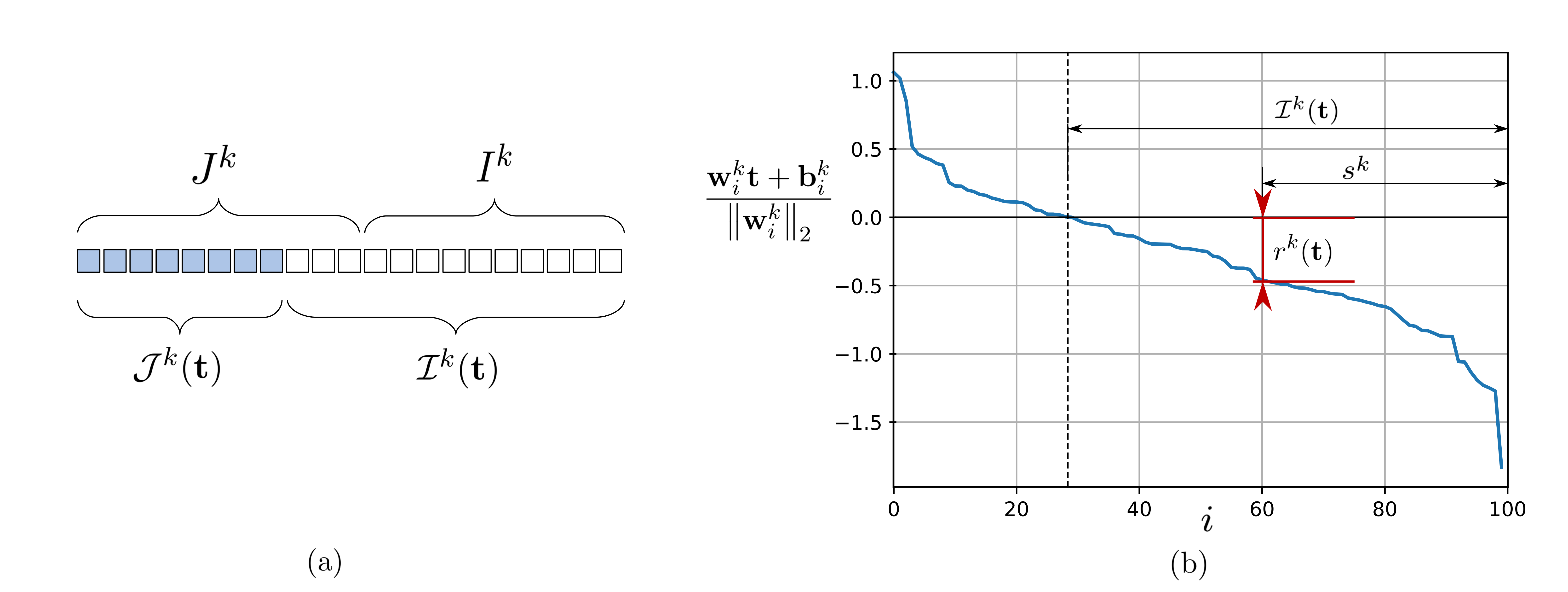}
    \caption{(a) Illustration of the sets $\mathcal{J}^k(\vc{t})$, $\mathcal{I}^k(\vc{t})$, as well as $I^k$ and $J^k$, for a given intermediate input $\sigma(\W^k \vc{t}+\vc{b}^k)$. Colored squares represent non-zero elements, ordered here without loss of generality. (b) Illustration of the radius $\rinp^{(k)}(\vc{t}, \vc{s}^{(k)})$ for a one layer neural network, given the (sorted) values of the normalized pre-activations.  }
    \label{fig:supports_diagram} 
\end{figure}


Since $\Phi^{(k)}$ is an affine map composed with a \relu operator, the map $\Phi^{(k)}$ is Lipschitz with constant $\|\W^k\|_2$. In the following lemma, we move beyond this global characterization and define its sparse local radius as the maximum energy of a perturbation $\bm{\gamma} \in \mathbb{R}^{d^{k-1}}$ under which there exists at least one common inactive index set $I^k$ of size $s^k$ for both $\Phi^{(k)}(\vc{t})$ and $\Phi^{(k)}(\vc{t}+\bm{\gamma})$. 
Along with the local radius function, we also define a specific inactive set in the output that is guaranteed to withstand any input perturbation.
 
\begin{lemma}
\label{lemma: fnn-layer-lip}
Let $\vc{s}^{(k)} = (s^{k-1}, s^k)$ denote the input-output sparsity levels for $\Phi^{(k)}$, so that  $s^{k-1}\leq d^{k-1}-\|\vc t\|_0$ and $s^{k}\leq d^k-\|\Phi^{(k)}(\vc t)\|_0$. The representation map $\Phi^{(k)}$ is SLL w.r.t. its input with stable inactive index set $I^k$, sparse local radius function $\rinp^{(k)}$ and sparse local Lipschitz scale function $\linp^{(k)}$ defined as, 
\begin{align*}
 \displaystyle I^{k}(\vt, \vc{s}^{(k)}) &:= \argmax_{\substack{I \subseteq \mathcal{I}^k(\vc{t}),\; |I| = s^k}}~ \min_{i \in I}~ \frac{\left \lvert \Wvec{k}{i} \vc{t} + \vb^k_i \right \rvert}{\norm{\Wvec{k}{i}}_2},  \quad 
 \rinp^{(k)}(\vc{t}, \vc{s}^{(k)}) := 
\min_{i \in I^k}~ \frac{\left \lvert \Wvec{k}{i} \vc{t} + \vb^k_i \right \rvert}{\norm{\Wvec{k}{i}}_2} \\
\linp^{(k)}(\vc{t}, \vc{s}^{(k)}) &:= \max_{\substack{I^{k-1} \subseteq \mathcal{I}^{k-1}(\vc{t}),\;\; |I^{k-1}|=s^{k-1}}}~ \norm{\mathcal{P}_{J^k, J^{k-1}}(\W^k)}_2.
  \end{align*}
where, $J^k = (I^k(\vt, \vc{s}^{(k)}))^c, J^{k-1} = (I^{k-1})^c$ are index sets of size $(d^k-s^k)$ and $(d^{k-1}-s^{k-1})$.
Further the feedforward map has monotone ordering. 
\end{lemma}
\addition{\begin{proof}
Fix any layer input $\vt$ and sparsity level $\vc{s}^{(k)}$. The set $I^k$ defined above is inactive for the representation $\Phi^{(k)}(\vt)$. Further by the definition of the local radius function $\rinp^{(k)}$, 
for each inactive index $i \in I^k$, $\Wvec{k}{i} \vc{t} + \vb^k_i \leq - \norm{\Wvec{k}{i}}_2\cdot \rinp^{(k)}(\vc{t}, \vc{s}^{(k)})$. 
Let $\vttil$ be a perturbed input that is within the sparse local radius $\|\vttil - \vt\|_2 \leq \rinp^{(k)}(\vc{t}, \vc{s}^{(k)})$ and further shares a sparsity pattern $\mathcal{P}_{I^{k-1}}(\vc t) = \mathcal{P}_{I^{k-1}}(\vttil) = \bm{0}$ for some $I^{k-1}$. 
We shall show that the set $I^k$ of size $s^k$ is also inactive for the perturbed representation $\Phi^k(\vttil)$, for any $i\in I^k$, 
\begin{equation*}
    \Wvec{k}{i} \vttil + \vb^k_i  
   = \big( \Wvec{k}{i}\vc{t} + \vb^k_i \big) + \Wvec{k}{i} (\vttil-\vt) \leq - \norm{\Wvec{k}{i}}_2\cdot \rinp^{(k)}(\vc{t}, \vc{s}^{(k)}) + \norm{\Wvec{k}{i}}_2 \norm{\vttil-\vt}_2 \leq 0.
\end{equation*}
Thus, we have shown the existence of a common inactive set for both the original and the perturbed representation. To bound the distance between the representations, note that 
\begin{equation*}
	\norm{\Phi^{(k)}(\vttil) - \Phi^{(k)}(\vc{t})}_2
	= \norm{ \act{\W^k\vttil + \vb^k} - \act{\W^k \vc{t}+\vb^k}}_2 
	\leq \norm{ \mathcal{P}_{J^k, J^{k-1}}(\W^k) \left(\vttil - \vc{t}\right)}_2. 
\end{equation*}
The second inequality above stems from ignoring common inactive sets.
To see that the $\rinp$ and $\linp$ have monotone ordering, for input sparsity levels $s_1^{k-1} \leq s_2^{k-1}$ and any fixed output sparsity level $s \in [d^k]$, we have, $I^k(\vt, (s_1^{k-1},s)) = I^k(\vt, (s_2^{k-1},s))$ by definition, and hence  
\begin{equation*}
\rinp(\vt, (s_1^{k-1},s)) = \rinp(\vt, (s_2^{k-1}, s)), \; \quad 
\linp(\vt, (s_1^{k-1},s)) \geq \linp(\vt,  (s_1^{k-1},s)). 
\end{equation*}
Similarly, for a fixed $s \in [d^{k-1}]$, and $s_1^{k} \leq s_2^{k}$, note that $I^k(\vt, (s,s_1^k)) \subseteq I^k(\vt, (s,s_2^k))$, and hence 
\begin{equation*}
\rinp(\vt, (s, s_1^{k})) \geq \rinp(\vt, (s, s_2^k)), \; \quad 
\linp(\vt, (s, s_1^{k})) \geq \linp(\vt,  (s, s_1^{k})). 
\end{equation*}
\end{proof}
}
It is worthwhile stressing the implications of the local radius function $\rinp^{(k)}$ for the intermediate feedforward layer $\Phi^{(k)}$. 
\addition{An analogous expression for this quantity can be obtained by considering the (normalized) vector $\vc{q}^k(\vt) := \left[ \frac{  \Wvec{k}{i} \vc{t} + \vb^k_i }{\norm{\Wvec{k}{i}}_2}  \right]_{i=1}^{d^k}$ of pre-activations,
\begin{equation}\label{eq: normalized-q}
I^k(\vt, \vc{s}^{(k)}) := \textsc{Top-k}(- \vc{q}^k(\vt), s^k), 
\quad \rinp^{(k)}(\vc{t}, \vc{s}^{(k)}) = \relu\left(\textsc{sort}(- \vc q^k(\vt), s^k)\right).
\end{equation}
Here for a vector $\vt$ and index $j$, $\textsc{Top-k}(\vt,j)$ is the index set of the top $j$ entries and $\textsc{sort}(\vt, j)$ is the $j^{th}$ largest entry in $\vt$. The above expression also reveals that the computation of the sparse local radius can be easily incorporated into the forward pass of neural networks (with an additional normalization and sort operation). 
The radius quantifies the minimal distance of a neuron in the ``most inactive set'' $I^k$ (of size $s^k$) to becoming active. As illustrated in \Cref{fig:supports_diagram}(b), 
$r^{(k)}_{\mathrm{inp}}(\vc{t}, \vc{s}^{(k)})$ 
is a \emph{decreasing} function of $s^k$. 
While \Cref{lemma: fnn-layer-lip} studies sensitivity to $\ell_2$ perturbations, one can easily extend this to any $\ell_p$-norm perturbation. To do so, it suffices to compute the vector $\vc{q}^k(\vt)$ 
by normalizing with a $\ell_q$-norm (with $1/q+1/p=1$). The corresponding Lipschitz scale is then $\norm{\mathcal{P}_{J^k, J^{k-1}}(\W^k)}_{p\rightarrow p}$. 
We now present the robustness certificate that combines \Cref{lemma:cert-rob-gen-loc,lemma: sparse-local-lip-composition,lemma: fnn-layer-lip}.} 
\begin{corollary}\label{corollary:cert-rob-fnn-loc}
Consider a trained depth-$K$ feedforward neural network $h \in  \HNNK$. 
Let $\vc{s}$  be a fixed choice of sparsity levels at each layer so that $s^{k}\leq d^k-\|\Phi^{(k)}(\vc t)\|_0$, and $\vc{s}^{(k)}$ be the corresponding layer-wise input-output sparsity levels. \addition{The cumulative sparse local radius $r^{[K]}$ and local Lipschitz scale $l^{[K]}$ 
are defined as,
\begin{equation*}
   \rinp^{[K]}(\x, \vc{s}) := \min_{1\leq k \leq K}
   \frac{
   \mathrm{ReLU}\left(\textsc{sort}(- \vc q^k(\Phi^{[k-1]}(\x) ), s^k)\right),
   }{\prod_{n=1}^k \norm{\mathcal{P}_{J^n, J^{n-1}}(\W^n)}_2}, \quad~ 
   \linp^{[K]}(\x, \vc{s}) := \prod_{k=1}^K 
   \norm{ \mathcal{P}_{J^k, J^{k-1}}(\W^k)}_2  
\end{equation*}
}
The predicted label remains unchanged, i.e. $\hat{y}(\x+\bdel) = \hat{y}(\x)$, whenever $\norm{\bdel}_2 \leq \rcert(\x, \vc{s})$,
\[
\rcert(\x,\vc{s}):= \min \left\{r^{[K]}_{\mathrm{inp}}(\x,  \vc{s}) ,~  \frac{\mathcal{M}(h(\x), \hat{y}(\x))}{2 \norm{\A}_2 \cdot l^{[K]}_{\mathrm{inp}}(\x,\vc{s}) } \right \}.
\]
\end{corollary}
The proof can be found in \Cref{app: corollary:cert-rob-fnn-loc}. \addition{Since feedforward networks 
have monotone ordering, for each $\nu > 0$, 
\Cref{alg: greedy-sparse}  %
provides optimal sparsity vector $\hat{\vc{s}}$ for certification.} 
\subsection{Discussion}\label{subsec: cert-rob-further}
For each input $\x$, the inactive index sets $I^k$ of size $s^k$ in \Cref{lemma: fnn-layer-lip} are chosen to maximize the local radius $\rinp^{(k)}(\x)$.
The certificate in \Cref{corollary:cert-rob-fnn-loc} reflects a trade-off between the sparsity of the representations at each intermediate layer $s^k$ (via the reduced operator norms) and $\mathcal{M}(h(\x), \hat{y}(\x))$, the classification margin in feature space. 
Before moving on, we summarize a few key remarks about our approach.

\paragraph{Reduced Model Perspective}\label{subsubsec: cert-rob-reducedmodel}
At each input $\x$, for sparsity levels $s^k$  
and their chosen index sets
$J^k$ of size $d^k-s^k$, we can define a reduced neural network $\Phi^{[K]}_{\text{red}} : \mathbb{R}^{d^0} \rightarrow \mathbb{R}^{d^K-s^K}$, 
\begin{equation}\label{eq:reducedcertnetwork}
    \Phi^{[K]}_{\text{red}}(\x) :=  
     \act{\W^K_{\text{red}} \;\act{
     \W^{K-1}_{\text{red}}
     \cdots \act{ \W^1_{\text{red}} \;\mathcal{P}_{J^0}(\x) + \vb^1_{\text{red}}} \cdots 
     + \vb^{K-1}_{\text{red}}
     }
     + \vb^K_{\text{red}}} 
\end{equation}
where $\W^k_{\text{red}} \in \mathbb{R}^{(d^k-s^k)\times (d^{k-1}-s^{k-1})}$ are defined as the sub-matrices of the parameters of $\Phi^{(k)}$ at specific active sets, i.e. $\W^k_{\text{red}} = \mathcal{P}_{J^k, J^{k-1}}(\W^k)$ -- and similarly for the biases.
\cref{corollary:cert-rob-fnn-loc} essentially shows that, at each input $\x$, the feedforward neural network $\Phi^{[K]}$ is equivalent to a particular reduced feedforward neural network $\Phi^{[K]}_{\text{red}}$ in a local neighborhood around input $\x$.
That is, for all perturbations $\bdel$ such that $\norm{\bdel}_2 \leq \rinp^{[K]}(\x, \vc{s}^{[k]})$, the following holds, 
\begin{equation}
\label{eq:reduced_equivalence}
    \Phi^{[K]}(\x + \bdel) = \Phi^{[K]}_{\text{red}}(\x+\bdel).
\end{equation}
The reduction of the active weights in the network locally can be seen as a form of input-specific pruning of the neural network. 
Importantly, this observation goes beyond the statement that a feedforward neural network is locally \textit{linear} in a neighborhood of $\x$. Making this observation would be too stringent, as the size of this neighborhood could be arbitrarily small. Instead, the equivalence in Eq. \eqref{eq:reduced_equivalence} holds for the \emph{nonlinear} function $\Phi^{[K]}_{\text{red}}$. Within the specified neighborhood, the activation patterns might change, but only in the complement of the sets $I^k$. The definition of these sets provides a controllable knob via the sparsity requirements.

\paragraph{Comparison to Related Work}\label{subsubsec: cert-rob-relatedwork}
These observations are related to the analysis based on max-affine operators  in \cite{Balestriero2019TheGO}, providing a partitioning of the input space $\cX$ based on successive feedforward layers. 
That work shows an effective way to compute the distance to the partition boundary, and this can be seen as a version of the local radius function we have defined \emph{only} when the sparsity level is set to the exact number of inactive neurons $|\mathcal{I}^{k}(\x)|$ for each input in each layer. 
When there exists a row $\Wvec{k}{j}$ that is nearly active, i.e. $-\xi < \Wvec{k}{j}\Phi^{[k-1]}(\x) + \vb^k_j < 0$ for small $\xi \approx 0$,
the distance from $\x$ to the input partition is near zero. 
In this case, the flexibility for different levels of sparsity in our analysis is crucial, allowing us to expand the active set and increase the allowable radius. 

The work in \cite{pmlr-v70-cisse17a} makes the empirical observation that \textit{Parseval networks} --loosely speaking, those with operator norms close to 1-- result in better robustness. Our results show that, in fact, this is not necessary as long as the operator norms of the \emph{reduced} matrices are close to one.
In particular, increasing depth can have a dramatic impact on the local Lipschitz scale $\linp$ if the reduced linear map is contractive while the original map is not, {\it i.e.}  $\norm{\mathcal{P}_{J^k,J^{k-1}}(\W^k)} < 1 < \norm{\W^k}_2$. 
More generally, our results expose an inefficiency in approaches that directly compute the Lipschitz constant of the full feedforward network. Our measure of sensitivity that accounts for \emph{locality} and \emph{sparsity} are at least as good as global measures (and potentially much tighter). 
Note that one could use any algorithm for estimating the Lipschitz constant of a neural network \cite{ doi:10.1137/19M1272780, Fazlyab2019EfficientAA, Gmez2020LipschitzCE, Scaman2018LipschitzRO, Weng2018TowardsFC, Zhang2018EfficientNN} applied to the reduced model,  $\Phi^{[K]}_{\text{red}}$, in order to estimate $\linp^{[K]}(\x, \vc{s}^{[K]})$ efficiently. 
Lastly, the authors in \cite{Mehta2012OnTS, sulam2020adversarial} study the case of supervised sparse coding, performing a similar sensitivity analysis for the hypothesis class $\cHssc$ focusing on a local radius threshold (or ``\emph{encoder gap}'') that preserves the support, or sparsity level, of the representation obtained under corruption. 
The robustness certificate  developed in \cite[Theorem 5.1]{sulam2020adversarial} is equivalent to the application of \Cref{lemma:cert-rob-gen-loc} for the class $\cH_{\text{SSC}}$ (see \Cref{app: lemma-sulam-adv} for full details). Thus, our work generalizes results in \cite{sulam2020adversarial}. 

\paragraph{Dependence on Input}\label{subsubsec: Dependence on Input}
The results presented above are input-specific, and require the computation 
of the operator norms of the reduced sub-matrices.
In many settings, it might be more relevant to have a similar notion for a set of inputs instead. For each layer, this can be done by searching for the worst case sub-matrix of $\W^k$ of size $(d^k-s^k, d^{k-1}-s^{k-1})$ via an extension of the well-studied notion of 
\textit{babel function} \cite{tropp2003improved}.
\begin{definition} (Reduced Babel Function)
For any matrix $\W \in \mathbb{R}^{d_1 \times d_2}$, we define the reduced babel function at row sparsity level $s_1 \in \{0, \ldots, d_1-1\}$ and column sparsity level $s_2 \in \{0, \ldots, d_2-1\}$ as, 
\[
\mu_{s_1,s_2}(\W) := \underset{\substack{J_1 \subset [d_1],\\ |J_1|=d_1-s_1}}{\max} \;
\max_{j \in J_1} 
\Bigg[
\sum_{\substack{i \in J_1,\\ i \neq j}}\;
\underset{\substack{J_2 \subseteq [d_2]\\ |J_2| = d_2 - s_2}}{\max}
\frac{\lvert \mathcal{P}_{J_2}(\vc{w}_i) \mathcal{P}_{J_2}(\vc{w}_j)^T \rvert}{\norm{\mathcal{P}_{J_2}(\vc{w}_i)}_2\norm{\mathcal{P}_{J_2}(\vc{w}_j)}_2} \Bigg],
\]
the maximum cumulative mutual coherence between a reference row in $J_1$ of size $(d_1-s_1)$ and any other row in $J_1$, each restricted to any subset of columns $J_2$ of size\footnote{
When $s_1 = d_1-1, |J_1| = 1$, we simply define $\mu_{(s_1,s_2)}(\W) := 0$.
} $(d_2-s_2)$. 
\end{definition}
The reduced babel function is computationally tractable\footnote{Replacing the $\norm{\mathcal{P}_J(\vc{w}_i)}_2$ in the denominators of the definition with $\norm{\W}_{2,\infty}$ vastly reduces the complexity of evaluating the reduced babel function and all subsequent lemmas still hold.} albeit more expensive to compute than the babel function.
The additional flexibility is showcased in the following result.
\begin{lemma}
\label{lemma: bound-submatrix-norm}
For any matrix $\W \in \mathbb{R}^{d_1 \times d_2}$, 
the operator norm of any non-trivial\footnote{That is $0\leq s_1 \leq d_1-1$ and $0 \leq s_2 \leq d_2-1$.} sub-matrix indexed by sets $J_1 \subseteq [d_1]$ of size $(d_1-s_1)$ and $J_2\subseteq [d_2]$ of size $(d_2-s_2)$ can be bounded as, 
$\norm{\mathcal{P}_{J_1, J_2} (\W)}_2 \leq \sqrt{1+\mu_{s_1, s_2}(\W)} \cdot \norm{\W}_{2,\infty}.$
\end{lemma}
\begin{proof}\addition{
By the Gerschgorin Disk Theorem, for any eigenvalue $\lambda$ of $\mathcal{P}_{J_1, J_2}(\W)$, there exists index $j$ such that $\lambda_i$ lies in the Gerschgorin disks centered at $\langle\mathcal{P}_{J_2}(\Wvec{}{j}),\mathcal{P}_{J_2}(\Wvec{}{j}) \rangle$ with radius $\sum_{i \neq j} \langle\mathcal{P}_{J_2}(\Wvec{}{j}),\mathcal{P}_{J_2}(\Wvec{}{i}) \rangle$.
The conclusion follows by bounding the radius using the reduced babel function $\mu_{s_1, s_2}(\W)$. 
For a complete proof see \Cref{app: lemma: bound-submatrix-norm}.}
\end{proof}
A key feature of the above lemma is the bound on the operator norm of a sub-matrix that only depends on the size of the chosen index sets. 
\addition{For non-trivial sparsity levels, the proposed bound often improves on the naive bound $\norm{\mathcal{P}_{J_1, J_2}(\W)}_2 \leq \norm{\W}_2$.}
Similar notions have been proposed in \cite{aberdam2020and}.
The upper bounds from \Cref{lemma: bound-submatrix-norm} can be computed offline and used in place of $\norm{\mathcal{P}_{J^k,J^{k-1}}(\W^k)}_2$ for quick certification for a new sample with appropriate sparsity. 

\section{Robust Generalization}\label{sec: rob-gen}
In this section we move beyond deterministic robustness certificates and provide a generalization bound for the robust risk of SLL hypotheses that only has a mild dependence on the energy of the adversary. 
We do so by studying the sensitivity of a predictor to simultaneous changes in input \emph{and} parameters. 

Recall that a representation-linear hypothesis class $\cH$ contains predictors of the form $h(\cdot) = \A \Phi_\W (\cdot)$ where the parameters $(\A, \W) \in \cA \times \cW$. 
Until further instantiated, we assume that the parameter sets $\cA$ and $\cW$ 
are bounded with respect to embedded norms, 
such that $\norm{\A}_2 \leq \mathsf{M}_{\cA}$ and $\norm{\W}_\cW \leq \mathsf{M}_{\cW}$. 
We define the norm 
of a representation-linear predictor $h$ to be 
$\norm{h}_\cH := \max \left\{\frac{\norm{\A}_2}{\mathsf{M}_\cA}, \frac{\norm{\W}_\cW}{\mathsf{M}_\cW}\right\}.$
For any pair predictors $h$ and $\hat{h}$ with weights $(\A, \W)$ and $(\Ahat, \What)$, we let $\Phi$ and $\hat{\Phi}$ be their corresponding representation maps. 
The distance between predictors is measured by the induced distance metric\footnote{The induced distance metric is, 
$\norm{\hat{h}-h}_\cH := \max \left \{ \frac{\norm{\Ahat-\A}_2}{\mathsf{M}_\cA}, \frac{\norm{\What-\W}_\cW}{{\mathsf{M}_\cW}}\right\}$.}.
For a (globally) Lipschitz predictor $h$, one can obtain a uniform bound (see \Cref{app: thm: unif-gen-bound}) on the generalization error that is $\mathcal{O}\left(\sqrt{\frac{\ln(\mathcal{N}(\frac{1}{m},\cH)) + \ln(\frac{1}{\alpha})}{m}} + \frac{\Lpar}{m} \right)$, with probability $1-\alpha$, and $\mathcal{N}(\epsilon, \cH)$ is the proper covering number of $\cH$ w.r.t induced distance metric at resolution $\epsilon$. In the analysis that follows, we refine the second term by exploiting data-dependent properties of a trained predictor.

\subsection{Sparse local Lipschitz w.r.t. Parameters}\label{subsec: SLL-param}
We start by characterizing a class of functions that preserve sparsity in the representations.

\begin{definition}
\label{def: lip-param}
Let $h$ be a representation-linear hypothesis such that $\Phi(\x)$ is $s$-sparse at $\x$.
The hypothesis $h$ is $s$-sparse local Lipschitz w.r.t. parameters at $\x$ 
if there exists an inactive index set $I$ of size $s$, 
a  local radius $r\geq 0$
and a local Lipschitz scale $l\geq 0$
such that, for any perturbed predictor $\hat{h}\in\mathcal B^{\mathcal H}_r(h)$, 
\[
\mathcal{P}_{I}(\hat{\Phi}(\x)) = \mathcal{P}_{I}(\Phi(\x)) = \mathbf{0} \quad \land \quad 
\norm{\hat{h}(\x) - h(\x)}_2 \leq l \norm{\hat{h}-h}_\cH.
\]
Furthermore, $h$ is sparse local Lipschitz w.r.t. parameters if for every $\x\in\mathcal X$ and all appropriate sparsity levels $s$, $h$ is $s$-sparse local Lipschitz (w.r.t. parameters) at $\x$, with corresponding local radius $\rpar(\x,s)$ and local Lipschitz scale $\lpar(\x,s)$.
\end{definition}

As before, the local neighborhood is defined both in terms of a norm radius and a sparsity level, and both $r(\x,s)$ and $l(\x,s)$ decrease with increasing $s$.
Importantly, for Lipschitz hypotheses $h$, there always exist $s$ so that $h$ is $s$-sparse local Lipschitz. As the reader can likely foresee, the utility of \Cref{def: lip-param} is that, when $\rpar(\x, s) > 0$, the sparsity level $s$ indicates the existence of a stable inactive set of indices $I$ of size $s$, so that one can restrict the analysis to the (extended) active set $J = [d]\backslash I$ in such a way that
\begin{equation*}
\|\hat{h}-h\|_\cH \leq \rpar(\x, s) \implies \hat{h}(\x) = \Ahat \hat{\Phi}(\x) = \mathcal{P}_{[C], J }(\Ahat)\; \mathcal{P}_{J }(\hat{\Phi}(\x)).
\end{equation*}


\noindent The range of sparsity levels for each input is in $\{0,\ldots, p - \norm{\Phi(\x)}_0\}$, and this recovers Lipschitz functions at the trivial choice of sparsity level $s=0$ with $\rpar(\x,0)=\infty$ and 
$\lpar(\x, 0) := \Lpar$. 

In order to study generalization, one must extend this property over a finite set of samples $\cV \subset \cX$. 
For a certain radius threshold $\epsilon > 0$, among all feasible sparsity levels, we choose the optimal $s$ that minimizes the worst-case local Lipschitz scale across the set $\cV$ while guaranteeing a sufficiently large local radius, that is
\begin{equation}\label{eq:optimal_set_sparsity}
s^*(\cV,\epsilon) := \underset{s}{\argmin} \max_{\x\in\cV}     ~\lpar(\x, s)~~ \text{s.t.}~~\epsilon \leq \min_{\x\in\cV}
 \rpar(\x, s).
\end{equation}
Note that $s=0$ is always feasible for the optimization problem defined.
We can now define the \textit{sparse regularity} of a predictor $h$ w.r.t. reference set $\cV$ and a fixed radius threshold $\epsilon$ as
\[
\mathcal L(h,\cV,\epsilon) := \underset{\x \in \cX}{\max}~ \lpar(\x, s^{\star}(\cV, \epsilon))~~\text{s.t.}~~\epsilon \leq 
 \rpar\left(\x, s^{\star}(\cV, \epsilon)\right).
\]
This sparse regularity measures the worst-case local Lipschitz scale at any input in $\cX$ with a sufficiently large local radius at the reference sparsity level $s^*(\cV, \epsilon)$.
In the unfavorable (\emph{not sparse}) case $s^*(\cV, \epsilon) = 0$, the corresponding sparse regularity $\mathcal L(h,\cV,\epsilon) = \Lpar$, the global Lipschitz constant. 
Thus a generalization bound that relying on $\mathcal L(h,\cV,\epsilon)$ is, at worst, dependent on the global Lipschitz constant $\Lpar$, but potentially much tighter. 
We now present our generalization bound for SLL predictors, which makes use of \emph{unlabeled} samples, $\samp_U$, in addition to the training set $\samp_T$, both with $m$ samples. The former will be used to inform the sparse regularity of the predictor, while the latter is used to fit the parameters of the models. 

\begin{theorem}
\label{thm: generalization-sll}
With probability at least $(1-\alpha)$ over the choice of i.i.d training sample $\samp_T$ and unlabeled data $\samp_U$ each of size $m$, 
for any predictor $h \in \cH$ with parameters $(\A,\W)$,
the generalization error, \addition{with a $b$-bounded and  $\mathsf{L}_{\mathrm{loss}}$-Lipschitz loss}, is bounded by
\begin{equation}
R\left(h\right) \leq  
\hat{R}\left(h\right)
+ \mathcal{O}\left(
b \sqrt{\frac{\ln\left(\mathcal{N}(\frac{1}{m},\cH)\right) + \ln(\frac{2}{\alpha})}{2m}} 
+  \frac{\Lloss \cdot \mathcal{L}(h, \samp_T \cup \samp_U,\frac 1{2m})}{m}\right).
\end{equation}
\end{theorem}
\addition{This result follows from standard arguments by constructing an $\epsilon$ cover of the hypothesis space and bounding the stability of the function's outputs on this cover, and can be considered a generalization of the bound presented in \cite{sulam2020adversarial}. The complete proof can be found in \Cref{app: thm: generalization-sll}. 
}
Note that to compute the sparse regularity of $h$, $\rpar$ must be easy to compute at each input, and $\lpar$ must be regular enough to optimize over. 
The requirement of additional unlabeled data 
can be seen as a limitation, however this dependence is mild as it incurs in a linear increase in the number of training samples.

\subsection{Robust Sparse Local Lipschitz}\label{subsec: rob-SLL-param}

To extend \Cref{thm: generalization-sll} to the robust setting one needs to  characterize the parameter sensitivity of the predictor under corrupted inputs. Let $\nu > 0$ be the adversarial energy, recall that $\Lparnu$ the global Lipschitz constant, 
\[
\forall~ \hat{h}, h, \x, \quad \max_{\bdel \in \cB^{\cX}_{\nu}(\mathbf{0})} \norm{\hat{h}(\x+\bdel) - h(\x+\bdel)}_2 \leq \Lparnu \norm{\hat{h}-h}_\cH.
\]
Based on this, a uniform generalization bound, analogous to \Cref{thm: unif-gen-bound} but for the robust risk $R_{\mathrm{rob}}(h)$ can be readily established with a dependence of $\mathcal{O}(\Lparnu/m)$. 
To move beyond the global analysis we extend the sparse local Lipschitz property.
\begin{definition}
\label{def: lip-param-rob}
Let $h$ be a representation-linear hypothesis so that $\Phi(\x)$ is $s$-sparse at $\x$. 
The hypothesis $h$ is robust $s$-sparse local Lipschitz w.r.t. parameters at $\x$, 
if there exists an inactive index set $I$ of size $s$, a local radius $r\geq 0$, and a local Lipschitz scale $l\geq0$,
such that,
for any perturbed predictor $\hat{h}\in\mathcal B^{\mathcal H}_r(h)$ and any corruption $\bdel \in \cB^{\cX}_\nu(\mathbf{0})$, the index set $I$ remains inactive after input and parameter perturbations, and the distance between the predictor outputs are bounded, that is
\[
\mathcal{P}_{I}(\hat{\Phi}(\x+\bdel)) = \mathcal{P}_{I}(\Phi(\x+\bdel)) = \mathbf{0} \quad \land \quad 
\norm{\hat{h}(\x+\bdel) - h(\x+\bdel)}_2 \leq l \norm{\hat{h}-h}_\cH.
\]


Additionally, the hypothesis $h$ is robust sparse local Lipschitz w.r.t. parameters if $h$ is sparse local Lipschitz (w.r.t. parameters) for every $\x\in\mathcal X$ and any appropriate sparsity level $s$, with corresponding local radius $\rparnu(\x,s)$ and local Lipschitz scale $\lparnu(\x,s)$.
\end{definition}
For a robust sparse local Lipschitz $h$, at any input $\x$ and sparsity level $s$ where the predictor has nontrivial robust local radius (i.e. $\rparnu(\x, s) \geq 0$), there exists an inactive index set $I$ of size $s$ for the representation $\Phi(\x)$ that withstands simultaneous perturbations to inputs and parameters. Note that the chief difference between this and \Cref{def: lip-param} is that here the sensitivity is evaluated at the point $(\x+\bdel)$ while being a property of $h$ at $\x$. Indeed, a sparse local Lipschitz predictor is also robust sparse local Lipschitz, 
\begin{equation*}
 \rparnu (\x, s)  = \min_{\bdel \in \cB^{\cX}_{\nu}(\mathbf{0})}~ \rpar(\x + \bdel, s),\quad \lparnu (\x,s)  = \max_{\bdel \in \cB^{\cX}_{\nu}(\mathbf{0})}~ \lpar(\x + \bdel, s).
\end{equation*}
At the trivial sparsity level $s=0$, we simply let $\rparnu(\x,0) = \infty$ and the $\lparnu(\x, 0) = \Lparnu$, for any Lipschitz $h$.
Leveraging the controllable trade-off between sparsity levels and the local sensitivity we define the robust optimal level $s^*_{\mathrm{rob}}(\cV,\epsilon)$ and robust sparse regularity, 
\begin{align}\label{eq:robustsparseregularity}
s_{\mathrm{rob}}^*(\cV,\epsilon) &:= \underset{s}{\argmin} \max_{\x\in\cV}     ~\lparnu(\x, \vc s)~~ \text{s.t.}~~\epsilon \leq \min_{\x\in\cV}
 \rparnu(\x, \vc s).\\
\nonumber \mathcal L_{\mathrm{rob}}(h,\cV,\epsilon) &:= \underset{\x \in \cX}{\max}~ \lparnu(\x, s^{\star}_{\mathrm{rob}}(\cV, \epsilon))~~\text{s.t.}~~\epsilon \leq 
 \rparnu\left(\x, s^{\star}_{\mathrm{rob}}(\cV, \epsilon)\right).
\end{align}
Using these ideas,  the result in \Cref{thm: generalization-sll} can be extended to the robust setting 
(see \Cref{app: thm: robust-generaalization-sll}). 
We omit this and move on to our analysis for feedforward neural networks.

\subsection{Robust Generalization for Feedforward Neural Networks}
\label{subsec: rob-gen-fnn}

For simplicity, we only consider networks with zero bias (as in \cite{Bartlett2017SpectrallynormalizedMB, Nagarajan2019DeterministicPG, Neyshabur2017ExploringGI, Neyshabur2015NormBasedCC})\footnote{Results for networks with non-zero bias can also be derived from our analysis.}.
We consider depth-$(K+1)$ feedforward neural networks where the representation map $\Phi$ has layer weights $\{\W^k\}_{k=1}^K$.  
For notational convenience, we denote the classification weight in the final linear layer as $\W^{K+1}$ (in lieu of $\A$). 
For the reminder of this section, we consider a fixed set of constants $\{\mathsf{M}^k_{\cW}, \{\mathsf{M}^k_{s}\}\}_{k=1}^{K+1}$ that defines a hypothesis space $\cH^{K+1}$ with parameters in $ \prod_{k=1}^{K+1} \cW^k$ where,  
\begin{equation*}
\cW^k := 
\Bigg\{
\W \in \mathbb{R}^{d^k \times d^{k-1}} ~\Big|~ \norm{\W}_{2,\infty} \leq \mathsf{M}^k_{\cW}, \;\;\forall~ (s^k, s^{k-1}), \;\;  \mu_{s^k, s^{k-1}}(\W) \leq \mathsf{M}^k_{s^k}
\Bigg\},
\end{equation*}
while $\cW^{K+1} := \Bigg\{\W \in \mathbb{R}^{d^K \times C} ~\Big|~ \norm{\W}_{2,\infty} \leq \mathsf{M}^{K+1}_{\cW},\; \; \forall~s^K, \;  \mu_{s^K, 0}(\W) \leq \mathsf{M}^{K+1}_{s^K} \Bigg\}$.
The final classification weight space\footnote{For convenience, $\cW^{K+1}$ has been defined as a subset of $\mathbb{R}^{d^K \times C}$ rather than $\mathbb{R}^{C\times d^K}$.} accounts for the sparsity in the representation output as opposed to the output of the predictor.
In this manner, a predictor $h \in \cH^{K+1}$ is defined as $h(\x) = (\W^{K+1})^T \Phi^{[K]}(\x)$, where 
the representation map $\Phi^{[K]}$ is the composition of $K$ feedforward maps, so that
$\Phi^{[K]}(\x)
=\act{\W^K \act{\W^{K-1} \cdots \act{\W^1\x }}}$. 
While the weight spaces are constrained in the group norm, we define the following scaled norm $\norm{\cdot}_{\cW^k}$ fit to the purpose of measuring parameter perturbations, 
$\norm{\W}_{\cW^k} := \frac{\sqrt{d^k}}{\mathsf{M}^k_{\cW}} \cdot \norm{\W}_{2,\infty}$ 
such that $\norm{\W^k}_{\cW^k} \leq \sqrt{d^k}$ for any $\W^k \in \cW^k$. 
Based on the scaled norms, we define the norm 
of any feedforward network,
$
\norm{h}_{\cH^{K+1}} := \displaystyle \max_{1 \leq k\leq K+1}~ \norm{\W^k}_{\cW^k} %
$.
Additionally, the predictors in $\cH^{K+1}$ are constrained by the reduced babel function\footnote{Naturally, we only consider constraints that match the properties of the reduced babel function: since by definition $\mu_{(d^k-1,s)}(\W)=0$, we require
$\mathsf{M}^k_{d^k-1}=0$ for all layers.
Furthermore, for any $a,b : a \geq b$, we require $\mathsf{M}^k_a \leq \mathsf{M}^k_b$ to mirror the fact that $\mu_{a,s}(\W^k) \leq \mu_{b,s}(\W^k)$.} at each layer for all appropriate sparsity levels. 
As before, at each layer $1\leq k \leq K$, $\mathcal{I}^k(\x) := \{j \in [d^k] : \w^k_j \Phi^{[k-1]}(\x) \leq 0\}$ and $\bar{s}^k(\x) := |\mathcal{I}^k(\x)|$, denote the index set of all inactive rows and their sizes, respectively. We further let $\bar{\vc{s}}(\x) := \{\bar{s}^1(\x), \ldots, \bar{s}^K(\x)\}$. 
Although \Cref{def: lip-param-rob} only requires a scalar sparsity level corresponding to the representation output, $\Phi(\x)$,  
for the case of multi-layered neural networks we will refine this  definition by a vector of layer-wise sparsity levels $\vc{s} = \{s^0, s^1,\ldots, s^K\}$ that achieve the same goal of sparsity in each layer representation at the level $s^k$. Note that for a representation at a given point, $\Phi(\vc x)$, this latter vector denotes potential sparsity levels, whereas the previous $\bar{\vc s}(\x)$ denotes the maximal possible sparsity; i.e. $s^k\leq\bar{s}^k(\x)$ for all $k$.

The intermediate sparsity levels can improve the robust properties further. 
To quantify this phenomenon, for any sparsity vector $\vc{s}$, we define $\zeta^0(\vc{s}) := 1$ and for $1\leq k \leq K+1$,
$\zeta^k(\vc{s}) := \prod_{n=1}^k \mathsf{M}^n_{\cW}\sqrt{1+\mathsf{M}^n_{s^n}}$. 
Indeed, as per \Cref{lemma: bound-submatrix-norm}, $\zeta^k(\vc{s})$ provides an upper bound on the product of operator norms of reduced linear maps, i.e.
\begin{equation}
\zeta^k(\vc{s}) 
\geq \prod_{n=1}^k \sup_{\W^n \in \cW^{n}}\; \norm{\W^n}_{2,\infty}\sqrt{1+\mu_{s^n,s^{n-1}}(\W^n)}    
\geq \prod_{n=1}^{k} 
\underset{\substack{\W^n \in \cW^{n}, \\|J^{n}| = d^n-s^{n},\\|J^{n-1}| = d^{n-1}-s^{n-1}}}{\sup} \norm{\mathcal{P}_{J^n,J^{n-1}}(\W^n)}_2. 
\end{equation}
For 
any $\tilde{\vc{s}} \succeq \vc{s} \succeq \mathbf{0}$, 
$\zeta^k(\tilde{\vc{s}}) \leq \zeta^k(\vc{s}) \leq \zeta^k(\vc{0})$, where $\zeta^k(\vc{0})$ is an upper bound on the product of operator norms of the full linear maps $\W^k$. 
For the induced distance metric\footnote{For any two networks $h$ and $\hat{h}$ with weights $\{\W^k\}_{k=1}^K$ and $\{\W^k\}_{k=1}^{K+1}$ respectively, 
\[\norm{\hat{h}-h}_{\cH^{K+1}} := \max_{1\leq k\leq K+1}~ \norm{\What^k-\W^k}_{\cW^k}.\]} corresponding to $\norm{\cdot}_\cH$, we note the following \emph{robust} (global) Lipschitz of a neural network. 
\begin{lemma}\label{lemma: robust-global-lip-fnn} 
\addition{For a fully connected neural network with $K$ layers, $\Phi^{[K]}(\x)$, its robust global Lipschitz constant can be upper bounded by $\Lparnu \leq (K+1) \zeta^{K+1}(\vc{0}) \cdot (1+\nu)$.} 
\end{lemma}
\addition{
\textit{Proof sketch.}
    The proof is a simple application of the definitions above and operator norm inequalities. Given predictors $h, \hat{h} \in {\cH^{K+1}}$ with weights $\{\W^k\}$ and $\{\What^k\}$, we note that for $1\leq k \leq K$, for any layer weight matrix $\W^k \in \cW^k$,  $\norm{\W^k}_2 \leq \sqrt{1+\mathsf{M}^k_{0}} \cdot \mathsf{M}^k_{\cW},$
and further, $\norm{\What^k-\W^k}_2 \leq \mathsf{M}^k_{\cW} \norm{\hat{h}-h}_{\cH^{K+1}}.$ Similar inequalities hold for $K+1$. 
We then show that, at any layer $k \leq K$, the distance between the perturbed representations are bounded, $\norm{\hat{\Phi}^{[k]}(\x+\bdel) - \Phi^{[k]}(\x+\bdel)} \leq k \zeta^{k}(\vc{0}) \cdot (1+\nu) \cdot \norm{\hat{h}-h}_{\cH^{K+1}}$. The final proof follows by employing the upper bound to the operator norms, given by $\zeta^k$, and this latter bound of the representations at every layer. The full proof can be found in \Cref{app: lemma: robust-global-lip-fnn}.
}

Note that $\Lparnu$ is exponential in the network's depth and captures the worst case interaction between layer matrices, inputs and adversarial perturbations.
Frequently generalization bounds measure the sensitivity of a hypothesis class using $\Lparnu$ and hence $\zeta^{K}(\vc{0})$ \cite{Bartlett2017SpectrallynormalizedMB, Neyshabur2017ExploringGI}. 
To instead measure sensitivity using $\zeta^k(\vc{s})$,  
we need to characterize the inactive set at each layer, and hence we identify a \emph{critical angle} between the rows of $\w^k$ and the layer input.  
\begin{definition} (Critical Angular Distance)\label{def: crit-angle} 
\addition{Let $\mathsf{M}_{\cW}$ and $\mathsf{M}_{\mathcal{T}}$ be domain hyper-parameters.} Consider a matrix $\W \in \cW \subset \mathbb{R}^{p \times q}$ such that $\norm{\W}_{2,\infty} \leq \mathsf{M}_\cW$, and a vector $\vt \in \mathcal{T} \subset \mathbb{R}^{q}$ such that $\norm{\vt}_2 \leq \mathsf{M}_{\mathcal{T}}$. 
The angular distance\footnote{The term ``distance'' here is an abuse of notation. More precisely, the vector $\beta(\W,\vt)$ contains a distance value in each component. } between the matrix $\W$ and vector $\vt$ is defined to be the vector function $\beta : \cW \times \mathcal{T} \rightarrow [0,1]^{p}$,
\[
[\beta(\W, \vt)]_i := \frac{1}{\pi} \cdot \arccos \Bigg( \frac{\langle \Wvec{}{i},\vt\rangle}{\mathsf{M}_\cW \mathsf{M}_{\mathcal{T}}}  \Bigg), \quad \forall~ i \in [p].
\]
The critical angular distance $\theta : \cW \times \mathcal{T} \times [p] \rightarrow [0,1]$ at sparsity level $s$ is the $s^{th}$-largest entry, 
\[
\theta(\W, \vt, s) := \textsc{sort} \left( \beta(\W, \vt) , s\right).
\]
\end{definition}

\begin{wrapfigure}{R}{.4\textwidth}
    \centering
    \includegraphics[trim=10 10 10 40, width = .35\textwidth]{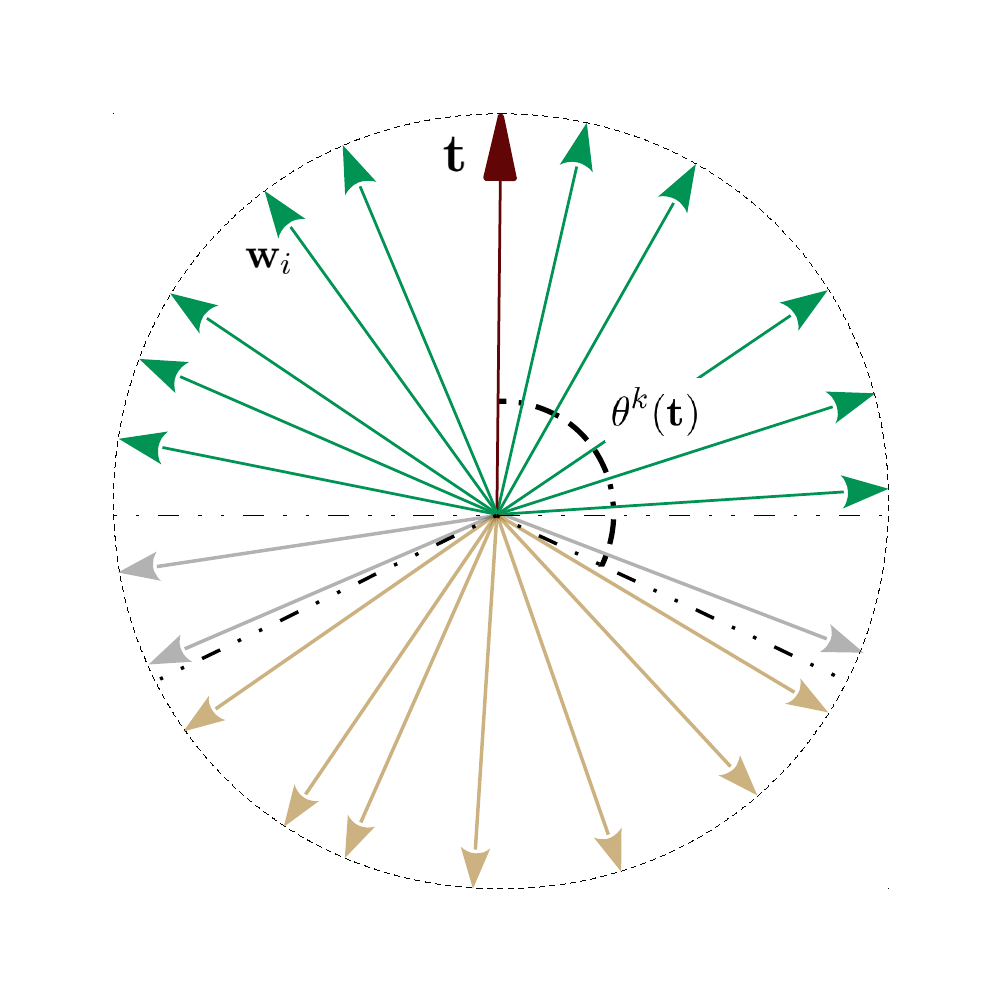}
    \caption{
    Illustration of the \emph{critical angular distance} for layer matrix $\W^k$ and input $\vc{t}$. Green weights are active 
    while 
    grey and orange ones are inactive. 
    The critical angle, $\pi\cdot \theta^k(\vc{t})$, is denoted by black dashed line.
    }
    \label{fig:crit_angle}
\end{wrapfigure}
Each component of $\beta(\W,\vt)$ quantifies the angular distance between a row $\Wvec{}{i}$ and $\vt$\footnote{Note that, naturally, $\langle \vc{w}_i,\vt\rangle = \mathsf{M}_\cW \mathsf{M}_{\cal T} \cos\left(\pi [\beta(\W, \vt)]_i\right)$. Furthermore, our definition represents a scaled version of the true angular distance since
\[
\Big\lvert\frac{\langle \vc{w}_i,\vt\rangle}{\mathsf{M}_{\cW} \mathsf{M}_{\mathcal{T}}}\Big\rvert \leq  \Big \lvert\frac{\langle \vc{w}_i,\vt\rangle}{\norm{\vc{w}_i}_2\norm{\vt}_2}\Big \rvert.
\]
}.
In turn, the critical angular distance $\theta(\W, \vt, s)$ represents the $s^{th}$-largest angular distance formed by any row of matrix $\W$ with the input $\vt$. 
A larger critical angular distance indicates that some set of $s$ inactive indices in $\sigma(\W \vt)$ is resilient to bounded perturbations in the input $\vt$ or weight vectors $\W$,
as illustrated in \cref{fig:crit_angle}.

For the case of multi-layered neural networks, a layer-wise angular distances and critical angular distance can be evaluated at each layer with domain hyper-parameters $\mathsf{M}^k_{\cW}$ and $\zeta^{k-1}(\mathbf{0})$, which we  
denote by $\beta^k(\x) := \beta(\W^k, \Phi^{[k-1]}(\x))$ and $\theta^k(\x, \vc{s}) := \textsc{sort}(\beta^k(\x), s^k)$.
Associated with the critical angular distance $\theta^k(\x, \vc{s})$, at layer $k$, there is an index set $I^k = \textsc{Top-k}(\beta^k(\x), s^k)$ such that, for each $i\in I^k$, 
\[
\frac{\Wvec{k}{i} \Phi^{[k-1]}(\x)}{\mathsf{M}^k_{\cW}\zeta^{k-1}\left({\vc{0}}\right)
} = 
\cos\left(\pi \cdot \beta^k(\x)_i \right)
\leq 
\cos{\left(\pi \cdot \theta^k(\x, \vc{s})\right)}.
\] 
Therefore, if $\cos{\left(\pi \theta^k(\x, \vc{s})\right)} < 0$ (i.e. if $\theta^k(\x, \vc{s}) \geq \frac{1}{2}$) the set $I^k$ is inactive for 
$\Phi^{[k]}(\x)$\footnote{
In the presence of appropriately scaled and nonzero bias $\vb^k$, the lower bound here becomes
\[
\theta^k(\x) \geq \frac{1}{\pi}\cdot \arccos\Bigg(\frac{-\vb^k_i}{\mathsf{M}^k_{\cW} \zeta^{k-1}\left({\vc{0}}\right)}\Bigg).
\]
}.
Thus, 
critical angular distance can capture the existence of a \emph{stable} inactive index sets. 

\begin{lemma}
\label{lemma: robust-local-lip-fnn}
A feedforward neural network $h \in \cH^{K+1}$ is $\vc s$-robust sparse local Lipschitz w.r.t. parameters, 
with scale $\lparnu(\x, \vc{s}) := (K+1) \zeta^{(K+1)}(\vc{s}) \cdot (1+\nu)$ and radius
\[\rparnu(\x, \vc{s}) := \min_{1\leq k \leq K} \frac{ \iota(\vc{s}) + 
\max\left \{0, -\cos(\pi\theta^k(\x,\vc{s}) - \nu\right\}}{k(1+\nu)}. \]
Here $\iota(\cdot)$ is the extended indicator\footnote{For all $\vc{s} \succ \mathbf{0}$, $\iota(\vc{s}) = 0$ and at the trivial choice sparsity levels, $\iota(\vc{0}) = \infty$.} 
 function of the positive orthant $\mathbb{R}^{K}_{+}$.
\end{lemma} 
\addition{
\textit{Proof sketch.}  In a nutshell, this result can be shown by noting that for the defined radius $\rparnu(\x, \vc{s})$, 
the critical angular distance $\Theta^k(\x, \vc{s})$ at each layer is sufficiently large so that rows corresponding to 
strongly inactive set $I^k$ 
remain inactive upon perturbations to the model weights (by no more than $\rparnu(\x, \vc{s})$). One can then notice that since the strongly inactive sets at each layer are maintained, one can follow similar logic to \Cref{lemma: robust-global-lip-fnn} to obtain a Lipschitz scale dependent on the operator norms of the reduced weights at each every layer. The scaling factor $k \cdot (1+\nu)$ stems from requiring larger critical angle distance in last few layers to withstand the multiplicative effect of perturbation in composed predictors such as $\Phi^{[K]}$. 
The proof is simple, alas somewhat long, and so we defer the full version to 
\Cref{app: fix1}}

To interpret \Cref{lemma: robust-local-lip-fnn}, consider a fixed parameter radius $\epsilon$. If the angular distances of the farthest $s^k$ vectors at each layer $k$ is sufficiently large -- i.e., if the $s^k$ vectors have a sufficiently negative correlation with the input --  then they can withstand
input perturbation of magnitude $\nu$, 
and parameter perturbation of magnitude $\epsilon$,
to still remain inactive. 
If at input $\x$, the robust local radius $\rpar(\x, s) < \epsilon$ for all non-trivial sparsity levels $s>0$, then one cannot guarantee the preservation of a (non-trivial) inactive set at any layer. 
In this case, the distance between predictors outputs is bounded by 
$\Lparnu = \lparnu(\x, \vc{0})$. 

\paragraph{Reduced Model Perspective}\label{subsec: rob-gen-reduced-model}
If the robust sparse local radius $\rparnu(\x, \vc{s}) > 0$, 
\Cref{lemma: robust-local-lip-fnn} establishes 
the existence of stable inactivity of index sets $I^k = \textsc{Top-k}(\beta^k(\x), s^k)$ at each layer. 
In such an event, the effect of $\hat{h} \in \cH^{K+1}$ with perturbed parameters (within the radius) on any perturbed input $\x+\bdel$ can be reduced to a predictor $\hat{h}_{\mathrm{red}}(\x+\bdel)$ using only the reduced weight matrices $\What_{\mathrm{red}}^k  =\mathcal{P}_{J^k, J^{k-1}}(\What^k) \in \mathbb{R}^{(d^k-s^k) \times (d^{k-1}-s^{k-1})}$, 
\begin{equation*}
\hat{h}(\x +\bdel) 
= 
(\What^{K+1}_{\mathrm{red}})^T
  \act{\What^K_{\text{red}} \;
     \cdots \act{ \What^1_{\text{red}} \;(\x+\bdel) } \cdots 
     } 
=: \hat{h}_{\mathrm{red}}(\x+\bdel).
\end{equation*} 
As in \Cref{subsubsec: cert-rob-reducedmodel}, the reduced predictor is still nonlinear and
the index sets $J^k = (I^k)^c$ are determined by the input $\x$ and unperturbed predictor $h$. 
The above reduction is of course also true for the original predictor.
At a fixed sparsity vector $\vc{s}$, although the index sets vary across inputs, we can analyze and bound the Lipschitz constant of a worst-case such reduction. 
This explains the independence of the robust local Lipschitz scale on a specific input and the utilization of $\zeta^{[K+1]}(\vc{s})$, an upper bound on the worst-case product of operator norms of reduced weights.


As per \Cref{eq:optimal_set_sparsity} let the robust optimal sparsity level be $\vc{s} := {s}^*_{\mathrm{rob}}(\cV, \frac{1}{|\cV|})$\footnote{Note here the sparsity levels are vectors, and $s^{\star}_{\mathrm{rob}}(\cV, \epsilon)$ searches over layer-wise sparsity levels.}, the robust sparse regularity for a feedforward network predictor $h$ is,\vspace{-5pt}
\begin{equation}\label{eq:sparse-reg-fnn} \vspace{-5pt}
\mathcal{L}_{\mathrm{rob}}(h, \cV,\frac{1}{|\cV|}) = \lparnu(\x, \vc{s}) = (K+1)\zeta^{K+1}(\vc{s}) \cdot (1+\nu) 
\end{equation}
This quantifies the worst-case local Lipschitz scale of the predictor $h$ at any input $\x$ where it has a sufficiently large local radius, is a data-dependent norm-based regularity measure that can be much smaller than $\lparnu(\x, \mathbf{0})$, depending on the set $\cV$, and scales linearly with the adversarial energy $\nu$. 
We are finally ready to present the main result bounding the robust generalization error of feedforward networks. 

\begin{theorem}
\label{thm: nonuniformriskriskMNN}
With probability at least $(1-\alpha)$ over the choice of i.i.d training sample $\texttt{S}_{T}$ and unlabeled data $\texttt{S}_{U}$, each of size $m$, 
for any feedforward network predictor $h \in \cH^{K+1}$, 
\vspace{-5pt}
\begin{equation*}\vspace{-15pt}
R_{\mathrm{rob}}\left(h\right) 
\leq  
\hat{R}_{\mathrm{rob}}\left(h\right) + \mathcal{O}\left(
b \sqrt{\frac{ \mathcal{E}_{\mathcal{N}}(\mathcal{H},m) + \ln(\frac{2}{\alpha})}{2m}} 
+  \frac{\Lloss \cdot \mathcal{L}_{\mathrm{rob}}(h, \samp_T \cup \samp_U,\frac 1{2m})}{m (K+1)}\right),
\end{equation*}
where 
$\mathcal{E}_{\mathcal{N}}(\mathcal{H},m)=\ln\left(\mathcal{N}\left(\frac{1}{m(K+1)},\cH^{K+1}\right)\right)$ is the log of the covering number. 
\end{theorem}
The robust sparse regularity $\mathcal{L}_{\mathrm{rob}} (h, \samp_T \cup \samp_U, \frac{1}{2m})$ for a feedforward neural network is solely determined by the sparsity levels $\vc{s}$ via $\zeta^k(\vc{s})$ 
, the worst-case operator norm of \textbf{any} reduced layer weight in $\cW^k$. 
One can tune this result to be dependent on a specific trained predictor 
(see \Cref{app: thm: nonuniformriskriskMNN-predictor-dependent}). 
We state a specific instance of such an improved result. 


\begin{corollary}
\label{corollary: nonuniformriskriskMNN}
With probability at least $(1-\alpha)$ over the choice of training data $\texttt{S}_{T}$ and unlabeled data $\texttt{S}_{U}$ each of size $m$, for any soft-margin threshold $\gamma > 0$, for any feedforward neural network $h \in \cH^{K+1}$, 
there exists layer-wise sparsity levels $\vc{s}=[s^1, \ldots, s^K]$ dependent on data $\samp_T \cup \samp_U$ so that the probability of robust misclassification is bounded, 
\begin{equation*}
R_\text{rob}^{[0/1]}(h) \leq \hat{R}^{\gamma}_\text{rob}(h)
+  \tilde{\mathcal{O}}
\Bigg(
 \sqrt{\frac{
 \mathcal{E}_{\mathcal{N}}(\mathcal{H},m)+
\ln(\frac{2}{\alpha})}{m}} 
+ 
\frac{(1+\nu)}{\gamma m}\prod_{k=1}^{K+1} \norm{\W^k}_{2,\infty}\sqrt{1+\mu_{s^k,s^{k-1}}(\W^k) 
}
\Bigg)
\end{equation*}
where $R_\text{rob}^{[0/1]}(h)$ denotes the robust risk for the zero-one loss $\ell^{[0/1]}$ (i.e. probability of robust misclassification at any input)\footnote{Formally, 
$R_\text{rob}^{[0/1]}(h) = \expect_{(\x,y) \sim \cD_\cZ} 
\bigg[
\max_{\bdel \in \mathcal{B}^{\cX}_\nu(\x)} \mathbf{1}\left\{y \neq \argmax_j h(\x+\bdel)\right\}
\bigg]
$. 
}, $\hat{R}^{\gamma}_\text{rob}(h)$ is the robust empirical risk for the margin loss $\ell^{\gamma}$. 
\end{corollary}
In the above theorem $\tilde{\mathcal{O}}(g)$ denotes complexity that suppresses log factors\footnote{For functions $f$ and $g$ with the same arguments, $f = \tilde{\mathcal{O}}(g)$ if there exists a constant $C$ such that for any sequence of arguments $\{t^j\}_{j\geq \infty}$ such that $t^j \rightarrow \infty$ we have that $\limsup_{j\rightarrow \infty} \frac{f(t^j)}{g(t^j)\mathrm{polylog}(g(t^j)} \leq C$.}\cite{Bartlett2017SpectrallynormalizedMB}. We have thus characterized the robust generalization ability of predictors that are SLL w.r.t. parameters. 
\paragraph{Comparison to Related Work}
Unlike prior work in robust generalization \cite{yin2018rademacher, Awasthi2020AdversarialLG} our bound has a milder dependence on the adversarial energy, \addition{$\mathcal{O}(\frac{\nu}{m})$ rather than $\mathcal{O}(\frac{\nu}{\sqrt{m}})$.} 
The full proof found in \Cref{app: corollary: nonuniformriskMNN} 
employs logic akin to structural risk minimization inspired from \cite{Bartlett2017SpectrallynormalizedMB}.
The second term in the bound is a predictor dependent instantiation of the robust sparse regularity normalized by margin threshold,  
bearing resemblance to other spectrally normalized margin bounds \cite{ Awasthi2020AdversarialLG, Bartlett2017SpectrallynormalizedMB, Golowich2018SizeIndependentSC, khim2018adversarial,  Neyshabur2018APA, Neyshabur2015NormBasedCC, yin2018rademacher}.
The reduced babel function captures the coherence between any row and a subset of other rows (with an additional column restriction). 
Coupled with the group norms, as per \Cref{lemma: bound-submatrix-norm}, the second term in the above bound scales as the product of operator norms of reduced-linear maps, rather than the full weight matrices $\W^k$. 
The sparsity level $\vc{s}$ above is determined by the training sample $\samp_T$ and the unlabeled data $\samp_U$, as $\vc{s} = s^*_{\mathrm{rob}} (\samp_T\cup \samp_U, \frac{1}{2m})$. 
Only for a worst case choice of data distribution and trained network, the sparsity levels are trivial ($\vc{s}=\mathbf{0}$) in which case one recovers a result that only depends on the global Lipschitz constants. 

Our analysis studies robust generalization using both the favorable properties of a training data and the local sensitivity of a trained predictor, and is closest in spirit to 
results on standard (benign) generalization 
\cite{Banerjee2020DerandomizedPM, Nagarajan2019DeterministicPG,ShaweTaylor1997APA,Wei2019DatadependentSC}. 
We note that the bounds in \cite{Banerjee2020DerandomizedPM, Nagarajan2019DeterministicPG, Wei2019DatadependentSC} do not have explicit dependence on the number of parameters. \addition{In contrast, a bound on the covering number term $\mathcal{N}(\frac{1}{m(K+1)}, \cH^{K+1})$ in \Cref{corollary: nonuniformriskriskMNN} can be obtained by parameter counting (see \Cref{lemma:covermnn}), i.e. 
$\mathcal{E}_{\mathcal{N}}(\cH^{K+1}, m) \propto \mathcal{O}( \log(m) + \sum_{k=1}^{K+1} d^k d^{k-1})$.
This renders our result vacuous in the over-parameterized regime when $\sum_{k=1}^{K+1} d^k d^{k-1} \gg m$.}
We believe this is a limitation of the current proof technique rather than intrinsic to our sensitivity analysis, and conjecture that combining SLL as per \Cref{lemma: robust-local-lip-fnn} with the other learning theoretic approaches could remove this dependence on dimensions. We leave this extension to future work. 

Similar to our analysis, \cite{Nagarajan2019DeterministicPG} captures a reduced dimensionality of neural networks and
bound the generalization error of the original deterministic network by derandomizing the standard PAC-Bayesian bound. 
Their analysis weakens the exponential dependence on depth (i.e. the global Lipschitz constant) and does not require an additional unlabeled dataset.
However, their bound depends inversely on the minimum absolute pre-activation level in each layer -- which can be arbitrarily small in practice. 
The sensitivity of the predictor is quantified assuming that the original active sets $\mathcal{I}^k(\x)$ at each layer remain unchanged upon a Gaussian perturbation (equivalent to requiring $\rpar(\x, \bar{\vc{s}}(\x)) > 0$, in our notation).
This presents a rather strong condition, and our analysis moves past this limitation. Additionally, our results also hold for the robust adversarial setting.  
The analysis in \cite{Wei2019DatadependentSC}, on the other hand,  links the parameter sensitivity of predictors to generalization using an augmented loss function that encourages favorable data-dependent properties, such as low Jacobian norms. 
Their bound also avoids exponential dependence on depth but is restricted to smooth activations (and the benign, non-adversarial setting). 
The work in \cite{Banerjee2020DerandomizedPM} presents an alternative approach by studying the curvature of the loss as given by the Hessian.
While their bound avoids explicit dependencies on the global Lipschitz constant, it is unclear if all dependence on the latter is avoided in their characterization of the failure probability.


\section{Experimental Results}
\label{sec:experimental}
In this section we showcase the potential benefits of 
sparse local Lipschitz analysis. We compute sparse the certified radius and sparse regularity for feedforward neural networks trained for classification on MNIST and SVHN datasets. 
For more complex tasks, such as ImageNet, one needs additional architectural choices like convolution and pooling, 
which we regard as future extensions to our work.
\addition{
\paragraph{Training Setup}
We train feedforward networks $h$ with weights $\{\W^k\}_{k=1}^{K+1}$ where $\W^k \in \mathbb{R}^{d^k \times d^{k-1}}$ using the cross-entropy loss with stochastic gradient descent (SGD) with default hyper-parameter settings in PyTorch for 2,000 steps with a batch size of 100.
Each network is trained with \textit{Orthogonal Frame Regularization}, 
a measure suggested in \cite{pmlr-v70-cisse17a} for improving robustness that encourages normalized layer weights to be near orthogonal. 
All trained models\footnote{The extra regularization term does not increase the computational cost of training. }
$h_{\eta}$ are described as follows,
\begin{equation*}
h_{\eta} \in \argmin_{\{\W^k\}_{k=1}^{K+1}} ~~ \frac{1}{m}\sum_{i=1}^m \ell \Big(h, (\x_{i},y_{i})\Big) + \frac{\eta}{K+1}\sum_{k=1}^{K+1} \norm{I - \tilde{\W}^k(\tilde{\W}^k)^T}^2_F. 
\end{equation*}
Here $\tilde{\W}^k$ has normalized rows $\tilde{\w}^{k}_{i} := \frac{\Wvec{k}{i}}{\norm{\Wvec{k}{i}}_2}$. 
We study models trained with 4 different choices of $\eta \in \{0, 0.001, 0.01, 0.1\}$. 
For both MNIST and SVHN datasets, the official training sets are randomly split into train and validation data (55,000:5,000 for MNIST, 61,257:12,000 for SVHN). The models are optimized on the training data and the resulting measures are computed on validation data. 
}

\paragraph{Certified Radius}
For any network, at each input $\x$ there exists a true robust radius $\bar{r}(\x)$ -- the minimal energy required for a successful adversarial perturbation.
The naive lower bound for the certified radius, $r_{\mathrm{global}} (\x) :=r_{\mathrm{cert}}(\x, \vc{0})$, is computed from \Cref{corollary:cert-rob-fnn-loc} at the trivial choice of sparsity. 
The SLL certificate, $r_{\mathrm{sparse}}(\x)$, is obtained by binary search (with tolerance $10^{-6}$) using 
the optimal sparsity $\vc{s}$ is computed by 
\Cref{alg: greedy-sparse}.
The SLL certificate relies on the product of operator norms of specific reduced linear maps.
These estimates provide a lower bound $r_{\mathrm{global}}(\x) \leq  r_{\mathrm{sparse}}(\x) \leq \bar{r}(\x)$. 
While the value of $\bar{r}(\x)$ is not  computable\footnote{Indeed, this is NP hard as it involves the optimization of a non-convex loss.}, one can obtain a surrogate upper bound $r_{\text{adv}}(\x) \geq \bar{r}(\x)$ by measuring the minimal size of an adversarial example found via an ensemble of popular (and effective) adversarial attack strategies such as PGD \cite{Madry2018TowardsDL}, Carlini-Wagner \cite{carlini2017evaluating}, etc.

\paragraph{Benchmark via Security Curves}
Using a certified radius, one can compute the \textit{certified accuracy} of a collection of inputs, measured as the fraction of samples that are certified to predict faithfully against a specified size of corruption. 
The robustness of the trained networks (with 2 hidden layers of dimension 500), on both MNIST and SVHN datasets are shown in \Cref{fig:mnist_sec,fig:svhn_sec} via \textit{security curves}, which plot the obtained certified accuracy for increasing size of adversarial perturbations. We compare our ``SparseLip'' certification based on $r_{\mathrm{sparse}}$ with other state-of-the-art certification algorithms that are based on Lipschitz constant
estimation: FastLip \cite{Weng2018TowardsFC} and Recurjac \cite{Zhang2019RecurJacAE}. Both methods upper bound the Jacobian norm using entry wise bound propagation. 

The results in \Cref{fig:mnist_sec,fig:svhn_sec} demonstrate that orthogonal frame regularization results in improved robustness, as seen via the certified accuracy and the robust accuracy under attack. 
RecurJac provides the best certification for models trained without regularization, while SparseLip is comparable to FastLip for the same setting. 
For regularized models, SparseLip provides the best certified accuracy for MNIST, and is among the best for SVHN. The performance of both RecurJac and FastLip drops significantly for regularized models.
Finally, \Cref{fig:time_taken} depicts the considerable computational benefit of our approach, SparseLip.  


\begin{figure}[]
    \centering
    \subfloat[MNIST]{\includegraphics[scale=0.45]{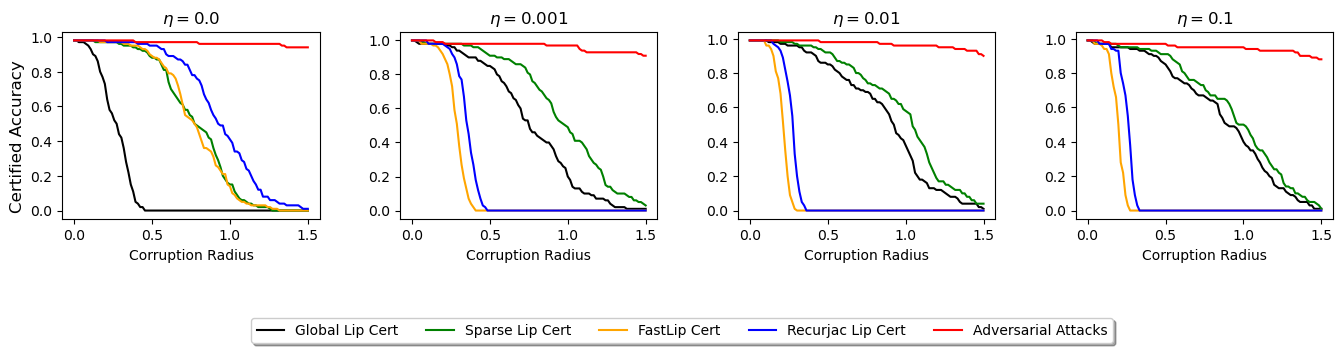}
    \label{fig:mnist_sec}}\\
    \subfloat[SVHN]{
    \includegraphics[scale=0.45]{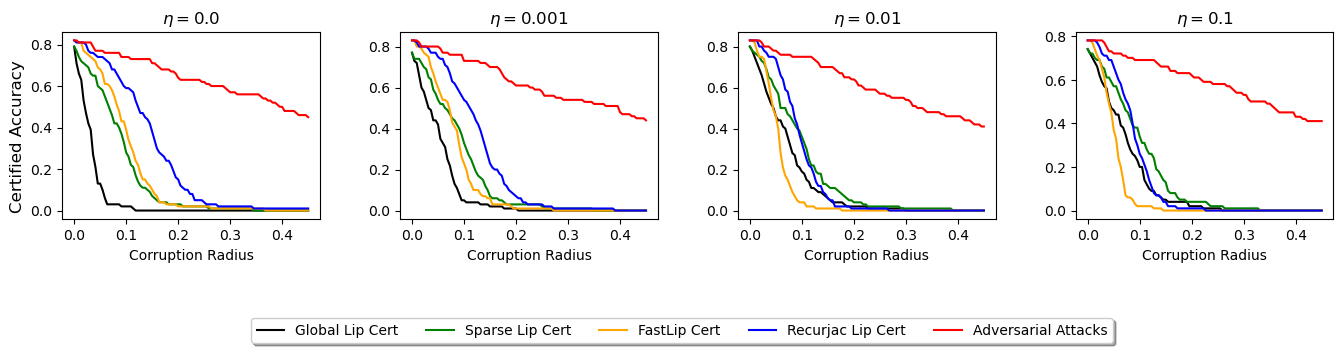}
    \label{fig:svhn_sec}}
    \caption{Security curves for feedforward neural networks with layer widths [500,500]}
    \label{fig:security_curve}
\end{figure}

\begin{figure}[]
    \centering
    \subfloat[MNIST]{\includegraphics[trim = 0 10 0 0 ,width=0.48\textwidth]{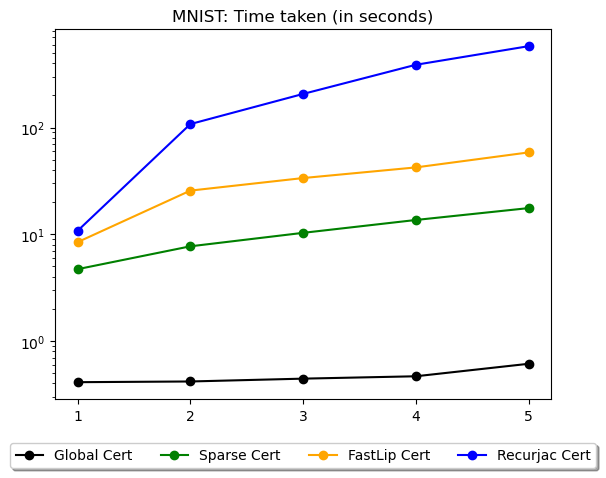}
    \label{fig:mnist_time}}
    \subfloat[SVHN]{
    \includegraphics[trim = 0 10 0 0 ,width=0.48\textwidth]{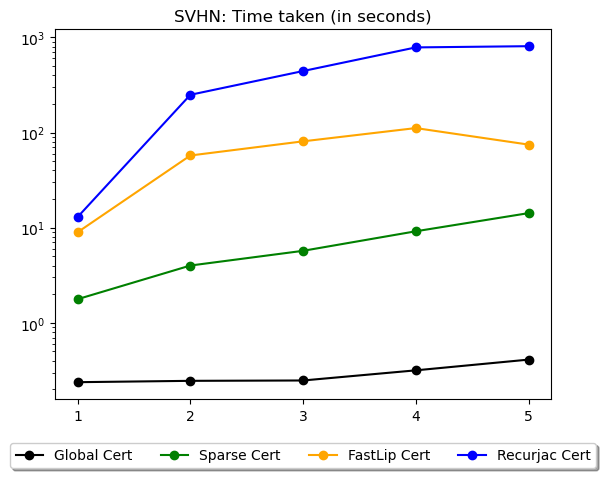}
    \label{fig:svhn_time}}
    \caption{Time taken (y-axis) for certification on batch of 100 inputs for models with layer dimensions $[100]*(K+1)$ with varying depth $K$ (x-axis).}
    \label{fig:time_taken}
\end{figure}


\begin{figure}[]
    \centering
    \subfloat[MNIST]{\includegraphics[trim = 0 10 0 0 ,width=0.48\textwidth]{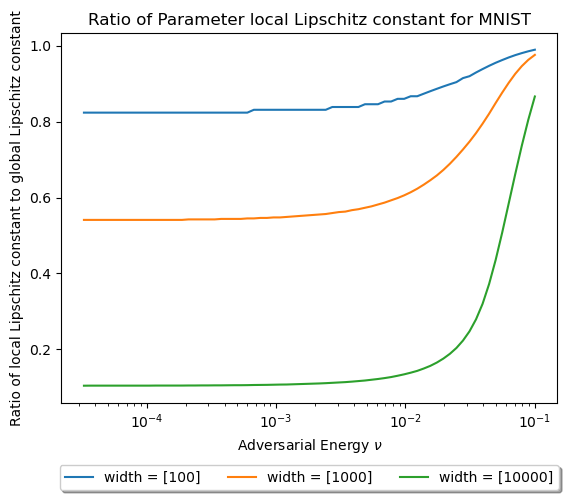}
    \label{fig:mnist_param}}
    \subfloat[SVHN]{
    \includegraphics[trim = 0 10 0 0 ,width=0.48\textwidth]{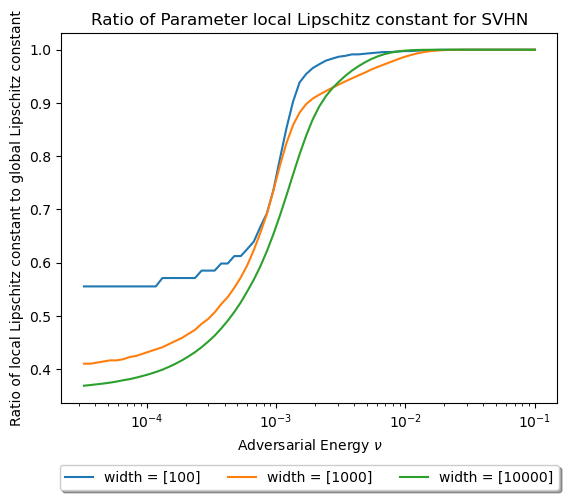}
    \label{fig:svhn_param}}
    \vspace{-10pt}
    \caption{Ratio of Lipschitz constants, $\frac{\mathcal{L}_{\mathrm{rob}}(h, \mathcal{V}, \frac{1}{|\cV|(K+1)})}{\mathsf{L}_{\mathrm{param,\nu}}}$ 
    for networks of varying widths.}
    \label{fig:param}
\end{figure}
\begin{figure}[]
    \centering
    \subfloat[MNIST]{\includegraphics[trim = 0 10 0 0 ,width=0.48\textwidth]{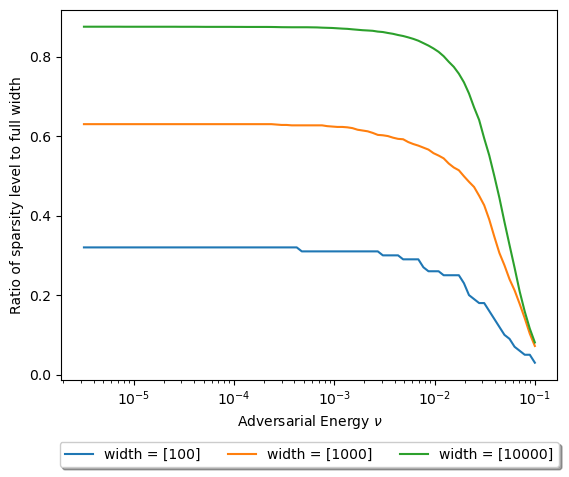}
    \label{fig:mnist_width_sparse}}
    \subfloat[SVHN]{
    \includegraphics[trim = 0 10 0 0 ,width=0.48\textwidth]{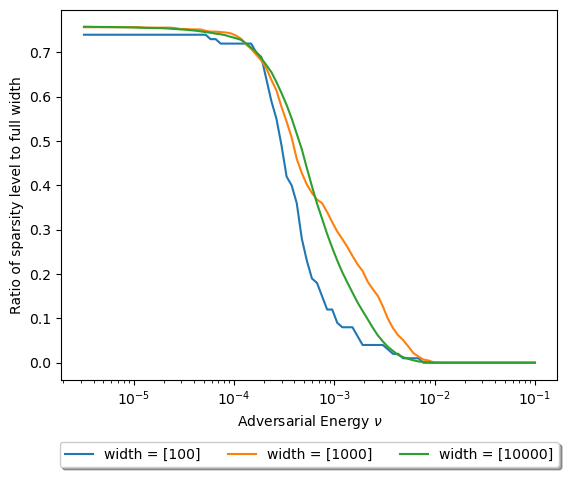}
    \label{fig:svhn_width_sparse}}
    \vspace{-10pt}
    \caption{Ratio of optimal sparsity to width $d$, $\frac{s^*_{\mathrm{rob}}(h, \mathcal{V}, \frac{1}{|\cV|(K+1)})}{d}$ 
    for networks of varying widths.}
    \label{fig:param_width}
\end{figure}
\paragraph{Sparse local Lipschitzness w.r.t Parameter}
We finally study how SLL analysis can aid the study of generalization, by considering a 1-hidden layer feedforward networks of different widths trained via SGD without regularization.
Across varying adversarial energy levels $\nu$, \Cref{fig:param,fig:param_width} plots the robust sparse regularity $\mathcal L_{\mathrm{rob}}(h,\cV,\epsilon)$ and robust optimal sparsity level $ s^*_{\mathrm{rob}}(h,\cV,\epsilon)$ w.r.t validation data $\cV$  as defined in  \Cref{eq:robustsparseregularity} at $\epsilon = \frac{1}{|\cV|(K+1)}$. 
We note that for large enough $\nu$ the SLL sensitivity is equivalent to global Lipschitz analysis (and correspondingly the optimal sparsity level approaches 0), but for moderate values of $\nu$ the robust sparse regularity can be significantly better. 
\Cref{fig:param} and \Cref{fig:param_width} demonstrate the observation in networks with large widths, and the ability of SLL analysis to capture the reduced local sensitivity. Importantly, the wider the network, the more significant the reduction in the equivalent Lipschitz scale of the model.


\section{Conclusions}
\label{sec:conclusions}
In this work we study adversarial robustness via the lens of sparse local Lipschitzness (SLL). 
We show that feedforward neural networks are SLL and  equivalent to a reduced \emph{nonlinear} mapping with decreased sensitivity in a local neighborhood around each input. 
In using Lipschitzness properties locally rather globally, and benefiting from sparse structures, our approach provides an improved certified radius at any input, and bounds on the robust generalization error  with only a mild dependence on the adversarial corruption.
Our work is a step towards producing data-dependent non-uniform bounds that leverage the favorable properties of a trained predictor on a sample data.
We believe that the ideas presented here are extensible to other hypothesis classes that encourage other structural priors, such as convolutional, attention or graph neural networks. 
The identification of the particular reduced models for each class presents an intriguing topic of future research. 

\section*{Acknowledgments}
This work was partially supported by NSF grant CCF 2007649, DARPA GARD award HR00112020004, and NSF TRIPODS CCF-1934979 via the MINDS Fellowship award. The authors thank Raman Arora for useful comments in early stages of this work.


\bibliographystyle{plain}
\bibliography{references}
\newpage
\appendix 

\section{Missing Proofs}\label{appendix:proofs}
In this section we provide detailed proofs of all our results. 
\begin{table}[H]
    \centering
    \begin{tabular}{|c|c|}
    \hline 
    Result & Proof \\ 
    \hline
    \Cref{lemma:cert-rob-gen-loc}     & \Cref{app: lemma:cert-rob-gen-loc} \\
    \Cref{lemma: sparse-local-lip-composition}     &  \Cref{app: lemma: sparse-local-lip-composition}\\
    \Cref{lemma: fnn-layer-lip} & \Cref{app: lemma: fnn-layer-lip}\\
    \Cref{corollary:cert-rob-fnn-loc} & \Cref{app: corollary:cert-rob-fnn-loc}\\
    \Cref{lemma: bound-submatrix-norm} & \Cref{app: lemma: bound-submatrix-norm}\\
    \Cref{thm: generalization-sll} & \Cref{app: thm: generalization-sll}\\
    \Cref{lemma: robust-global-lip-fnn} & \Cref{app: lemma: robust-global-lip-fnn}\\
    \Cref{lemma: robust-local-lip-fnn} & \Cref{app: fix1}\\
    \Cref{thm: nonuniformriskriskMNN} & \Cref{app: thm: nonuniformriskMNN}\\
    \Cref{corollary: nonuniformriskriskMNN} & \Cref{app: corollary: nonuniformriskMNN}\\
    \hline
    \end{tabular}
    \vspace{0.1cm}
    \caption{Proof Glossary}
    \label{tab:proof_glossary}
\end{table}

\subsection{Proof for \Cref{lemma:cert-rob-gen-loc}}\label{app: lemma:cert-rob-gen-loc}
We define the margin operator $\mathcal{M} : \cY' \times \cY \rightarrow \mathbb{R}$, $\mathcal{M}(\vt, y) := [\vt]_{y} - \max_{j \neq y} [\vt]_j$. Thus, $\rho(\x) = \mathcal{M}(h(\x),\hat{y}(\x))$. 
Observe that the predicted labels remain unchanged when
$\mathcal{M}(h(\x+\bdel), \hat{y}(\x)) \geq 0$.
The margin operator $\mathcal{M}(\cdot, j)$ is $2$-Lipschitz in $\cYd$ w.r.t. $\ell_p$ norm for any $p\geq 1$ (see \cite[Lemma A.3.]{Bartlett2017SpectrallynormalizedMB}), thus
\[
 \mathcal{M}\big(h(\x),\hat{y}(\x)\big) - \mathcal{M}\big(h(\x+\bdel), \hat{y}(\x)\big) \leq 2 \norm{h(\x+\bdel) - h(\x)}_p.
\]
Note that for $\vc{s} = (0, s_{\mathrm{out}})$, there are no sparsity constraints on the perturbed input. 
Hence,
\[
\norm{\bdel}_2 \leq \rinp(\x,\vc{s})
\implies
\norm{h(\x+\bdel)-h(\x)}_2 \leq \norm{\A}_2 \cdot \linp(\x, \vc{s}) \norm{\bdel}_2.
\]
Thus, 
\[
\rho(\x) - \mathcal{M}\big(h(\x+\bdel), \hat{y}(\x)\big)
\leq 2 \norm{\A}_2 \cdot \linp(\x, \vc{s}) \norm{\bdel}_2.
\]
The conclusion follows from rearranging the above inequality. 

\subsection{Proof for \Cref{lemma: sparse-local-lip-composition}}
\label{app: lemma: sparse-local-lip-composition}
If each intermediate representation map $\Phi^{(k)}$ is $\Linp^{(k)}$-Lipschitz, the composed representation $\Phi^{[K]}$ is $\left(\prod_{k=1}^K \Linp^{(k)}\right)$-Lipschitz. Similarly, if each representation map is locally Lipschitz at $\x$, with local radius function $\rinp^{(k)}(\x)$ and local Lipschitz scale $\linp^{(k)}(\x)$, then the composed map is also local Lipschitz with local radius and local Lipschitz scale defined as
 \begin{align}\label{eq:compose-local}
     r_{\mathrm{inp}}^{[k]}\big(\x\big) := \min_{1\leq n \leq k}  \frac{r_{\mathrm{inp}}^{(n)}\Big(\Phi^{[n-1]}(\x) \Big) }{\displaystyle l_{\mathrm{inp}}^{[n-1]} \big(\x \big)},
     ~\quad~
     l_{\mathrm{inp}}^{[k]}\big( \x \big) := \prod_{n=1}^k l_{\mathrm{inp}}^{(n)}\Big(\Phi^{[n-1]}(\x)\Big). 
\end{align}
We will prove a more general result shortly, but for a proof sketch of this claim consider a perturbation $\bdel$ in the initial input. 
It is necessary to ensure that the local radius for layer $\rinp^{[k]}$ is such that at each layer $n \leq k$ the perturbed representation does not exceed the corresponding local radius, i.e. $\norm{\Phi^{[n-1]}(\x+\bdel) - \Phi^{[n-1]}(\x) }_2\leq \rinp^{(n)}(\Phi^{[n-1]}(\x))$. Since $\norm{\Phi^{[n-1]}(\x+\bdel) - \Phi^{[n-1]}(\x) }_2\leq \linp^{[n-1]}(\x)$, the local radius $\rinp^{[k]}$ as defined above satisfies the requirement. 
If each of the intermediate representation maps are SLL w.r.t. inputs, then by appropriately weaving the sparsity levels at each layer we can show that the composed representation map is also SLL. 

For $k=1$, there are no compositions. Note that $\vc{s}^{[1]}=\vc{s}^{(1)}$ and for all input $\x$, 
    \begin{align*}
        \Phi^{[1]} (\x) = \Phi^{(1)}(\x), \quad \rinp^{[1]}\left(\x, \vc{s}^{[1]}\right) = \rinp^{(1)}\left(\x, \vc{s}^{(1)}\right), \quad \linp^{[1]}\left(\x, \vc{s}^{[1]}\right) = \linp^{(1)}\left(\x, \vc{s}^{(1)}\right).
    \end{align*}

    \textbf{Base case ($k=2$)}:
    If $\xtil, \x \in \mathbb{R}^{d_0}$ are such that there exists a common inactive index set $I^0$ of size $s^0$ and $\norm{\xtil-\x}_2 \leq  \rinp^{[1]}(\x, \;\vc{s}^{[1]})$, 
    then the distance between the intermediate representations are bounded and there exists a common inactive index set $I^1$ of size $s^1$, 
    \begin{align*}
      \norm{\Phi^{[1]}(\xtil) - \Phi^{[1]}(\x)}_2 \leq \linp^{[1]}(\x, \;\vc{s}^{[1]})\norm{\xtil - \x}_2 \\
      \mathcal{P}_{I^1}\left(\Phi^{[1]}(\xtil)\right) = \mathcal{P}_{I^1}\left(\Phi^{[1]}(\x)\right) = \mathbf{0}.
    \end{align*}
    If the distance between the initial inputs is additionally bounded as
    \[
    \norm{\xtil - \x}_2 \leq \frac{\rinp^{(2)}\big(\Phi^{[1]}(\x),\; \vc{s}^{(2)} \big)}{\linp^{[1]}(\x, \;\vc{s}^{[1]})},\]
    then the distance between the representations $\Phi^{[1]}(\xtil)$ and $\Phi^{[1]}(\x)$ is below the local radius threshold for the layer map $\Phi^{(2)}$,
    \[
    \norm{\Phi^{[1]}(\xtil) - \Phi^{[1]}(\x)}_2 \leq 
    \rinp^{(2)}\big(\Phi^{[1]}(\x),\; \vc{s}^{(2)} \big).
    \]
    Thus, there exists a common inactive index set $I^{2}$ of size $s^{2}$ and the distance between the outputs of $\Phi^{[2]}$ can be bounded using the local Lipschitz scale, 
    \begin{align*}
     \norm{\Phi^{[2]}(\xtil)- \Phi^{[2]}(\x)}_2 
     &= \norm{\Phi^{(2)}\circ\Phi^{[1]}(\xtil) - \Phi^{(2)}\circ\Phi^{[1]}(\x)}_2 \\
     &\leq \linp^{(2)}\big(\Phi^{[1]}(\x),\; \vc{s}^{(2)} \big) \cdot \norm{\Phi^{[1]}(\xtil) - \Phi^{[1]}(\x)}_2 \\
     &\leq \prod_{n=1}^2 \linp^{(n)}\big(\Phi^{[n-1]}(\x),\; \vc{s}^{(n)} \big) \cdot \norm{\xtil - \x}_2
    \end{align*}
    Thus the expressions in the Lemma statement are valid for $k=2$.
    
    \textbf{Induction Case ($k > 2$)}: Assume that the lemma statement is valid for all $1\leq n\leq k$ for some $k\geq 2$. Thus, 
    \begin{align*}
    r_{\mathrm{inp}}^{[k]}\Big(\x,\;  \vc{s}^{[k]}\Big) := \min_{1\leq n \leq k}  \frac{r_{\mathrm{inp}}^{(n)}\Big(\Phi^{[n-1]}(\x),\; \vc{s}^{(n)} \Big) }{\displaystyle l_{\mathrm{inp}}^{[n-1]} \Big(\x,\;  \vc{s}^{[n-1]} \Big)},
    ~\quad~
    l_{\mathrm{inp}}^{[k]}\Big( \x,\; \vc{s}^{[k]} \Big) := \prod_{n=1}^k l_{\mathrm{inp}}^{(n)}\Big(\Phi^{[n-1]}(\x),\; \vc{s}^{(n)}\Big). 
    \end{align*}
    If $\xtil, \x \in \mathbb{R}^{d_0}$ are such that there exists a common inactive index set $I^0$ of size $s^0$ and $\norm{\xtil-\x}_2 \leq  \rinp^{[k]}(\x, \;\vc{s}^{[k]})$, 
    then the distance between the intermediate representations are bounded and there exists a common inactive index set $I^k$ of size $s^k$, 
    \begin{align*}
      \norm{\Phi^{[k]}(\xtil) - \Phi^{[k]}(\x)}_2 \leq \linp^{[k]}(\x, \;\vc{s}^{[k]})\norm{\xtil - \x}_2 \\
      \mathcal{P}_{I^k}\left(\Phi^{[k]}(\xtil)\right) = \mathcal{P}_{I^k}\left(\Phi^{[k]}(\x)\right) = \mathbf{0}.
    \end{align*}
    If the distance between the initial inputs is additionally bounded as
    \[
    \norm{\xtil - \x}_2 \leq \frac{\rinp^{(k+1)}\big(\Phi^{[k]}(\x),\; \vc{s}^{(k+1)} \big)}{\linp^{[k]}(\x, \;\vc{s}^{[k]})},\]
    then the distance between the representations $\Phi^{[k]}(\xtil)$ and $\Phi^{[k]}(\x)$ is below the local radius threshold for the layer map $\Phi^{(k+1)}$,
    \[
    \norm{\Phi^{[k]}(\xtil) - \Phi^{[k]}(\x)}_2 \leq 
    \rinp^{(k+1)}\big(\Phi^{[k]}(\x),\; \vc{s}^{(k+1)} \big).
    \]
    Thus, there exists a common inactive index set $I^{k+1}$ of size $s^{k+1}$ and the distance between the outputs of $\Phi^{[k+1]}$ can be bounded using the local Lipschitz scale, 
    \begin{align*}
     \norm{\Phi^{[k+1]}(\xtil)- \Phi^{[k+1]}(\x)}_2 
     &= \norm{\Phi^{(k+1)}\circ\Phi^{[k]}(\xtil) - \Phi^{(k+1)}\circ\Phi^{[k]}(\x)}_2 \\
     &\leq \linp^{(k+1)}\big(\Phi^{[k]}(\x),\; \vc{s}^{(k+1)} \big) \cdot \norm{\Phi^{[k]}(\xtil) - \Phi^{[k]}(\x)}_2 \\
     &\leq \prod_{n=1}^{k+1} \linp^{(n)}\big(\Phi^{[n-1]}(\x),\; \vc{s}^{(n)} \big) \cdot \norm{\xtil - \x}_2
    \end{align*}
    Hence, we can define a local radius and local Lipschitz scale for compositions up to $k+1$ layers as follows,
    \begin{align*}
    r_{\mathrm{inp}}^{[k+1]}\Big(\x,\;  \vc{s}^{[k+1]}\Big) &:= \min\left\{r_{\mathrm{inp}}^{[k]}\Big(\x,\;  \vc{s}^{[k]}\Big),   \frac{r_{\mathrm{inp}}^{(k+1)}\Big(\Phi^{[k]}(\x),\; \vc{s}^{(k+1)} \Big) }{\displaystyle l_{\mathrm{inp}}^{[k]} \Big(\x,\;  \vc{s}^{[k]} \Big)} \right\},\\
    l_{\mathrm{inp}}^{[k+1]}\Big(\x,\;  \vc{s}^{[k+1]}\Big) &:=
    l_{\mathrm{inp}}^{[k]}\Big( \x,\; \vc{s}^{[k]} \Big) \cdot  l_{\mathrm{inp}}^{(k+1)}\Big(\Phi^{[k]}(\x),\; \vc{s}^{(k+1)}\Big). 
    \end{align*}
    Hence proved.

\subsection{Proof for \Cref{lemma: fnn-layer-lip}}
\label{app: lemma: fnn-layer-lip}
With the local radius function $\rinp^{(k)}$ and inactive set $I^k$ defined above, 
for each inactive index $i \in I^k$, $\Wvec{k}{i} \vc{t} + \vb^k_i \leq - \norm{\Wvec{k}{i}}_2\cdot \rinp^{(k)}(\vc{t}, \vc{s}^{(k)})$. For a perturbed input $\vttil : \|\vttil - \vt\|_2 \leq \rinp^{(k)}(\vc{t}, \vc{s}^{(k)})$ and with a common inactive set $\mathcal{P}_{I^{k-1}}(\vc t) = \mathcal{P}_{I^{k-1}}(\vttil) = \bm{0},$
\begin{align*}
    \Wvec{k}{i} \vttil + \vb^k_i  
   &= \big( \Wvec{k}{i}\vc{t} + \vb^k_i \big) + \Wvec{k}{i} (\vttil-\vt) \\
   &\leq - \norm{\Wvec{k}{i}}_2\cdot \rinp^{(k)}(\vc{t}, \vc{s}^{(k)}) + \norm{\Wvec{k}{i}}_2 \norm{\vttil-\vt}_2 \leq 0.
\end{align*}
This implies that the set $I^k$ of size $s^k$ is also inactive for the perturbed representation $\Phi^{(k)}(\vttil)$. 
Thus, we have shown the existence of a common inactive set for both the original and the perturbed representation. To bound the distance between the representations, note that 
\begin{align*}
	\norm{\Phi^{(k)}(\vttil) - \Phi^{(k)}(\vc{t})}_2
	&= \norm{ \act{\W^k\vttil + \vb^k} - \act{\W^k \vc{t}+\vb^k}}_2 \\
	&\leq \norm{ \mathcal{P}_{J^k, J^{k-1}}(\W^k) \left(\vttil - \vc{t}\right)}_2 \quad {(\small \text{ignoring common inactive sets})}\\
	& \leq \norm{\mathcal{P}_{J^k, J^{k-1}}(\W^k)}_2 \cdot \norm{\vttil-\vt}_2.
\end{align*}
This concludes the proof.

\subsection{Proof for \Cref{corollary:cert-rob-fnn-loc}}
\label{app: corollary:cert-rob-fnn-loc}
To prove this statement, observe that \ref{lemma: fnn-layer-lip} demonstrates that each layer feed-forward neural networks is sparse local Lipschitz w.r.t. input. 
Additionally \Cref{lemma: sparse-local-lip-composition} shows that compositions of sparse local Lipschitz representations are sparse local Lipschitz.
To arrive at the exact expression, we first note that $s^0$ is fixed to 0 and hence $J^{0}=[d^0]$, thus removing any sparsity requirement on the input or the perturbation directly. 
A bound on the local Lipschitz scale of the intermediate layers is given by
$
\linp^{(k)}(\vc{t}, \vc{s}^{(k)}) := \underset{\substack{I^{k-1} \subseteq \mathcal{I}^{k-1}(\vc{t}),\\|I^{k-1}|=s^{k-1}}}{\max}~ \norm{\mathcal{P}_{J^k, J^{k-1}}(\W^k)}_2
$. However, in the case of composed representations $\Phi^{[k]}$, for $k\geq1$ there is a fixed reference set $I^{k-1}$ for each value of $s^{k-1}$ determined solely by the original input $\x$.
This reference set is obtained either as $I^{0}=\emptyset$, the empty set, or by the application of \Cref{lemma: fnn-layer-lip} to the previous layer $\Phi^{(k-1)}$.
Thus, a tighter local Lipschitz scale can be established resulting in
$\linp^{[k]}(\vc{x}, \vc{s}^{[k]}) :=\prod_{n=1}^k \norm{\mathcal{P}_{J^n, J^{n-1}}(\W^n)}_2.$
The final corollary statement is an application of \Cref{lemma:cert-rob-gen-loc} which  presents a robustness certificates for a general sparse local Lipschitz representation.

\subsection{Restatement of result from \cite{sulam2020adversarial}}
We adapt the results in \cite{sulam2020adversarial} to fit with the notation here and to demonstrate the overarching structure of these results. 
Given any input $\x \in \cX$ and its representation $\Phi(\x)$, we denote $\Lambda^{s} := \{ I \subset [p], |I| = s\}$, a set of index sets. 
\begin{lemma}\cite[Lemma 5.2]{sulam2020adversarial})\label{app: lemma-sulam-adv}
The representation $\Phi$ is sparse local Lipschitz\footnote{We only discuss the case when $s_{\mathrm{in}}=0$.} with local radius threshold $\rinp$ and local Lipschitz function defined $\linp$ defined by
\begin{align*}
    \rinp(\x, \vc{s}) := \max_{I \in \Lambda^{s_{out}}} \min_{i \in I}~ \frac{1}{2}\big(\lambda - \left \lvert \langle \ten{w}_i, \x - \W \Phi(\x) \rangle \right \rvert \big),
    \quad 
    \linp(\x, \vc{s}) := \frac{1}{\sqrt{1-\eta_{p-s_{out}}}}.
\end{align*}
\end{lemma}
Here, for any corruption $\bdel$ with energy $\nu < \rinp(\x, \vc{s})$, there exists an index set $I$ that is inactive for both the original representation $\Phi(\x)$ and the corrupted representation $\Phi(\x+\bdel)$.

\subsection{Proof for \Cref{lemma: bound-submatrix-norm}}\label{app: lemma: bound-submatrix-norm} 
We prove the lemma by considering two cases based on the sparsity level $s^k$. 

\textbf{Case I} : $s^k=d^{k}-1$.

Here $J^k$ is a singleton set $\{j\}$ for some $j\in [d^k]$. Hence, $\mathcal{P}_{J^k, J^{k-1}}(\W)$ corresponds to row $j$ of the matrix $\W$ restricted further by the column index set $J^{k-1}$. 
For this case we can bound the operator norm by noting that, 
\[
\norm{\mathcal{P}_{J^k, J^{k-1}}(\W)}_2 = \norm{\mathcal{P}_{J^{k-1}}(\Wvec{}{j})}_2 \leq  \mathsf{M}^k_{\cW} = \sqrt{1+\mu_{(s^{k},s^{k-1})}(\W)}\cdot \norm{\W}_{2,\infty}.
\]
The last equality above holds since $\mu_{(s^k,s^{k-1})}(\W) = \mu_{(d^{k-1},s^{k-1})}(\W) = 0$.

\textbf{Case II} : $0\leq s^k < d^k-1$

To bound the operator norm of the reduced matrix $\mathcal{P}_{J^k,J^{k-1}}(\W)$ it is sufficient to bound the maximal eigenvalue $\lambda_{\max}$ of its Gram matrix.
By the Gerschgorin Disk Theorem we observe the following upper bound, 
\begin{align*}
& \norm{\mathcal{P}_{J^k,J^{k-1}}(\W)}^2_2 \\
&\leq \lambda_{\max}\left([\mathcal{P}_{J^k,J^{k-1}}(\W)]^T[\mathcal{P}_{J^k,J^{k-1}}(\W)]\right)\\
&= \lambda_{\max}\left([\mathcal{P}_{J^k,J^{k-1}}(\W)][\mathcal{P}_{J^k,J^{k-1}}(\W)]^T\right)\\
&\leq \Big( \lvert\langle \mathcal{P}_{J^{k-1}}(\Wvec{}{j}),~ \mathcal{P}_{J^{k-1}}(\Wvec{}{j})\rangle\rvert + \sum_{i\neq j,~ i\in J^k}  \lvert\langle \mathcal{P}_{J^{k-1}}(\Wvec{}{i}),~ \mathcal{P}_{J^{k-1}}(\Wvec{}{j})\rangle\rvert \Big)\;\; (\text{for some index j}\in J^k)\\
&\leq \Big(1+\mu_{(s^k, s^{k-1})}(\W)\Big)\cdot \norm{\W}_{2,\infty}^2 .
\end{align*}
In the above lines, the index $j \in J^k$ corresponds to the disk in which the maximum eigenvalue lies. The last line follows from the definition of the reduced babel function.

\paragraph{Further Discussion}
The above result is in the same spirit as upper bounds on singular values based on standard babel function and mutual coherence bounds \cite{tropp2003improved}. 
To contextualize further, a standard bound for the operator norm of a sub matrix is,
\[
\norm{\mathcal{P}_{J_1,J_2}(\W)}_2 \leq \norm{\mathcal{P}_{J_1,J_2}(\W)}_F \leq \norm{\mathcal{P}_{J_1}(\W)}_F \leq  \sqrt{|J_1|}\norm{\W}_{2,\infty}.
\]
Thus, the result in \Cref{lemma: bound-submatrix-norm} is tighter when $\mu_{s_1,s_2}(\W) < |J_1| -1$.
For matrices $\W$ with near orthogonal rows, the mutual coherence of $\W^T$ is low, or equivalently, $\mu_{s_1,0}(\W) \approx 0$. 
Our analysis further extends this notion to matrices with rows whose worst-case sub vectors are near orthogonal. 

To generate a certified radius for a large batch of inputs,
at each input $\x$ one can evaluate the specific index set $I^k$ 
and the corresponding cumulative local radius for a given sparsity level, $\rinp^{[k]}(\x, \vc{s}^{[k]})$, at each layer $k$,
 and then compute the cumulative Lipschitz constant $\linp^{[k]}(\x, \vc{s}^{[k]})$. 
The local radius $\rinp^{(k)}$ can be obtained by sorting the normalized vector of layer pre-activations, which is fast to compute. 
The local Lipschitz scale is the product of operator norms of sub-matrices that are different for each input. 
For ease of computation, it could be useful to modulate this dependence to have local Lipschitz scales independent of any particular input but dependent on the sparsity levels.

\begin{lemma}\label{app: lemma: fnn-layer-lip-weaker1}
For a feedforward neural network, the representation map $\Phi^{(k)}$ at layer $k$ is SLL w.r.t. inputs,
\[
I^{k} := \argmax_{\substack{I \subseteq \mathcal{I}^k(\vc{t}),\\ |I| = s^k}}~ \min_{i \in I}~ \frac{\left \lvert \Wvec{k}{i} \vc{t} + \vb^k_i \right \rvert}{\norm{\Wvec{k}{i}}_2}, \quad 
\rinp^{(k)}(\vc{t}, \vc{s}^{(k)}) := 
\min_{i \in I^k}~ \frac{\left \lvert \Wvec{k}{i} \vc{t} + \vb^k_i \right \rvert}{\norm{\Wvec{k}{i}}_2}
\]
\[
\linp^{(k)}(\vc{t}, \vc{s}) :=  \sqrt{1+\mu_{s_1, s_2}(\W)} \cdot \norm{\W}_{2,\infty}.
\]
\end{lemma}
\begin{proof}
   The proof follows from the same sequence of steps as in \Cref{lemma: fnn-layer-lip} and using the result from \cref{lemma: bound-submatrix-norm}. 
\end{proof}

\subsection{Difference of Loss}\label{diffofloss}
In this subsection we demonstrate the utility of the (robust) global Lipschitz constant $\Lpar$ ($\Lparnu$) and the (robust) local Lipschitz scales $\lparnu$
($\lparnu$).
We first show a simple lemma that bounds the difference in robust loss for Lipschitz representations.
\begin{lemma}
\label{lemma: global-local-lip}
    Consider any sample $\vc{z} = (\x,y)$, and any two predictors $h$, $\hat{h}$ in $\cH$. 
	Assume that hypothesis $h$ is global Lipschitz w.r.t. parameters with constant $\Lpar$.	The difference in losses is bounded,
	\[
	\Big\lvert
	\ell(h,\vc{z}) 
	- \ell(\hat{h},\vc{z}) \Big\rvert \leq \Lloss \Lpar \norm{\hat{h}-h}_\cH.
	\]
	If the hypothesis $h \in \cH$ is sparse local Lipschitz w.r.t. parameters. For an appropriate sparsity level $s \leq p-\|\Phi(\x)\|_0$,
\[
\norm{\hat{h}-h}_\cH \leq \rpar(\x, s) \implies 
\Big\lvert
\ell_{\nu}(h,\vc{z}) 
- \ell_{\nu}(\hat{h},\vc{z}) \Big\rvert \leq  \Lloss \cdot \lpar(\x,s) \norm{\hat{h}-h}_\cH. 
\]
\end{lemma}
\begin{proof}
    When $h$ is global Lipschitz, the conclusion follows from combining the Lipschitz property of the loss function and the hypothesis. 
    When $h$ is sparse local Lipschitz, by definition, 
\[
\norm{\hat{h}-h}_\cH \leq \rpar(\x,s) \implies \max_{\bdel} \norm{\hat{h}(\x+\bdel)- h(\x+\bdel)}_2 \leq \lpar(\x, s) \norm{\hat{h}-h}_\cH.
\]
\end{proof}

Next we extend these lemmas to the robust setting. 

\begin{lemma}\label{lemma: robust-global-local-lip}
Consider a hypothesis class $\cH$ with predictors that are global Lipschitz w.r.t. parameters with constant $\Lparnu$.
At any data point $\z=(\x,y) \in \cZ$, the difference in robust losses between two predictors $h, \hat{h} \in \cH$ is bounded as, 
\[
\Big\lvert
\ell_{\nu}(h,\vc{z}) 
- \ell_{\nu}(\hat{h},\vc{z}) \Big\rvert \leq  \Lloss \Lparnu \norm{\hat{h}-h}_\cH. 
\]
If the hypothesis $h$ is robust sparse local Lipschitz, then for an appropriate sparsity level $s \leq p-\|\Phi(\x)\|_0$,
\[
\norm{\hat{h}-h}_\cH \leq \rparnu(\x, s) \implies 
\Big\lvert
\ell_{\nu}(h,\vc{z}) 
- \ell_{\nu}(\hat{h},\vc{z}) \Big\rvert \leq  \Lloss \cdot \lparnu(\x,s) \norm{\hat{h}-h}_\cH. 
\],
\end{lemma}
\begin{proof}
   The difference in robust loss for all $\z=(\x,y) \in \cZ$ can be bounded as,
\begin{align*}
	\big\lvert \ell_\nu(\hat{h},\vc{z}) - \ell_\nu({h},\vc{z}) \big\rvert
	& = \big\lvert \max_{\bdel}~  \ell\big(\hat{h}, (\x+\bdel,y)\big) - \max_{\bdel}~  \ell \big({h}, (\x+\bdel,y)\big)\big\rvert \\
	& \leq \max_{\bdel}~ \big\lvert \ell\big(\hat{h}, (\x+\bdel,y)\big) - \ell\big({h}, (\x+\bdel,y)\big) \big\rvert \\
	& \leq \Lloss \cdot \max_{\bdel} \norm{\hat{h}(\x+\bdel)-h(\x + \bdel)}_{2}\\
	& \leq \Lloss \cdot \Lparnu.
	\end{align*}
Using the definition of robust sparse local Lipschitzness, we can provide a tighter bound on the difference in robust loss for nearby predictors, 
\[
\norm{\hat{h}-h}_\cH \leq \rparnu(\x,s) \implies \max_{\bdel} \norm{\hat{h}(\x+\bdel)- h(\x+\bdel)}_2 \leq \lparnu(\x, s) \norm{\hat{h}-h}_\cH.
\]
The conclusion now follows using the same sequence of steps as above.
\end{proof}

\subsection{Uniform Generalization bound for Lipschitz Predictors} \label{app: thm: unif-gen-bound}
\begin{theorem}
\label{thm: unif-gen-bound}
For a class of Lipschitz predictors $\cH$, 
with probability at least $(1-\alpha)$ over the choice of i.i.d sample $\samp_T$, 
the population risk of any predictor $h \in \cH$ is bounded as, 
\begin{align*}
 R(h) 
 \leq \hat{R}(h)  
 + b \sqrt{\frac{\ln \left(\mathcal{N}(\frac{1}{m},\cH)\right) + \ln(\frac{2}{\alpha})}{2m}}
  + \frac{2 \mathsf{L}_{\mathrm{loss}} \mathsf{L}_{\mathrm{par}}}{m}. 
\end{align*}\end{theorem}
\begin{proof}
For some $\epsilon > 0$, we construct proper epsilon-covers of the hypothesis class $\cH$ w.r.t. the $\norm{\cdot}_\cH$ norm. Let $\mathcal{N}(\epsilon,\cH)$ denote the size of a minimal proper epsilon cover $\{h^j\}_{j=1}^{\mathcal{N}(\epsilon,\cH)}$. Let $\hat{h}$ be the cover element closest to the predictor $h$. 
We can decompose the difference between robust stochastic risk and robust empirical risk, 
\begin{align*}
    |R(h) - \hat{R}(h)|
    &= 
    \Big\lvert\expect_{\z \sim \cD_\cZ} 
    \Big[ \ell(h,\z) \Big]
    - \frac{1}{m}\sum_{i=1}^m \ell\Big(h,\z_i\Big)\Big\rvert \\
    & \leq  
    \Big\lvert\expect_{\z \sim \cD_\cZ} 
    \Big[ \ell(h,\z) \Big] 
    - \expect_{\z \sim \cD_\cZ} \Big[ \ell(\hat{h},\z) \Big] \Big\rvert \qquad \text{Term 1}\\
    &\quad 
    + \sup_{j \in [\mathcal{N}(\epsilon,\cH)]} \Big \lvert \expect_{\z \sim \cD_\cZ} \Big[ \ell(h^j,\z) \Big] -  \frac{1}{m}\sum_{i=1}^m \ell\Big(h^j,\z_i\Big) \Big\rvert 
    \qquad {\text{Term 2}}\\
    & \quad 
    + \Big \lvert \frac{1}{m}\sum_{i=1}^m \ell\Big(\hat{h},\z_i\Big) - \frac{1}{m}\sum_{i=1}^m \ell\Big(h,\z_i\Big)\Big\rvert. \qquad \text{Term 3}
\end{align*}

Term 2 can be bounded using standard concentration of measure with union bound and Hoeffding's inequality, 
\begin{align*}
    & \text{Pr} \Big( \sup_{j \in [\mathcal{N}(\epsilon,\cH)]} \Big \lvert \expect_{\z \sim \cD_\cZ} \Big[ \ell(h^j,\z) \Big] -  \frac{1}{m}\sum_{i=1}^m \ell\Big(h^j,\z_i\Big) \Big\rvert \geq \xi \Big) \\ 
    &\leq \sum_{t \in [\mathcal{N}(\epsilon,\cH)]} \text{Pr} \Big( \Big \lvert \expect_{\z\sim \cD_\cZ} \Big[ \ell(h^j, \z) \Big] -  
    \frac{1}{m}\sum_{i=1}^m \ell\Big(h^j,\z_i\Big) \Big\rvert \geq \xi \Big) \\ 
    & \leq \sum_{t \in [\mathcal{N}(\epsilon,\cH)]} 2 \exp{\left(\frac{-2m^2\xi^2}{\sum_{i\in [m]} b^2}\right)} 
    = 2 \mathcal{N}(\epsilon,\cH) \exp{\left(\frac{-2m\xi^2}{b^2}\right)} .
\end{align*}
Hence w.p. $(1-\alpha)$, 
\[
\text{Term 2 } \leq b \sqrt{\frac{\ln\left(\mathcal{N}(\epsilon,\cH)\right) + \ln(\frac{2}{\alpha})}{2m}} .
\] 
Term 1 and Term 3 both have a difference of robust loss and thus can be bounded using the Lipschitz property of the representation and \Cref{lemma: global-local-lip}. 
Both Term 1 and Term 3 are upper bounded as $\Lloss \Lpar \epsilon$. The conclusion follows by setting $\epsilon=\frac{1}{m}$.
\end{proof}

\subsection{Proof for \Cref{thm: generalization-sll}} 
\label{app: thm: generalization-sll}
To bound the population risk $R(h)$, we first construct proper epsilon-covers of the parameter space.
We choose the cover resolution $\epsilon = \frac{1}{m}$. 
Let $\mathfrak{C}$ be a minimal proper cover w.r.t. the norm $\norm{\cdot}_\cH$,
\[ \displaystyle
\mathfrak{C} := \left\{\A^j,\W^j\right\}_{j=1}^{\mathcal{N}(\epsilon,\cH)} \subseteq (\mathcal{A} \times \cW )
\]
We let $\cH_{\mathfrak{C}}$ denote predictors with weights in the proper cover $\mathfrak{C}$.
For the predictor $h \in \cH$, we denote by $\hat{h} \in \cH_{\mathfrak{C}}$ the predictor with parameters from the closest cover element so that $\norm{\hat{h}-h}_{\cH} \leq \frac{1}{m}$.
We can then decompose the population risk as
\begin{align*}
    R(h) &= \underbrace{R(h) - R(\hat{h})}_{{\color{red} I}} + \underbrace{R(\hat{h})-\hat{R}(\hat{h})}_{{\color{red} II}} + \underbrace{\hat{R}(\hat{h}) - \hat{R}(h)}_{{\color{red} III}} + \hat{R}(h)
\end{align*}
The first term above is a measure of expected sensitivity of the predictor $h$ to perturbation in the parameter. 
The second term is the generalization gap of the cover predictor $\hat{h}$. 
The third term above is a measure of empirical sensitivity of the predictor $h$ to perturbation in the parameter. 
The last term on the above decomposition is the empirical risk evaluated on the training sample $\samp_T$. 
Let $\rpar(\cdot,\cdot)$ and $\lpar(\cdot,\cdot)$ be the local radius and local Lipschitz scale functions corresponding to the predictor $h$. 

Using the sparse regularity of the predictor $h$ w.r.t. the set $\samp_T \cup \samp_U$, we can deterministically bound term {\color{red} III}. By definition for any $\z=(\x,y)\in \samp_T$, 
\begin{align*}
|\ell(\hat{h}, \z) - \ell(h, \z)| 
&\leq \Lloss \cdot \norm{\hat{h}(\x)-h(\x)}_2 \\   
&\leq \Lloss \cdot \mathcal{L}\left(h, \samp_T ,\epsilon \right) \norm{\hat{h}-h}_\cH \\
&\leq \Lloss \cdot \mathcal{L}\left(h, \samp_T \cup \samp_U ,\epsilon \right) \norm{\hat{h}-h}_\cH. 
\end{align*}
Hence, the difference in empirical risk evaluated w.r.t. $\samp_T$, 
\begin{align*}
\hat{R}(\hat{h}) - \hat{R}(h)
& \leq \frac{\Lloss \cdot \mathcal{L}\left(h, \samp_T \cup \samp_U, 1/m \right)}{m}.
\end{align*}

We bound {\color{red} II} with high probability on choice of training sample $\samp_T$ using standard concentration on the generalization gap in the entire class of cover predictors $\cH_{\mathfrak{C}}$,
\begin{align*}
R(\hat{h})-\hat{R}(\hat{h}) 
&= \expect_{\z \sim \cD_\cZ} [\ell (\hat{h}, \z)] - \frac{1}{m} \sum_{i=1}^m \ell (\hat{h},\z_i)\\
&\leq \sup_{h_j \in \cH_{\mathfrak{C}}} \Big \lvert \expect_{\z\sim \cD} \Big[ \ell(h_j,\z) \Big] -  \frac{1}{m}\sum_{i=1}^m \ell\Big(h_j,\z_i\Big) \Big\rvert\\
&\leq b \sqrt{\frac{ \ln(|\cH_{\mathfrak{C}}|) + \ln(\frac{2}{\alpha})  }{2m}}, \qquad \text{w.p. }(1-\frac{\alpha}{2}) \text{ over } \samp_T \sim (\cD_\cZ)^m.
\end{align*}
The last inequality above follows from Hoeffding's Inequality. 
Here $|\cH_{\mathfrak{C}}| = |\mathfrak{C}| = \mathcal{N}(\frac{1}{m},\cH)$, the size of the minimal proper epsilon-cover.

It is now left to bound term {\color{red} I}. 
For any sparsity level $s$, we define \textsc{Bad}$(h,(\x,y),s)$ to be the event that the local radius at $\x$ with level $s$ is below the threshold $\epsilon=\frac{1}{m}$ i.e. $\rpar(\x, s) < \frac{1}{m}$.
The event $\textsc{Bad}(h,\z,s)$ indicates that the predictor $h$ cannot be guaranteed to preserve inactive indices in the representation outputs upon perturbation to parameter. 
Note that the probability of the event $\textsc{Bad}(h, \z, 0)$ is zero for Lipschitz functions.
We denote by $\textsc{Good}(h, \z, s)$ the complement of the bad event where either the sparsity level is trivial or the local radius at input $\x$ is sufficiently large. 
If $\textsc{Good}(h, \z, s)$ holds, then
\begin{align*}
|\ell(\hat{h}, \z) - \ell(h, \z)| 
& \leq \Lloss \norm{\hat{h}(\x) - h(\x)}_2 \leq \Lloss \cdot \lpar(\x, s) \cdot \norm{\hat{h}-h}_\cH \\
&\leq \frac{\Lloss \cdot \lpar(\x, s) }{m}.
\end{align*}
If $\textsc{Bad}(h, \z, {s})$ holds then we simply bound the difference in losses by
\[
\Big\lvert
\ell(h,\vc{z}) 
- \ell(\hat{h},\vc{z}) \Big\rvert 
\leq b.
\]
To bound term {\color{red} I}, observe that
\begin{align*}
    R(h) - R(\hat{h})  
    &= \expect_{\z \sim \cD_\cZ} [ \ell (h, \z) - \ell (\hat{h}, \z) ]\\
    &= \expect_{\z \sim \cD_\cZ}\Big[ \ell (h, \z) - \ell (\hat{h}, \z) ~|~ \textsc{Good}(h, \z, {s})\Big] \text{Pr}_{\z \sim \cD_\cZ}\left(\textsc{Good}(h, \z, {s})\right)\\
    & \quad + \expect_{\z \sim \cD_\cZ}\Big[ \ell (h, \z) - \ell (\hat{h}, \z) ~|~ \textsc{Bad}(h, \z,{s})\Big] \text{Pr}_{\z \sim \cD_\cZ}\left(\textsc{Bad}(h, \z, {s})\right) \\
    & \leq \frac{\Lloss}{m} \cdot \max_{\x \in \cX}~ \lpar(\x, s) +  b \cdot  \text{Pr}_{\z \sim \cD_\cZ}\left(\textsc{Bad}(h, \z, {s})\right).
\end{align*}
For ease of notation we denote $\kappa(s) := \frac{\Lloss}{m} \cdot \max_{\x \in \cX}~ \lpar(\x, s)$. The above bound simultaneously holds for any choice of sparsity level $0 \leq s \leq p$. Hence, 
\begin{align*}
     R(h) - R(\hat{h})  
    & \leq   
    \min_{{s} } \Big \{ \kappa(s) +  b \cdot  \text{Pr}_{\z \sim \cD_\cZ}\left(\textsc{Bad}(h, \z, {s})\right)\Big\}.
\end{align*}
We can use the unlabeled data $\samp_U$ to estimate the probability of the bad event at randomly sampled $\z$. For the predictor $h$ at any fixed sparsity level $0 \leq s \leq \min_{\x\in\samp_U} p-\|\Phi(\x)\|_0$, the local radius $\rpar(\x, s)$ is a random variable. 
Since the predictor $h$ is trained only on $\samp_T$, 
at each $\x_i \in \samp_U$, $\rpar(\x_i; {s})$ is an i.i.d observation.
Consider the C.D.F $F(t)$ and $F_{U}(t)$, an unbiased estimate from $\samp_U$, 
\[
F(t; {s}):= \cD_\cZ(\rpar(\x ; {s}) \leq t), \quad F_U(t; {s}) := \frac{1}{m} \cdot \sum_{\x_i \in \samp_U} \mathbb{1}\{\rpar(\x_i, s) \leq t\}.
\]
By uniform convergence \cite{wainwright_2019}, with probability $(1-\alpha')$ over the unlabeled data $\samp_U \sim (\cD_\cX)^{m}$,
\[\sup_{t} |F_U(t; \vc{s}) - F(t; \vc{s})| \leq  8\sqrt{\frac{ 2\ln(m) }{m}} + 2\sqrt{\frac{\ln(\frac{1}{\alpha'})}{m}}
\]
Note that w.p. $(1-\alpha')$ over the unlabeled data, 
\begin{align*}
\text{Pr}_{\z \sim \cD_\cZ}\left(\textsc{Bad}(h,\z,{s})\right) = F\left(\frac{1}{m}; {s}\right) \leq F_U\left(\frac{1}{m}; {s}\right) + 8\sqrt{\frac{ 2\ln(m) }{m}} + 2\sqrt{\frac{\ln(\frac{1}{\alpha'})}{m}}
\end{align*}
or equivalently,
\begin{align*}
\text{Pr}_{\samp_U \sim (\cD_\cZ)^{m}} 
\Bigg[
\text{Pr}_{\z \sim \cD_\cZ}\left(\textsc{Bad}(h,\z,{s})\right) > F_U(\nu; {s}) + 8\sqrt{\frac{ 2 \ln(m) }{m}} + 2\sqrt{\frac{\ln(\frac{1}{\alpha'})}{m}}
\Bigg] \leq \alpha'.
\end{align*}
We let $\alpha' = \frac{\alpha}{2p}$. To complete the bound on term {\color{red} I} note that, 
\begin{align*}
    &\underset{\samp_U \sim (\cD_\cZ)^{m}}{\text{Pr}} 
    \Bigg[
     \min_{1 \leq {s} \leq p} \Bigg(\kappa({s}) + b\cdot \text{Pr}_{\z \sim \cD_\cZ}\left(\textsc{Bad}(h,\z,\vc{s})\right)\Bigg)
    > \\
    & \qquad \qquad 
    \min_{1 \leq {s} \leq p} \Bigg(\kappa({s}) + b\cdot \Bigg[ F_U\left(\frac{1}{m}; {s}\right) +  8\sqrt{\frac{ 2\ln(m) }{m}} + 2\sqrt{\frac{\ln(\frac{1}{\alpha'})}{m}} \; \Bigg]\Bigg)
    \Bigg]\\
    &= \underset{\samp_U \sim (\cD_\cZ)^{m}}{\text{Pr}} 
    \Bigg[ 
     \exists~ \bar{s}~ \;\;\text{Pr}_{\z \sim \cD_\cZ}\left(\textsc{Bad}(h,\z,\bar{s})\right) > F_U\left(\frac{1}{m}; \bar{s}\right) +  8\sqrt{\frac{ 2\ln(m) }{m}} + 2\sqrt{\frac{\ln(\frac{1}{\alpha'})}{m}}
    \Bigg] \\
    &\leq \sum_{s = 1}^p \underset{\samp_U \sim (\cD_\cZ)^{m}}{\text{Pr}}   \Bigg[ 
     \text{Pr}_{\z \sim \cD_\cZ}\left(\textsc{Bad}(h,\z,{s})\right) > F_U\left(\frac{1}{m}; {s}\right) +  8\sqrt{\frac{ 2\ln(m) }{m}} + 2\sqrt{\frac{\ln(\frac{1}{\alpha'})}{m}}
    \Bigg]\\
    &\leq \sum_{s = 1}^p \alpha' \leq \frac{\alpha}{2}.
\end{align*}
Hence w.p. at least $(1-\frac{\alpha}{2})$, 
\begin{align*}
    & R(h) - R(\hat{h})  \\
    & \leq \min_{0 \leq {s} \leq p} \Bigg\{ \frac{\Lloss}{m} \cdot \max_{\x \in \cX}~ \lpar(\x, s) + b\cdot \left( F_U\left(\frac{1}{m}; {s}\right) +  8\sqrt{\frac{ 2\ln(m) }{m}} + 2\sqrt{\frac{\ln(\frac{1}{\alpha'})}{m}} \; \right)\Bigg\} 
\end{align*}
We can now upper bound the above RHS with any sparsity level ${s}$ in $0 \leq s \leq p$, possibly based on the unlabeled data $\samp_U$ and even the training data $\samp_T$. 
Specifically, we can choose the sparsity level ${s}^*\left(\samp_T \cup \samp_U, \frac{1}{m}\right)$ corresponding to the optimal sparsity level when the parameter radius is bounded by $\frac{1}{m}$. 
By definition at this sparsity level, 
\[
F_U\left(\frac{1}{m}; {s}^*\left(\samp_T \cup \samp_U, \frac{1}{m}\right)\right) = 0.
\]
Further, by definition of the sparse regularity of the predictor $h$, 
\[
\max_{\x \in \cX}~ \lpar\left(\x, {s}^*\left(\samp_T \cup \samp_U, \frac{1}{m}\right) \right)
\leq \mathcal{L}(h, \samp_T \cup \samp_U,1/m).
\]
Thus the bound can be rewritten as, 
\begin{align*}
    R(h) - R(\hat{h}) 
    & \leq 
    \frac{2 \Lloss \cdot \mathcal{L}(h, \samp_T \cup \samp_U,1/m)}{m}
    + 8b\sqrt{\frac{ 2 \ln(m) }{m}} + 2b\sqrt{\frac{ \ln(p) + \ln(\frac{2}{\alpha})}{m}}
\end{align*}
Combining the bound on all three terms and from the observation that $mp \leq \mathcal{N}(\frac{1}{m}, \cH)$ (due to \Cref{lemma:cov2}, we have our final result. 

The above result can also be improved to have the bound depend on the following, 
\begin{equation} \label{eq:tighterspcomp}
\mathfrak{L}_{\mathrm{sparse}}(h, \samp_U)
+ \max_{(\x,y) \in \samp_T} \lpar\left(\x, s^*\left(\{\x\},\frac{1}{m}\right)\right).
\end{equation}
Here, the local Lipschitz scale at each point in the training set is chosen at the sparsity level optimal for the specific point (as opposed to a set), which  is necessarily tighter than $\mathcal{L}(h, \samp_T \cup \samp_U)$. 
\Cref{eq:tighterspcomp} reflects the different roles of the training and unlabeled data sets. The learning algorithm chooses a predictor based on $\samp_T$. The unlabeled data $\samp_U$ is used as a reference set in bounding the local Lipschitz scale at the good event of sufficiently large local radius using the optimal sparse level ${s}^*(\cV,\frac{1}{m})$. 
For ease of presentation, we define the sparse regularity in terms of both the training and unlabeled data. 

We now state an extension of \Cref{thm: generalization-sll} to the robust setting.
\begin{theorem}\label{app: thm: robust-generaalization-sll}
With probability at least $(1-\alpha)$ over the choice of i.i.d training sample $\samp_T$ and unlabeled data $\samp_U$ each of size $m$, 
for any predictor $h \in \cH$ with parameters $(\A,\W)$,
the generalization error is bounded by
\begin{align*}
R_{\mathrm{rob}}\left(h\right) \leq  
\hat{R}_{\mathrm{rob}}\left(h\right)
+ \mathcal{O}\left(
b \sqrt{\frac{\ln\left(\mathcal{N}(\frac{1}{m},\cH)\right) + \ln(\frac{2}{\alpha})}{2m}} 
+  \frac{\Lloss \cdot \mathcal{L}_{\mathrm{rob}}(h, \samp_T \cup \samp_U,\frac 1{2m})}{m}\right).
\end{align*}
\end{theorem}
\begin{proof}
   The proof follows the same sequence of steps as above. The only difference comes from using \Cref{lemma: robust-global-local-lip} to bound the difference in robust losses and defining the robust sparse regularity $\mathcal{L}_{\mathrm{rob}}(h, \samp_T\cup\samp_U, \frac{1}{2m})$ using the robust sparse local radius $\rparnu$ and the robust local Lipschitz scale $\lparnu$.
\end{proof}

\subsection{Size of Feedforward network representations}
We first demonstrate a lemma that bounds the size of any layer representation, \begin{lemma}
\label{lemma: sizeofmnn}
Consider any input $\x \in \mathbb{R}^{d^0}$ and layer matrices $\W^k \in \cW^k$. 
The size of each layer of the representation for $1\leq k \leq K$ is bounded by
$\norm{\Phi^{[k]}(\x)}_2 \leq \zeta^k(\vc{0}) \norm{\x}_2$.

Further if the set of inactive columns in layer $k$, $\mathcal{I}^k(\x)$ has size of at least $s^k$, then the size of each layer of the representation is bounded by
$\norm{\Phi^{[k]}(\x)}_2 \leq \zeta^k(\vc{s}) \norm{\x}_2$. 
\end{lemma}
\begin{proof}
For $k=1$, we can bound the size of the representation at the first layer as,
\begin{align*}
\norm{\Phi^1(\x)}_2 
= \norm{\sigma\left(\W^1 \x \right)}_2 
\leq \norm{\W^1}_2 \norm{\x}_2 \leq \zeta^1(\vc{0})\norm{x}_2
\end{align*}
Thus the base case $k=1$ holds. Suppose the statement holds for some $1\leq n \leq k-1$ so that $\norm{\Phi^{[k-1]}(\x)}_2 \leq \zeta^{k-1}(\vc{0}) \norm{\x}_2$. 
Then,
\begin{align*}
\norm{\Phi^{[k]}(\x)}_2 
&= \norm{\sigma\left(\W^{k} \Phi^{[k-1]}(\x) \right)}_2 \\
&\leq \norm{\W^{k}}_2 \norm{\Phi^{[k-1]}(\x)}_2 \\
&\leq \sqrt{1+\mathsf{M}^k_{0}} \cdot \mathsf{M}^k_{\cW}\cdot \zeta^{k-1}(\vc{0}) \norm{\x}_2 \\
&= \zeta^{k}(\vc{0})\norm{\x}_2.
\end{align*}
Hence the first statement holds by induction.

To prove the second statement, we let $s^0 = 0$ and $J^0 = [d^0]$.
For $k=1$, by assumption there exists an inactive index set $I^1 \subseteq \mathcal{I}^1(\x)$ of size $s^1$. 
Let $J^1 = (I^1)^C \supseteq \mathcal{J}^1(\x)$, we can now bound the size of the representation at the first layer as,
\begin{align*}
\norm{\Phi^1(\x)}_2 
&= \norm{\sigma\left(\W^1 \x \right)}_2 \\
&= \norm{\sigma\left(\mathcal{P}_{J^1, J^0}(\W^1) \;\x \right)}_2 \\
&\leq \norm{\mathcal{P}_{J^1,J^0}(\W^1)}_2 \norm{\x}_2 \leq \zeta^1(\vc{s})\norm{x}_2
\end{align*}
Thus the base case $k=1$ holds. Suppose the statement holds for some $0 < k < K$ so that $\norm{\Phi^{[k]}(\x)}_2 \leq \zeta^k(\vc{s}) \norm{\x}_2$. Let $J^{k+1} \supseteq \mathcal{J}^{k+1}(\x)$ be an active set of $\Phi^{[k+1]}(\x)$ of size $(d^{k+1}-s^{k+1})$.
Then,
\begin{align*}
    \norm{\Phi^{[k+1]}(\x)}_2
    &= \norm{\act{\W^{k+1} \Phi^{[k]}(\x)}}_2\\
    &= \norm{\act{\mathcal{P}_{J^{k+1},J^k}(\W^{k+1}) \; \mathcal{P}_{J^{k}}(\Phi^{[k]}(\x))}}_2\\
    &\leq \norm{\mathcal{P}_{J^{k+1},J^k}\big(\W^{k+1}\big)}_2 \zeta^k(\vc{s}) \norm{\x}_2
    \leq \zeta^{k+1}(\vc{s}) \norm{\x}_2.
\end{align*}   
Hence the second statement holds by induction.
\end{proof}
As a consequence of \Cref{lemma: sizeofmnn}, for predictors $h, \hat{h}$, input $\x \in \cX$ and perturbation $\bdel \in \cB^{\cX}_{\nu}(\mathbf{0})$ under appropriate conditions, 
\[
\norm{\Phi^{[k]}(\x)} \leq \zeta^k(\vc{s}), \quad \norm{\Phi^{[k]}(\x+\bdel)} \leq \zeta^k(\vc{s}) (1+\nu), \quad \norm{\hat{\Phi}^{[k]}(\x+\bdel)} \leq \zeta^k(\vc{s}) (1+\nu).
\]
In particular, for each input $\x$, the tightest bound is $\norm{\Phi^{[k]}(\x)} \leq \zeta^k(\bar{\vc{s}}(\x))\norm{\x}_2$, using the size of full inactive index sets.

\subsection{Proof for \cref{lemma: robust-global-lip-fnn}}\label{app: lemma: robust-global-lip-fnn}
For reference the following is the definition of the robust global Lipschitz constant, 
\[
\forall h, \hat{h}, \x, \quad \max_{\bdel} \norm{\hat{h}(\x+\bdel)-h(\x+\bdel)}_2 \leq \Lparnu \norm{\hat{h}-h}_\cH.
\]

Consider two predictors $h, \hat{h} \in {\cH^{K+1}}$ with weights $\{\W^k\}$ and $\{\What^k\}$ respectively. 
We first note that for $1\leq k \leq K$, for any layer weight matrix $\W^k \in \cW^k$,  
\[\norm{\W^k}_2 \leq \norm{\W^k}_F \leq \sqrt{1 + \mathsf{M}^k_{0}}\norm{\W^k}_{2,\infty} \leq \sqrt{1+\mathsf{M}^k_{0}} \cdot \mathsf{M}^k_{\cW}.\]
and further,
\[
\norm{\What^k-\W^k}_2 \leq
\sqrt{d^k}\norm{\What^k-\W^k}_{2,\infty} 
= \mathsf{M}^k_{\cW} \norm{\What^k-\W^k}_{\cW^k} \leq \mathsf{M}^k_{\cW} \norm{\hat{h}-h}_{\cH^{K+1}}.
\]
Similar inequalities hold for $K+1$. 
To prove \Cref{lemma: robust-global-lip-fnn}, we first claim that at any layer $k \leq K$, the distance between the perturbed representations are bounded as,
\[
\norm{\hat{\Phi}^{[k]}(\x+\bdel) - \Phi^{[k]}(\x+\bdel)} \leq k \zeta^{k}(\vc{0}) \cdot (1+\nu) \cdot \norm{\hat{h}-h}_{\cH^{K+1}}
\]
We shall prove this claim by induction. 

\textbf{Base Case ($k=1$):}
To bound the distance between the representations, we note that
\begin{align*}
\norm{\hat{\Phi}^{[1]}(\x+\bdel) - \Phi^{[1]}(\x+\bdel)}_2 
& \leq \norm{(\What^1 - \W^1) (\x+\bdel)}_2 \\
& \leq \norm{(\What^1-\W^1)}_2 \norm{\x+\bdel}_2\\
& \leq \sqrt{d^1} \cdot \norm{\What^1 - \W^1}_{2, \infty} (1+\nu) \\
& \leq \norm{\hat{h}-h}_\cH \cdot \mathsf{M}^1_{\cW} (1+\nu) 
<  \zeta^1(\vc{0})\cdot (1+\nu) \cdot\norm{\hat{h}-h}_{\cH^{K+1}}.
\end{align*}
Thus we have bounded the distance between the perturbed representations $\Phi^{[1]}(\x+\bdel)$ and $\hat{\Phi}^{[1]}(\x+\bdel)$. 

\textbf{Induction case ($k>2$):}

Assume that the bound holds for $1 \leq n \leq k-1$, then for layer $k$,
\begin{align*}
&\norm{\hat{\Phi}^{[k]}(\x+\bdel) - \Phi^{[k]}(\x+\bdel)}_2 \\
&\leq \norm{\What^k \hat{\Phi}^{[k-1]}(\x+\bdel) - \W^k \Phi^{[k-1]}(\x+\bdel)}_2 \\
& \leq \norm{\What^k \hat{\Phi}^{[k-1]}(\x+\bdel) - \What^k \Phi^{[k-1]}(\x+\bdel)}_2 
+ \norm{\What^k \Phi^{[k-1]}(\x+\bdel) - \W^k \Phi^{[k-1]}(\x+\bdel)}_2 \\
&\leq  \norm{\What^k}_2 \norm{\hat{\Phi}^{[k-1]}(\x+\bdel)-\Phi^{[k-1]}(\x+\bdel)}_2
+ \norm{\What^k-\W^k}_2 \norm{\Phi^{[k-1]}(\x+\bdel)}_2\\
& \leq \sqrt{1+\mathsf{M}^k_{0}} \cdot \mathsf{M}^k_{\cW} \cdot (k-1)\zeta^{k-1}(\vc{0})\cdot (1+\nu) \cdot \norm{\hat{h}-h}_{\cH^{K+1}}
\\
& \qquad
+ \zeta^{k-1}(\vc{0})\cdot (1+\nu) \cdot \mathsf{M}^k_{\cW} \norm{\hat{h}-h}_{\cH^{K+1}}
\\
& < k \zeta^{k}(\vc{0}) \cdot (1+\nu) \cdot \norm{\hat{h}-h}_{\cH^{K+1}}.
\end{align*}
Thus we have bounded the distance between the perturbed representations $\Phi^{[k]}(\x+\bdel)$ and $\hat{\Phi}^{[k]}(\x+\bdel)$. 
Hence proved by induction. 
The conclusion for \Cref{lemma: robust-global-lip-fnn} follows from noting that, 
\begin{align*}
&\norm{\hat{h}(\x+\bdel) - h(\x+\bdel)}_2 \\
&\leq \norm{(\What^{K+1})^T \hat{\Phi}^{[K]}(\x+\bdel) - (\W^{K+1})^T \Phi^{[K]}(\x+\bdel)}_2 \\
& \leq \norm{(\What^{K+1})^T \hat{\Phi}^{[K]}(\x+\bdel) - (\What^{K+1})^T \Phi^{[K]}(\x+\bdel)}_2 
+ \norm{(\What^{K+1})^T \Phi^{[K]}(\x+\bdel) - (\W^{K+1})^T \Phi^{[K]}(\x+\bdel)}_2 \\
&\leq  \norm{(\What^{K+1})^T}_2 \norm{\hat{\Phi}^{[K]}(\x+\bdel)-\Phi^{[K]}(\x+\bdel)}_2
+ \norm{(\What^{K+1}-\W^{K+1})^T}_2 \norm{\Phi^{[K]}(\x+\bdel)}_2\\
& \leq \sqrt{1+\mathsf{M}^{K+1}_{0}} \cdot \mathsf{M}^{K+1}_{\cW} \cdot K\zeta^{K}(\vc{0})\cdot (1+\nu) \cdot \norm{\hat{h}-h}_{\cH^{K+1}}
\\
& \qquad
+ \zeta^{K}(\vc{0})\cdot (1+\nu) \cdot \mathsf{M}^{K+1}_{\cW} \norm{\hat{h}-h}_{\cH^{K+1}}
\\
& < (K+1) \zeta^{K+1}(\vc{0}) \cdot (1+\nu) \cdot \norm{\hat{h}-h}_{\cH^{K+1}}.
\end{align*}

\subsection{Proof for \Cref{lemma: robust-local-lip-fnn}}
\label{app: fix1}
For a feed forward network, we first explicitly define the layer-wise angular distance $\beta^k(\x, \vc{s}) \in [0,1]^{d^k}$ at each input $\x$ and a candidate sparsity level $\vc{s}$,
\[
[\beta^k(\x)]_i := [\beta(\W^k, \Phi^{[k-1]}(\x))]_i = \frac{1}{\pi} \cdot \arccos \Bigg( \frac{\langle \Wvec{k}{i},\Phi^{[k-1]}(\x)\rangle}{\mathsf{M}^k_{\cW} \zeta^{k-1}(\vc{0})}  \Bigg), \quad \forall~ i \in [d^k].
\]
Here the $\arccos$ is well defined since, $\norm{\Phi^{[k-1]}(\x)}_2 \leq \zeta^{k-1}(\vc{0})\norm{\x}_2 \leq \zeta^{k-1}(\vc{0})$ from \Cref{lemma: sizeofmnn}.
The layer-wise critical angular distance $\theta^k(\x,s) \in [0,1]$ is then defined as $(d^k-s^k)$-th smallest entry in $\beta^k(\x,\vc{s})$, 
\begin{equation} \label{eq: layer-wise-critical-angle-distance}
\theta^k(\x, \vc{s}) := 
\textsc{sort} \left( \beta^k(\x), s^k \right).
\end{equation}

For the case $\vc{s}=0$, by definition, the radius $\rparnu(\x,\vc{0})=\infty$ and $\lparnu(\x,\vc{0})=\Lparnu$ for all inputs $\x \in \cX$. 
If, $\vc{s}$ is too large such that, $-\cos(\pi\theta^k(\x,\vc{s}) - \nu \leq 0$, then $\rparnu(\x, \vc{s}) = 0$. 
Below we shall prove a stronger version of the theorem statement \Cref{lemma: robust-local-lip-fnn-upto-k} for when $\vc{s}$ is simultaneously non-trivial so that $\vc{s} \succ 0$ and not too large so that $\rparnu(\x,\vc{s}) > 0$. 

We first start by showing a weaker result (compared to \Cref{lemma: fnn-layer-lip}) on sparse local Lipschitzness of feed-forward network predictors w.r.t. inputs in terms of the critical angular distance.
\begin{lemma}\label{app: lemma: fnn-layer-lip-weaker}
\begin{align*}\rinp^{(k)}(\vc{t}, \vc{s}) &:= -\cos\left(\pi \cdot \theta^k(\vt, \vc{s})\right)\cdot \norm{t}_2.\\
\linp^{(k)}(\vc{t}, \vc{s}) &:= \sqrt{1+\mathsf{M}^k_{s^k}} \cdot \mathsf{M}^k_{\cW}.
\end{align*}
Thus, informally 
$\norm{\bdel}_2 \leq \rinp^{[K]}(\x, \vc{s}) \implies 
\norm{\Phi^{[K]}(\x+\bdel)-\Phi^{[K]}(\x)}_2 \leq \zeta^{K}(\vc{s}) \norm{\bdel}_2.$
\end{lemma}
\begin{proof}
   The proof follows from the same sequence of steps as in \Cref{lemma: fnn-layer-lip} while noting that $\sup_{J^{k-1}}\norm{\mathcal{P}_{J^k,J^{k-1}}(\W^k)}_2 \leq \sqrt{1+\mathsf{M}^k_{s^k}} \cdot \mathsf{M}^k_{\cW}$.
\end{proof}

\begin{lemma}\label{lemma: robust-local-lip-fnn-upto-k}
A feed forward neural network $h \in \cH^{K+1}$ is such that for all $1\leq k\leq K+1$, the cumulative representation map $\Phi^{[k]}$ is robust sparse local Lipschitz w.r.t. parameters, 
with the following robust local radius and robust local Lipschitz scale, 
\begin{align*}
\rparnu^{[k]}(\x, \vc{s}) &:= \min_{1\leq n \leq k} \frac{ \iota(\vc{s}) + \max\left \{0, -\cos(\pi\theta^n(\x,\vc{s}) - \nu\right\}}{n(1+\nu)}. \\
\lparnu^{[k]}(\x, \vc{s}) &:= k \zeta^{k}(\vc{s}) \cdot (1+\nu).
\end{align*}
Note that $\rparnu^{[k]}(\x, \vc{s}) \leq \rparnu^{[n]}(\x, \vc{s})$ for all $1\leq n\leq k$. 
Thus\footnote{Note : it is sufficient to ensure that the distance between parameters up to $k$ layers is below the local radius but for ease of presentation we stick with the norm and induced metric in $\cH^{K+1}$.},
for any hypothesis $h$, for $1\leq n \leq k$, there exists an inactive set $I^n$, such that for all corruptions $\bdel$ within the input radius $\nu$ and perturbed hypothesis $\hat{h} \in \cH^{K+1}$ within the parameter radius $\rparnu^{[k]}(\x, \vc{s})$, the following holds true for $1\leq n\leq k$,
\begin{align*}
(1) &~\mathcal{P}_{I^n}(\hat{\Phi}^{[n]}(\x+\bdel)) = \mathcal{P}_{I^n}(\Phi^{[n]}(\x+\bdel)) = \mathcal{P}_{I^n}(\Phi^{[n]}(\x)) = \mathbf{0} ,\\
(2) &~\norm{\hat{\Phi}^{[n]}(\x+\bdel) - \Phi^{[n]}(\x+\bdel)}_2 \leq \lparnu^{[n]}(\x,\vc{s}) \cdot \norm{\hat{h}-h}_{\cH^{K+1}}. 
\end{align*}
\end{lemma}
\begin{proof}
The proof will follow by induction.

\textbf{Base case $(k=1)$}:
We assume here that $\norm{\hat{h}-h}_{\cH^{K+1}} \leq \rparnu^{[1]}(\x, \vc{s})$.
The critical angular distance corresponding to layer $1$ is sufficiently large so that,
\[
\nu + (1+\nu) \norm{\hat{h}-h}_{\cH^{K+1}} \leq -\cos \left(\pi \theta^1(\x, \vc{s})\right).
\]
Note that $-\cos(\pi \theta^1(\x, \vc{s})) > 0 \implies \theta^1(\x, \vc{s}) > \frac{1}{2}$ and hence there exists an index set $I^1 \subset [d^1]$ of size $s^1$ that is inactive in $\Phi^{[1]}(\x)$ such that for each $i \in I^1$, 
\[
\Wvec{1}{i}\x \leq \mathsf{M}^1_{\cW} \cdot 1 \cdot \cos \left(\pi \theta^1(\x, \vc{s})\right) < 0.
\] We can show that $I^1$ is also inactive for $\Phi^{[1]}(\x+\bdel)$ and $\hat{\Phi}^{[1]}(\x+\bdel)$.
To see this for each $i\in I^1$, 
\begin{align*}
     \Wvec{1}{i}(\x+\bdel) 
    &= \Wvec{1}{i}\x + \Wvec{1}{i}\bdel\\
    &\leq
    \mathsf{M}^1_{\cW}  \cdot \left( \cos\left(\pi \theta^1(\x, \vc{s})\right)  + \nu \right) 
    \leq 0.
\end{align*}
Thus $I^1$ is also inactive for $\Phi^{[1]}(\x+\bdel)$. Similarly,
\begin{align*}
&\hat{\w}^1_i (\x+\bdel)\\
& = \Wvec{1}{i}(\x+\bdel) + 
\Big(\hat{\w}^1_i - \Wvec{1}{i}\Big) (\x+\bdel) \\
& \leq 
 \mathsf{M}^1_{\cW} \left( \cos\left(\pi \theta^1(\x, \vc{s})\right)  + \nu \right)
+ \norm{(\What^1-\W^1)}_{2,\infty} (1+\nu). \\
& \leq \mathsf{M}^1_{\cW} \cdot \Bigg[ 
\cos\left(\pi \theta^k(\x, \vc{s})\right) + \nu +  \frac{(1+\nu)\norm{\hat{h}-h}_{\cH^{K+1}}}{\sqrt{d^1}}\Bigg]  \leq 0.
\end{align*}
Hence $I^1$ is also inactive for $\hat{\Phi}^{[1]}(\x+\bdel)$.
Let $J^1, J^{0}$ be the complement active sets. To bound the distance between the representations, we note that
\begin{align*}
\norm{\hat{\Phi}^{[1]}(\x+\bdel) - \Phi^{[1]}(\x+\bdel)}_2 
& = \norm{\mathcal{P}_{J^1} \Big(\hat{\Phi}^{[1]}(\x+\bdel)-\Phi^{[1]}(\x+\bdel)\Big)}_2 \\
& \leq \norm{\mathcal{P}_{J^1} (\What^1 - \W^1) (\x+\bdel)}_2 \\
& \leq \norm{\mathcal{P}_{J^1} (\What^1-\W^1)}_2 \norm{\x+\bdel}_2\\
& \leq \sqrt{s^1} \cdot \norm{\What^1 - \W^1}_{2, \infty} (1+\nu) \\
& \leq \norm{\hat{h}-h}_{\cH^{K+1}}\sqrt{\frac{s^1}{d^1}} \cdot \mathsf{M}^1_{\cW} (1+\nu) \\
& < \norm{\hat{h}-h}_{\cH^{K+1}} \cdot \sqrt{1+\mathsf{M}^1_{s^1}} \cdot \mathsf{M}^1_{\cW} \cdot (1+\nu) \\
& = \norm{\hat{h}-h}_{\cH^{K+1}} \lparnu^{[1]}(\x, \vc{s}).
\end{align*}
Thus we have bounded the distance between the perturbed representations $\Phi^{[1]}(\x+\bdel)$ and $\hat{\Phi}^{[1]}(\x+\bdel)$ according to the local Lipschitz scale $\lparnu^{[1]}(\x, \vc{s})$. 

\textbf{Induction case $(k>2)$:}
Assume that the lemma holds for all $1\leq n\leq k-1$. Thus for each $n$, there exists an inactive set $I^n$ of size $s^n$ and hence $s^n \leq \bar{s}^n(\x)$.
Let $I^{k-1}$ be the index set common for $\Phi^{[k-1]}(\x)$, $\Phi^{[k-1]}(\x+\bdel)$ and $\hat{\Phi}^{[k-1]}(\x+\bdel)$. 
Since the lemma holds for $k-1$, we have that
\[
\norm{\hat{\Phi}^{[k-1]}(\x+\bdel)-{\Phi}^{[k-1]}(\x+\bdel)}_2 \leq (k-1) \zeta^{k-1}(\vc{s})(1+\nu).
\]
Suppose further that $\norm{\hat{h}-h}_{\cH^{K+1}} \leq \rparnu^{[k]}(\x, \vc{s})$, then the critical angular threshold $\theta^k(\x, \vc{s})$ is sufficiently large so that:
\[
\nu + (1+\nu) k\cdot\norm{\hat{h}-h}_{\cH^{K+1}}  \leq -\cos\left(\pi\theta^k(\x, \vc{s})\right). 
\]

Note that $-\cos\left(\pi \theta^k(\x, \vc{s})\right) > 0 \implies \theta^k(\x, \vc{s}) > \frac{1}{2}$.
Thus by definition, there exists an index set $I^k\subset [d^k]$ inactive for ${\Phi}^{[k]}(\x)$ such that for all $i\in I^k$,
\begin{align*}
\Wvec{k}{i}{\Phi}^{[k-1]}(\x) 
&\leq  
\mathsf{M}^k_{\cW}\zeta^{k-1}\left(\vc{0}\right)
\cos\left(\pi \theta^k(\x,\vc{s})\right).
\end{align*}
We can show that $I^k$ is also inactive for ${\Phi}^{[k]}(\x+\bdel)$ and $\hat{\Phi}^{[k]}(\x+\bdel)$. To see this, let $i\in I^k$,
\begin{align*}
    \Wvec{k}{i}{\Phi}^{[k-1]}(\x+\bdel) 
    &= \Wvec{k}{i}{\Phi}^{[k-1]}(\x) + \Wvec{k}{i}\left({\Phi}^{[k-1]}(\x+\bdel)-{\Phi}^{[k-1]}(\x)\right)\\
    &\leq \mathsf{M}^k_{\cW}\zeta^{k-1}(\vc{0})
   \cos\left(\pi \theta^k(\x,\vc{s})\right)\\
    &\qquad + \norm{\mathcal{P}_{J^{k-1}}\left(\Wvec{k}{i}\right)}_2\norm{\mathcal{P}_{J^{k-1}}\Big({\Phi}^{[k-1]}(\x+\bdel)-{\Phi}^{[k-1]}(\x)\Big)}_2\\
    &\leq \mathsf{M}^k_{\cW}\zeta^{k-1}(\vc{0})
   \cos\left(\pi \theta^k(\x,\vc{s})\right) + \mathsf{M}^k_{\cW} \zeta^{k-1}(\vc{s}) \nu \\
    &=
    \mathsf{M}^k_{\cW} \zeta^{k-1}(\vc{s}) \left(\cos\left(\pi \theta^k(\x, \vc{s})\right)  + \nu \right) \leq 0
\end{align*}
Thus, $I^k$ is inactive for $\Phi^{[k]}(\x+\bdel)$. In the above, note that by assumption for each $1\leq n \leq k-1$, $s^n \geq 0$ and hence $-\zeta^{k}(\vc{0}) \leq -\zeta^{k-1}(\vc{s})$.
Similarly, 
\begin{align*}
&\hat{\w}^k_i \hat{\Phi}^{[k-1]}(\x+\bdel)\\
& = \Wvec{k}{i}\Phi^{[k-1]}(\x+\bdel) + \Wvec{k}{i} (\hat{\Phi}^{[k-1]}(\x+\bdel)-\Phi^{[k-1]}(\x+\bdel)) + \Big(\hat{\w}^k_i - \Wvec{k}{i}\Big) \hat{\Phi}^{[k-1]}(\x+\bdel) \\
& \leq 
\mathsf{M}^k_{\cW} \zeta^{k-1}(\vc{s}) \left(  \cos\left(\pi \theta^k(\x, \vc{s})\right)  + \nu \right) 
 + \mathsf{M}^k_{\cW} (k-1) \norm{\hat{h}-h}_{\cH^{K+1}} \zeta^{k-1}(\vc{s})(1+\nu) \\
 & \qquad + \norm{(\What^k-\W^k)}_{2,\infty} \zeta^{k-1}(\vc{s})(1+\nu). \\
& \leq \mathsf{M}^k_{\cW} \zeta^{k-1}(\vc{s}) \cdot \Bigg[ 
 \cos\left(\pi \theta^k(\x)\right) + \nu + (1+\nu) (k-1)\cdot\norm{\hat{h}-h}_{\cH^{K+1}} + (1+\nu)\frac{\norm{\hat{h}-h}_{\cH^{K+1}}}{\sqrt{d^k}}\Bigg] \\
 & < \mathsf{M}^k_{\cW} \zeta^{k-1}(\vc{s}) \cdot \Bigg[ 
 \cos\left(\pi \theta^k(\x)\right) + \nu + (1+\nu) k\cdot\norm{\hat{h}-h}_{\cH^{K+1}} \Bigg]
 < 0.
\end{align*} 
Hence $I^k$ is also inactive for $\hat{\Phi}^{[k]}(\x+\bdel)$. 

Let $J^k, J^{k-1}$ be the complement of the active sets. To bound the distance between the representations, we note that
\begin{align*}
&\norm{\hat{\Phi}^{[k]}(\x+\bdel) - \Phi^{[k]}(\x+\bdel)}_2 \\
& = \norm{\mathcal{P}_{J^k} \Big(\hat{\Phi}^{[k]}(\x+\bdel)-\Phi^{[k]}(\x+\bdel)\Big)}_2 \\
& \leq \norm{\mathcal{P}_{J^k,J^{k-1}} (\What^k) \mathcal{P}_{J^{k-1}} \Big(\hat{\Phi}^{[k-1]}(\x+\bdel)\Big) - \mathcal{P}_{J^k,J^{k-1}} (\What^k) \mathcal{P}_{J^{k-1}} \Big(\Phi^{[k-1]}(\x+\bdel)\Big) }_2 \\
& \qquad + \norm{\mathcal{P}_{J^k,J^{k-1}} (\What^k) \mathcal{P}_{J^{k-1}} \Big(\Phi^{[k-1]}(\x+\bdel)\Big) - \mathcal{P}_{J^k,J^{k-1}} (\W^k) \mathcal{P}_{J^{k-1}} \Big(\Phi^{[k-1]}(\x+\bdel)\Big)}_2 \\
&\leq  \norm{\mathcal{P}_{J^k,J^{k-1}} (\What^k)}_2 \norm{\mathcal{P}_{J^{k-1}} \Big(\hat{\Phi}^{[k-1]}(\x+\bdel)-\Phi^{[k-1]}(\x+\bdel)\Big)}_2 \\
& \qquad + \norm{\mathcal{P}_{J^k,J^{k-1}} (\What^k-\W^k)}_2 \norm{\mathcal{P}_{J^{k-1}} (\Phi^{[k-1]}(\x+\bdel))}_2\\
& \leq \sqrt{1+\mathsf{M}^k_{s^k}} \cdot \mathsf{M}^k_{\cW} \cdot (k-1)\cdot\norm{\hat{h}-h}_{\cH^{K+1}} \cdot \zeta^{k-1}(\vc{s})(1+\nu) 
\\& \qquad
+ \sqrt{s^k} \cdot \norm{\What^k - \W^k}_{2, \infty} \zeta^{k-1}(\vc{s})(1+\nu) \\
& \leq  (k-1)\cdot\norm{\hat{h}-h}_{\cH^{K+1}}\cdot\zeta^{k}(\vc{s})(1+\nu) + \norm{\hat{h}-h}_{\cH^{K+1}} \cdot \sqrt{\frac{s^k}{d^k}} \cdot \mathsf{M}^k_{\cW} \cdot \zeta^{k-1}(\vc{s})(1+\nu). \\
& < k\cdot\norm{\hat{h}-h}_{\cH^{K+1}} \cdot \zeta^k(\vc{s})(1+\nu). \\
& = \lparnu^{[k]}(\x,\vc{s}) \cdot \norm{\hat{h}-h}_{\cH^{K+1}}.
\end{align*}
Thus we have bounded the distance between the perturbed representations $\Phi^{[k]}(\x+\bdel)$ and $\hat{\Phi}^{[k]}(\x+\bdel)$. 
Hence proved by induction. 
\end{proof}

The conclusion for \Cref{lemma: robust-local-lip-fnn} follows from noting that if $ \norm{\hat{h}-h}_{\cH^{K+1}} \leq \rparnu(\x,\vc{s})$, the local radius defined in the theorem statement, 
then 
\[
\norm{\hat{\Phi}^{[K]}(\x+\bdel) - \Phi^{[K]}(\x+\bdel)}_2 \leq K \zeta^K(\vc{s})(1+\nu)\cdot\norm{\hat{h}-h}_{\cH^{K+1}}.
\]
To calculate a local Lipschitz scale for the overall predictor outputs, 
\begin{align*}
&\norm{\hat{h}(\x+\bdel) - h(\x+\bdel)}_2 \\
&\leq \norm{(\What^{K+1}))^T \hat{\Phi}^{[K]}(\x+\bdel) - (\W^{K+1})^T \Phi^{[K]}(\x+\bdel)}_2 \\
& \leq \norm{(\What^{K+1})^T \hat{\Phi}^{[K]}(\x+\bdel) - (\What^{K+1})^T \Phi^{[K]}(\x+\bdel)}_2 
\\& \qquad 
+ \norm{(\What^{K+1})^T \Phi^{[K]}(\x+\bdel) - (\W^{K+1})^T \Phi^{[K]}(\x+\bdel)}_2 \\
&\leq  \norm{(\What^{K+1})^T}_2 \norm{\hat{\Phi}^{[K]}(\x+\bdel)-\Phi^{[K]}(\x+\bdel)}_2
+ \norm{(\What^{K+1})^T-\W^{K+1})^T}_2 \norm{\Phi^{[K]}(\x+\bdel)}_2\\
& \leq \sqrt{1+\mathsf{M}^{K+1}_{s^K}} \cdot \mathsf{M}^{K+1}_{\cW} \cdot K\zeta^{K}(\vc{s})\cdot (1+\nu) \cdot \norm{\hat{h}-h}_{\cH^{K+1}}
\\& \qquad
+ \zeta^{K}(\vc{s})\cdot (1+\nu) \cdot \mathsf{M}^{K+1}_{\cW} \norm{\hat{h}-h}_{\cH^{K+1}}
\\
& \leq (K+1) \zeta^{K+1}(\vc{s}) \cdot (1+\nu) \cdot \norm{\hat{h}-h}_{\cH^{K+1}}.
\end{align*}
Hence proved. 

\subsection{Proof for \Cref{thm: nonuniformriskriskMNN}}\label{app: thm: nonuniformriskMNN}
The proof follows the same sequence of steps as in \Cref{app: thm: generalization-sll}. The only modification is instantiating multiple bounds for all vector values of sparsity levels where for each $1\leq k\leq K$, $s^k\in [d^k]$. Correspondingly, we set the failure probabilty for each individual bound to $\alpha' = \frac{\alpha}{2\prod_{k=1}^K d^k}$. 
This results in an additional term that scales as $\sqrt{\frac{\sum_{k=1}^{K+1} d^k + \ln(\frac{2}{\alpha})}{m}}$. 
The conclusion follows the observation that $\sum_{k=1}^{K+1} d^k \leq \ln \Big(\mathcal{N}\left(\frac{1}{m (K+1)}, \cH^{K+1}\right)\Big)$ since for any $\epsilon > 0$ from \Cref{lemma:covermnn}, 
\[
\mathcal{N}\left(\epsilon, \cH^{K+1}\right)
\leq 
 \prod_{k=1}^{K+1} \left(\frac{5\sqrt{d^k}}{\epsilon }\right)^{d^{k} d^{k-1}}
\]

\subsection{Proof for \Cref{corollary: nonuniformriskriskMNN}}\label{app: corollary: nonuniformriskMNN}

\begin{theorem}
\label{app: thm: nonuniformriskriskMNN-predictor-dependent}
With probability at least $(1-\alpha)$ over the choice of i.i.d training sample $\texttt{S}_{T}$ and unlabeled data $\texttt{S}_{U}$, 
for any multi-layered neural network predictor $h \in \cH^{K+1}$ with parameters $\{\W^k\}$ 
the robust stochastic risk is bounded as, 
\begin{align*}
&R_{\mathrm{rob}}\left(h\right) 
-\hat{R}_{\mathrm{rob}}\left(h\right)\\
&\leq  
\mathcal{O}\left(
b \sqrt{\frac{\ln\left(\mathcal{N}\left(\frac{1}{m(K+1)},\cH^{K+1}\right)\right) +
+ \sum_{k=1}^{K+1} \ln(2+K\norm{\W^k}_{2,\infty})  
+
\ln(\frac{2}{\alpha})}{2m}}\right) \\
&  + \mathcal{O}\left( \frac{2\Lloss(1+\nu) \cdot \prod_{k=1}^{K+1} \left(\norm{\W^k}_{2,\infty} + \frac{1}{K}\right) \sqrt{2 +  \mu_{s^k,s^{k-1}}(\W^k)} }{m}\;\right).
\end{align*}

Here, $\vc{s} = s^{\star}_\nu(\samp_T \cup \samp_U, \frac{1}{2m})$ is the robust optimal sparsity level. Thus, 
\begin{align*}
&R_{\mathrm{rob}}\left(h\right) 
-\hat{R}_{\mathrm{rob}}\left(h\right) \leq\\
&  
\tilde{\mathcal{O}}\left(
b \sqrt{\frac{\ln\left(\mathcal{N}\left(\frac{1}{m(K+1)},\cH^{K+1}\right)\right) +
\ln(\frac{2}{\alpha})}{2m}}
+ \frac{\Lloss (1+\nu)}{m} \prod_{k=1}^{K+1} \norm{\W^k}_{2,\infty}\sqrt{1+\mu_{s^k,s^{k-1}}(\W^k)}
\right)
\end{align*}
\end{theorem}
\begin{proof}
In the proof to follow, we consider the dimensions $d^k$, the depth $K$ and the adversarial energy $\nu$ as fixed constants. 
We identify the set of feedforward network hyper parameters $\{\mathsf{M}^k_{\cW}, \mathsf{M}^k_{s}\}_{k=1}^{K+1}$ with a set of integer hyper-parameters $\mfi = \{\mfi^k_{\cW}, \mfi^k_{s}\}$ such that a network $h \in \cH^{(K+1)}_{\mathfrak{i}}$ has parameters $\{\W^k\}_{k=1}^{K+1} \in \prod_{k=1}^{K+1} \cW^k_{\mathfrak{i}}$ where,
\begin{align*}
    \cW^{K+1}_{\mathfrak{i}} &:= \{\W \in \mathbb{R}^{d^k\times C} ~|~  \norm{\W}_{2,\infty} < \frac{\mfi^{K+1}_{\cW}}{K}, \;\;     \mu_{(s_{\text{out}}, 0)}(\W) < \mfi^{K+1}_{s_{\text{out}}} ~~\forall~ s_{\text{out}} \in [d^{K}]\Big\}\\
    \cW^k_{\mathfrak{i}} &:= \Big\{ \W \in \mathbb{R}^{d^k \times d^{k-1}} ~|~ \norm{\W}_{2,\infty} < \frac{\mfi^k_{\cW}}{K}, \;\;     \mu_{(s_{\text{out}}, s_{\text{in}})}(\W) < \mfi^k_{s_{\text{out}}} \\ 
    &\qquad \qquad \qquad \qquad \qquad \forall~ s_{\text{out}} \in [d^k],\; s_{\text{in}} \in [d^{k-1}]
    \Big\}
\end{align*}
From \Cref{thm: nonuniformriskriskMNN}, we know that for the parameterized class $\cH^{(K+1)}_{\mathfrak{i}}$, w.p. $(1-\alpha_{\mathfrak{i}})$ over $\samp_T$ and $\samp_U$ the robust stochastic risk of any predictor $h$ is bounded as,
\begin{align*}
R_{\mathrm{rob}}\left(h\right) 
&\leq  
\hat{R}_{\mathrm{rob}}\left(h\right)\\
&\quad + \mathcal{O}\left(
b \sqrt{\frac{\ln\left(\mathcal{N}\left(\frac{1}{m(K+1)},\cH^{K+1}\right)\right) + \ln(\frac{2}{\alpha_\mfi})}{2m}} 
+  \frac{\Lloss \cdot \mathcal{L}_{\mathrm{rob}}(h, \samp_T \cup \samp_U,\frac 1{2m})}{m (K+1)}\right).
\end{align*}
Here the robust sparse regularity is, 
\[
\mathcal{L}_{\mathrm{rob}}(h, \samp_T \cup \samp_U,\frac 1{2m}) := (K+1)\zeta^{K+1}(\vc{s}) \cdot (1+\nu)
\]
where $\vc{s} = s^\star(\samp_T \cup \samp_U, \frac{1}{2m})$, the robust optimal sparsity level. Hence,
\begin{align}\label{eq:penultimateriskbound}
R_{\mathrm{rob}}\left(h\right) 
&\leq  
\hat{R}_{\mathrm{rob}}\left(h\right)\\
\nonumber &\quad + \mathcal{O}\left(
b \sqrt{\frac{\ln\left(\mathcal{N}\left(\frac{1}{m(K+1)},\cH^{K+1}\right)\right) + \ln(\frac{2}{\alpha_\mfi})}{2m}} 
+  \frac{\Lloss \cdot \zeta^{K+1}(\vc{s}) \cdot (1+\nu)}{m}\right).
\end{align}
Based on the definition of the parameter spaces, 
$\zeta^{K+1}(\vc{s}) 
= \prod_{k=1}^{K+1}
\frac{\mfi^k_{\cW}}{K} \sqrt{1+\mfi^k_{s^k}}
$. 
We seek a bound that depends directly on the true complexity measure evaluated on the particular trained network $h$ with weights $\{\W^k\}$ rather than their maximal sizes described by the set of integer hyper-parameters $\mfi$. 
For the sake of simplicity we group the terms in the RHS of \Cref{eq:penultimateriskbound} to get the following bound that depends on the integer hyper parameters $\mfi$ and a failure probability $\alpha_\mfi$ to make the following concise statement, 
\begin{align*}
 \text{w.p. }\geq (1-\alpha_\mfi) \;\; \forall~ h \in \cH_{\mathcal{M}}, \;\;
 R_{\mathrm{rob}}(h) \leq \hat{R}_{\mathrm{rob}}(h) + g(\mfi, \alpha_\mfi).
\end{align*}
We make a specific choice of the failure probability corresponding to a division of the space of integer hyper parameters,
\[
\alpha_\mfi := \alpha\cdot \prod_{k=1}^{K+1}
\left(\frac{1}{\mfi^k_{\cW}(\mfi^k_{\cW}+1)} \cdot \frac{1}{\prod_{s=1}^{d^k} \mfi^k_s(\mfi^k_s+1)}\right)
\]
Note that $\sum_\mfi \alpha_\mfi 
= \sum_{ \{\mfi^k_{\cW}, \{\mfi^k_s\}\}} \alpha_\mfi
= \alpha.$ and, 
\begin{align*}
\ln\left(\frac{1}{\alpha_\mfi}\right)
&\leq \ln\left(\frac{1}{\alpha}\right) 
+ 2\sum_{k=1}^{K+1} \left(\ln(\mfi^k_{\cW}+1) + \sum_{s=1}^{d^k} \ln(\mfi^k_{s}) \right)
\end{align*}
By union-bound over the choices of integer hyper parameters, w.p. $(1-\alpha)$ over $\samp_T$ and $\samp_U$,
\[
\forall \mfi,\; \forall h \in \cH_\mfi, \quad R_{\mathrm{rob}}(h) \leq \hat{R}_{\mathrm{rob}}(h) + g(\mfi, \alpha_\mfi). 
\]
Thus if a learning algorithm returns a predictor $h$, one can find the tightest set of integer hyper parameters $\mfi^* = \{(\mfi^{k}_{\cW})^*, (\mfi^k_{s^k})^*\}$ such that for all $1\leq k\leq K+1$,
\begin{align*}
    \frac{(\mfi^{k}_{\cW})^*-1}{K} &\leq \norm{\W^k}_{2,\infty} < \frac{(\mfi^{k}_{\cW})^*}{K}, \quad \\
    (\mfi^k_{s_{\text{out}}})^*-1 &\leq \mu_{s_{\text{out}}, s_{\text{in}}}(\W^k) < (\mfi^k_{s_{\text{out}}})^*, \quad \forall s_{\text{in}}, s_{\text{out}}.
    \end{align*}
Hence,
\begin{align*}
&\prod_{k=1}^{K+1} \norm{\W^k}_{2,\infty} \sqrt{1+\mu_{s^k,s^{k-1}}(\W^k)}\\
&\qquad \quad \leq \zeta^{K+1}(\vc{s})\\
& \qquad \qquad \qquad    \leq \prod_{k=1}^{K+1} \left(\norm{\W^k}_{2,\infty} + \frac{1}{K}\right) \sqrt{2 +  \mu_{s^k,s^{k-1}}(\W^k)}
\end{align*}The first result in the theorem statement is obtained by instantiating a bound for the chosen failure probability and predictor-dependent choice of hyper parameters $\mfi^*$. Further one can obtain an equivalent result by suppressing log factors in the limit when $m \rightarrow \infty$ and $\norm{\W^k}_{2,\infty} \rightarrow \infty$.
\end{proof}
We can now instantiate \Cref{app: thm: nonuniformriskriskMNN-predictor-dependent} for any $b$-bounded loss function $\ell$ with Lipschitz constant $\Lloss$. A popular choice is the ramp loss $\ell^\gamma$ with a fixed margin threshold $\gamma > 0$. The ramp loss $\ell^\gamma$ is $1$-bounded and $\frac{2}{\gamma}$-Lipschitz. 
To prove \Cref{corollary: nonuniformriskriskMNN}, we note that the $0-1$ misclassification loss is always upper bounded by the ramp-loss and one can extend the statement of \Cref{app: thm: nonuniformriskriskMNN-predictor-dependent} to ramp loses with any margin threshold $\gamma$ with the same covering trick used above to move away from bounds that depend on norm balls to specific predictor dependent properties. 

\subsection{Discussion on SLL.w.r.t parameters} 
\subsubsection{Input vs Parameter Sensitivity}\label{subsubsec:inp-vs-param}
We highlight the subtle differences between our sparse local Lipschitz analysis in \Cref{subsec: cert-rob-fnn} and \Cref{subsec: rob-gen-fnn}.  
In \Cref{corollary:cert-rob-fnn-loc}, we quantified the sparse local Lipschitzness of a neural network (w.r.t. inputs) in terms of the minimum absolute pre-activation value of a reduced subset of inactive components. 
For certified robustness, there are no parameter perturbations and thus the local Lipschitz scale directly depends on the norm of the (reduced) predictor weights, $\norm{\mathcal{P}_{J^k,J^{k-1}}(\W^k)}_{2,\infty}$. 
In \Cref{lemma: robust-local-lip-fnn} we study the sparse local Lipschitzness of a neural network (w.r.t. parameters) via the critical angular distance.
In the presence of parameter perturbations, the corresponding local Lipschitz scale also depends on the properties of the perturbed weight $\What^k$, hence our characterization uses the common norm bound given by $\mathsf{M}^k_{\cW}\sqrt{1+\mathsf{M}^k_{s^k}}$ for all reduced linear maps from $\cW^k$.
A weaker result to \Cref{corollary:cert-rob-fnn-loc} for input-sensitivty can be established with a smaller local radius dependent on the critical angular distance and a larger local Lipschitz scale dependent on $\zeta^{k}(\vc{s})$. 
This is the result alluded to in \Cref{subsubsec: Dependence on Input}.
For more details see \Cref{app: lemma: fnn-layer-lip-weaker}.

\subsubsection{Robust Flatness}\label{app: lemma:flatness}
\begin{theorem} (Robust Flatness of a Trained Predictor)\label{lemma:flatness}
Let $h$ be a SLL predictor w.r.t. both inputs and parameters. For a given sparsity level $\vc{s}$, at any input $\x$ in the training set $\samp_T$, the predicted output $j^* := \argmax_j [h(\x)]_j$ remains unchanged for simultaneous adversarial corruption $\bdel \in \cB^{\cX}_\nu(\mathbf{0})$ and parameter perturbation  
if 
\begin{itemize}
    \item The maximum size of the adversarial corruption is bounded such that,  \[\nu  \in \min_{\x \in \samp_T} \rinp(\x, \vc{s}).\]
    \item The parameter perturbation is within the robust flat radius, 
$\norm{\hat{h}-h}_\cH \leq r_{\mathrm{flat}, \nu}(h, \samp_T, \vc{s})$ where,
\begin{align*}
	r_{\mathrm{flat}, \nu}(h, \samp_T, \vc{s}) := \min_{\x \in \samp_T} \left\{\rparnu(\x,  \vc{s}) , ~  \frac{\max\{\rho(\x) - \zeta^{K+1}(\vc{s}) \nu,~0\}}{2\cdot \lparnu(\x, \vc{s}) } \right \}.
\end{align*}
\end{itemize}
\end{theorem}
\begin{proof}
The proof of the above results mirrors the analysis in \Cref{lemma:cert-rob-gen-loc}.
First, we note that when $\nu \leq \min_{\x \in \samp_T} \rinp(\x, \vc{s})$, for any corruption $\bdel \in \cB^{\cX}_\nu(\mathbf{0})$
\[
\norm{h(\x+\bdel)-h(\x)}_2 \leq \norm{\A}_2 \linp(\x, \vc{s})\norm{\bdel}_2. 
\]
Additionally, when $\norm{\hat{h}-h}_\cH \leq \rparnu(\x, \vc{s})$, 
\[
\norm{\hat{h}(\x+\bdel)-h(\x+\bdel)}_2 \leq \lparnu(\x, \vc{s}) \norm{\hat{h}-h}_\cH.
\]
Thus, when both these assumptions hold, 
\[
\norm{\hat{h}(\x+\bdel)-h(\x)}_2 \leq \lparnu(\x, \vc{s}) \norm{\hat{h}-h}_\cH + \norm{\A}_2 \linp(\x, \vc{s})\nu 
\]
Recall the margin operator $\mathcal{M} : \cY' \times \cY \rightarrow \mathbb{R}$, $\mathcal{M}(\vt, y) := [\vt]_{y} - \max_{j \neq y} [\vt]_j$. The original predicted label is $j^* = \argmax_j [h(\x)]_j$ and hence the classification margin $\rho(\x) = \mathcal{M}(h(\x),j^*)$.
Observe that the predicted labels remain unchanged when
$\mathcal{M}(\hat{h}(\x+\bdel), j^* ) \geq 0$.
The margin operator $\mathcal{M}(\cdot, j)$ is $2$-Lipschitz in $\cYd$ w.r.t. $\ell_p$ norm for any $p\geq 1$ (see \cite[Lemma A.3.]{Bartlett2017SpectrallynormalizedMB}), thus
\[
 \mathcal{M}\big(h(\x),j^* \big) - \mathcal{M}\big(\hat{h}(\x+\bdel), j^*\big) \leq 2 \norm{\hat{h}(\x+\bdel) - h(\x)}_p.
\]
The conclusion follows from substituting the upper bound for the RHS and rearranging the above inequality. 
\begin{align*}
\begin{cases}
\nu \leq \min_{\x \in \samp_T} \rinp(\x, \vc{s})\\
\norm{\hat{h}-h}_\cH \leq r_{\mathrm{flat}, \nu}(h, \samp_T, \vc{s}) \end{cases}
\implies \mathcal{M}\big(h(\x+\bdel), j^*\big) \geq 0. 
\end{align*}
\end{proof}
The robust flat radius varies can also be optimized for the choice of sparsity level and easily incorporated into training schemes as a regularizer that encourages predictors that are flat minima. 

\subsection{Lemmas on Covering Number} 
Let $(\mathcal{U},\rho)$ be a metric space and let $\mathcal{V}$ be a bounded subset of $\mathcal{U}$. Given $\epsilon > 0$, a set $\mathcal{C}=\{u_1, \ldots, u_n\} \subseteq \mathcal{U}$ is called an $\epsilon$-cover for the set $\mathcal{V}$ if for each $v \in \mathcal{V}$, there exists a cover element $u_i \in \mathcal{U}$ such that $\rho(u_i, v) \leq \epsilon$. The size of the minimal $\epsilon$-cover for the set $\mathcal{V}$ is denoted $\mathcal{N}(\epsilon, \mathcal{U}, \rho)$.
A proper $\epsilon$-cover for $\mathcal{V}$ is an $\epsilon$-cover $\mathcal{C}_{\text{proper}}$ such that $\mathcal{C}_{\text{proper}}\subseteq \mathcal{V} \subseteq \mathcal{U}$. 
The size of minimal proper $\epsilon$-cover is denoted $\mathcal{N}_{\text{proper}}(\epsilon, \mathcal{V}, \rho)$.  We quote below a useful lemma that bounds the size of a minimal proper $\epsilon$-cover in terms of a minimal $\epsilon$-cover. 

\begin{lemma}(\cite{vidyasagar2013learning} 2.1)
\label{lemma:propercov}
Let $\mathcal{U}$ be a normed vector space  and let $\mathcal{V} \subset \mathcal{U}$ be a constrained bounded set.  
\[
\mathcal{N}_{\text{proper}}\left(2\epsilon, \mathcal{V}, \rho\right) \leq \mathcal{N}\left(\epsilon, \mathcal{V}, \rho\right) \leq \mathcal{N}_{proper}\left(\epsilon, \mathcal{V}, \rho\right) 
\]
\end{lemma}

\begin{lemma}
\label{lemma:coverunitball}
For any closed unit ball $B^{\mathbb{R}^d}_1 := \{\x \in \mathbb{R}^d ~|~ \norm{\x}\leq 1\}$, the size of a proper $\epsilon$-cover w.r.t. its norm induced metric $\rho(\x_1, \x_2) := \norm{\x_1-\x_2}$, is bounded,
\[
\mathcal{N}(\epsilon, B^{\mathbb{R}^d}_1 , \rho) \leq \left(1+\frac{2}{\epsilon}\right)^d.
\]
\end{lemma}
When the context is clear we refer to covering numbers w.r.t. norms to indicate the covering number w.r.t. the norm induced metric. 
We rely on the simpler bounds on covering numbers of closed unit balls in their own metric to compute the covering numbers of the space of classification and representation weights for a representation-linear hypothesis class. 
\begin{lemma}
\label{lemma:cov2}
For the set $\cA := \{\A \in \mathbb{R}^{C \times p} ~|~ \norm{\A}_2 \leq M_{\cA}\}$, the size of a minimal $\epsilon$-cover is bounded,
\[
\mathcal{N}\left(\epsilon M_\cA, \mathcal{A}, \norm{\cdot}_2\right)\leq \left(1+\frac{2\sqrt{p}}{\epsilon}\right)^{C p}.
\]
\end{lemma}	
\begin{proof}
For any matrix $\A \in \cA$, note that $\norm{\A}_F \leq \sqrt{p}\norm{\A}_2 \leq \sqrt{p}M_\cA$. 
Further any cover for the set $\cA$ in frobenius norm $\norm{\cdot}_F$ is also a cover in operator norm $\norm{\cdot}_2$.   
We also denote the vector space $\tilde{\cA}_{\text{vec}} := \{ \vt \in \mathbb{R}^{Cp} ~|~ \norm{\vt}_2 \leq \sqrt{p} M_\cA\}$. One can construct a cover for the set $\cA$ in frobenius norm from a cover for the set $\tilde{\cA}_{\text{vec}}$ in vector euclidean norm. Hence we can bound the required covering number as,
\begin{align*}
\mathcal{N}\left(\epsilon M_\cA, \mathcal{A}, \norm{\cdot}_2\right)	&\leq \mathcal{N}\left(\epsilon M_\cA, \mathcal{A}, \norm{\cdot}_F\right) \\
	& = \mathcal{N}\left(\epsilon M_\cA, \tilde{\cA}_{\text{vec}}, \norm{\cdot}_2\right) \\
	& \leq \left(1 + \frac{2\sqrt{p}M_\cA}{\epsilon M_\cA}\right)^{C p}
	= \left(1+\frac{2\sqrt{p}}{\epsilon}\right)^{C p}. 
\end{align*}
Here the last line uses the result from \Cref{lemma:coverunitball}.
\end{proof}

\begin{lemma}
\label{lemma:cover12inf}
For the set $\cW^k$ defined as,
\begin{align*}
 \cW^{k} :=& \Big\{\W \in \mathbb{R}^{d^k \times d^{k-1}} ~|~ \norm{\W}_{2,\infty} \leq \mathsf{M}^k_{\cW}, \quad  \forall~ s^k\in [d^k], s^{k-1} \in [d^{k-1}],\; \nonumber\\
 &\qquad \qquad \qquad \quad \mu_{(s^k, s^{k-1})}(\W) \leq \mathsf{M}^k_{s^k} 
 \Big\}.
\end{align*}
The size of the minimal $\epsilon$-cover is bounded,
	\begin{align*}
		\mathcal{N}\left(\epsilon \mathsf{M}^k_{\cW}, \mathcal{W}^k, \norm{\cdot}_{2, \infty}\right) \leq  \left(1+\frac{2}{\epsilon }\right)^{d^{k} d^{k-1}}.
	\end{align*}
	\begin{proof}
	Let $\cW^k_{\text{group}} := \{ \W \in \mathbb{R}^{d^k \times d^{k-1}} ~|~ \norm{\W}_{2, \infty} \leq \mathsf{M}^k_{\cW}\}$. Note that $\cW^k \subseteq \cW^k_{\text{group}}$ as $\cW^k$ is additionally constrained via the restricted babel function. 
	Thus a minimal cover for the larger set $\cW^k_{\text{group}}$ in the group norm is also a cover for $\cW^k$ in the group norm. 
	Further to cover $\cW^k_{\text{group}}$ with the group norm $\norm{\cdot}_{2, \infty}$ one can instantiate $d^k$ minimal covers for each row. Thus it is sufficient to combine the covers for the vector space $\cW^k_{\text{vec}} := \{ \vc{w} \in \mathbb{R}^{d^{k-1}} ~|~ \norm{\vc{w}}_{2} \leq \mathsf{M}^k_{\cW}\}$. Hence 
		\begin{align*}
        	\mathcal{N}\left(\epsilon \mathsf{M}^k_{\cW}, \mathcal{W}^k, \norm{\cdot}_{2, \infty}\right) 
        	&\leq \mathcal{N}\left(\epsilon \mathsf{M}^k_{\cW}, \cW^k_{\text{group}} , \norm{\cdot}_{2, \infty}\right) \\
			&\leq \left(\mathcal{N}\left(\epsilon \mathsf{M}^k_{\cW}, \cW^k_{\text{vec}}, \norm{\cdot}_{2}\right)\right)^{d^k} 
			\\
			& \leq \left(1 + \frac{2 \mathsf{M}^k_{\cW}}{\epsilon \mathsf{M}^k_\cW}\right)^{d^{k} d^{k-1}} =  \left(1+\frac{2}{\epsilon }\right)^{d^{k} d^{k-1}}.
		\end{align*}
		Here the last line uses the result from \Cref{lemma:coverunitball}.
	\end{proof}
\end{lemma}
Consider the hypothesis class $\cH^{K+1}$ with the norm $\norm{\cdot}_\cH$ as defined in \Cref{subsec: rob-gen-fnn}. For any two predictors $\hat{h}, h \in \cH^{K+1}$,
\begin{align*}
    &\norm{\hat{h}-h}_{\cH^{K+1}} \leq \epsilon 
    \Leftrightarrow ~\forall~ k,\quad \norm{\What^k-\W^k}_{\cW^k} \leq \epsilon 
    \Leftrightarrow ~\forall~ k,\quad \norm{\What^k-\W^k}_{2,\infty} \leq \frac{\mathsf{M}^k_{\cW}\epsilon}{\sqrt{d^k}}.
\end{align*}

\begin{lemma}\label{lemma:covermnn}
For a set of cover resolutions $\bm{\epsilon}$, 
\begin{align*}
\mathcal{N}_{proper}\Bigg (\epsilon, \cH^{{K+1}}, 
\norm{\cdot}_{\cH^{K+1}}
\Bigg) 
\leq 
 \prod_{k=1}^{K+1} \left(1+\frac{4\sqrt{d^k}}{\epsilon }\right)^{d^{k} d^{k-1}}
\end{align*}
\end{lemma}
\begin{proof}
   The conclusion follows from combining \Cref{lemma:propercov}, and \Cref{lemma:cover12inf},
   \begin{align*}
\mathcal{N}_{proper}\Bigg (
\epsilon, \cH^{{K+1}}, 
\norm{\cdot}_{\cH^{K+1}}\Bigg) 
&= \prod_{k=1}^{K+1} \mathcal{N}_{proper}\Bigg( \frac{\mathsf{M}^k_{\cW}\epsilon}{\sqrt{d^k}},\; \cW^k, \norm{\cdot}_{2, \infty} \Bigg) \\
&\leq \prod_{k=1}^{K+1} \mathcal{N}\Bigg( \frac{\mathsf{M}^k_{\cW}\epsilon}{2\sqrt{d^k}},\; \cW^k, \norm{\cdot}_{2, \infty} \Bigg) \\
&\leq \prod_{k=1}^{K+1} \left(1+\frac{4\sqrt{d^k}}{\epsilon }\right)^{d^{k} d^{k-1}}.
\end{align*}
\end{proof}

\end{document}